\documentclass[dvipsnames,accepted]{article}
\usepackage{microtype}
\usepackage{graphicx}
\usepackage{booktabs} %

\usepackage{hyperref}

\usepackage{icml2025}

\usepackage{amsmath}
\usepackage{amssymb}
\usepackage{mathtools}
\usepackage{amsthm}

\usepackage{thmtools}
\usepackage{thm-restate}

\usepackage[capitalize,noabbrev]{cleveref}

\theoremstyle{plain}
\newtheorem{theorem}{Theorem}[section]

\newtheorem{lemma}[theorem]{Lemma}
\newtheorem{corollary}[theorem]{Corollary}
\theoremstyle{definition}

\theoremstyle{remark}
\newtheorem{remark}[theorem]{Remark}

\usepackage{amsmath}
\usepackage{amssymb}
\usepackage{amsthm}

\usepackage{xargs}
\usepackage{xspace}
\usepackage{xparse}

\usepackage{setspace}

\usepackage{xcolor}

\usepackage{booktabs}
\usepackage[flushleft]{threeparttable}

\usepackage{pifont}

\usepackage{xfrac}

\usepackage{subcaption}

\usepackage{algorithm}
\usepackage{algorithmic}

\usepackage{comment}

\usepackage{bold-extra}

\usepackage{cleveref}

\usepackage{enumitem}

\usepackage{tabularx} %
\usepackage{array} %

\usepackage{multirow}

\newcolumntype{C}{>{\centering\arraybackslash}X}

\usepackage{bbm}
\usepackage{bold-extra}

\usepackage{regexpatch}
\makeatletter
\xpatchcmd\thmt@restatable{%
\csname #2\@xa\endcsname\ifx\@nx#1\@nx\else[{#1}]\fi
}{%
\ifthmt@thisistheone
\csname #2\@xa\endcsname\ifx\@nx#1\@nx\else[{#1}]\fi
\else
\csname #2\@xa\endcsname[{Restated}]
\fi}{}{}
\makeatother

\def\rset{\mathbb{R}}
\def\nset{\mathbb{N}}

\def\markovkernel{\mathrm{K}_{(\step,H)}}
\newcommandx{\firstsentence}[1]{\textcolor{purple}{#1}}

\newcommandx{\paul}[1]{\todo[color=purple!20]{PM: #1}}
\newcommandx{\pauli}[1]{\todo[inline,color=purple!20]{PM: #1}}
\newcommandx{\todoi}[1]{\todo[inline,color=red!50]{TODO: #1}}

\renewcommand{\paul}[1]{}
\renewcommand{\pauli}[1]{}
\renewcommand{\todoi}[1]{}

\def\eqsp{\enspace}
\def\rmd{\mathrm{d}}
\newcommand{\eqdef}{\overset{\Delta}{=}}

\def\Id{\mathrm{Id}}
\DeclareMathOperator{\tr}{tr}

\DeclareMathOperator*{\argmin}{\arg\min}

\def\PE{\mathbb{E}}
\newcommandx{\CPE}[2]{\PE \left[ #1 ~\middle|~ #2 \right]}
\newcommandx{\expe}[1]{\PE \left[ #1 \right]}

\def\iid{i.i.d.}

\newcommandx{\norm}[2][2=]{ \lVert #1 \rVert_{#2} }
\newcommandx{\bnorm}[2][2=]{ \Big\lVert #1 \Big\rVert_{#2} }
\newcommandx{\pscal}[3][3=]{ \langle #1 , #2 \rangle_{#3}}
\newcommandx{\bpscal}[3][3=]{ \Big\langle #1 , #2 \Big\rangle_{#3}}
\newcommandx{\abs}[1]{ | #1 | }
\newcommandx{\babs}[1]{ \Big| #1 \Big| }
\def\wasserstein{\mathbf{W}_2}

\newcommand{\newcheckmark}{\textrm{\color{ForestGreen}\ding{51}}}%
\newcommand{\newcrossmark}{\textrm{\color{BrickRed}\ding{55}}}%

\DeclareMathOperator{\diag}{diag}

\renewcommandx{\iint}[2]{\{#1,\dots,#2\}}

\def\skipquad{\\ & \quad}
\def\skipqquad{\\ & \qquad}
\newcommand{\hiddencst}[1]{}

\def\hidleq{\lesssim}
\def\mainalign{&}
\def\restatealign{}
\def\nonumberinmain{\nonumber}

\newcommand{\resetspaces}{
  \renewcommand{\skipquad}{}
  \renewcommand{\skipqquad}{}
  \renewcommand{\hidleq}{\leq}
  \renewcommand{\nonumberinmain}{}
  \renewcommand{\mainalign}{}
  \renewcommand{\restatealign}{&}
}

\def\FedAvg{{\textsc{FedAvg}}\xspace}

\def\Scafflsa{\textsc{ScaffLSA}\xspace}
\def\Scaffold{\textsc{Scaffold}\xspace}
\def\ProxSkip{\textsc{ProxSkip}\xspace}

\def\SGD{\textsc{SGD}\xspace}

\def\param{\theta}

\def\paramp{\param^+}

\def\varparam{\vartheta}
\newcommand{\locparam}[2]{\param_{(#1)}^{#2}}
\newcommand{\tlocparam}[2]{\tilde{\param}_{(#1)}^{#2}}
\newcommand{\vlocparam}[2]{\varparam_{(#1)}^{#2}}
\newcommand{\locshiftparam}[2]{\widetilde{\param}_{(#1)}^{#2}}
\newcommand{\locvarparam}[2]{\varparam_{(#1)}^{#2}}
\newcommand{\locshiftvarparam}[2]{\widetilde{\varparam}_{(#1)}^{#2}}

\newcommand{\globparam}[1]{\param^{#1}}
\newcommand{\tglobparam}[1]{\tilde{\param}^{#1}}
\newcommand{\globvarparam}[1]{\varparam^{#1}}

\def\bigX{\mathrm{X}}

\def\bigXlim{\mathrm{X}^{\star}}

\def\vecX{\mathbf{x}}
\def\barvecX{\mathbf{\bar{x}}}
\def\vecY{\mathbf{y}}
\def\barvecY{\mathbf{\bar{y}}}

\def\Cvarw{\Xi}
\def\cvarw{\xi}
\def\varcvarw{\xi}

\newcommand{\cvar}[2]{\cvarw_{(#1)}^{#2}}
\newcommand{\tcvar}[2]{{\cvarw}_{(#1)}^{\prime #2}}

\def\paramlim{\param^{\star}}

\newcommand{\cvarlim}[1]{\cvar{#1}{\star}}

\def\opw{\msT}
\def\optw{\widetilde{\msT}}
\def\opcvw{\msV} 
\def\opscaffold{\msS}

\newcommandx{\locscafop}[4]{\opw_{(#1)}^{#2}(#3, #4)}
\newcommandx{\locstoscafop}[5]{\opw_{(#1)}^{#2}(#3; #4, #5)}
\newcommandx{\globstoscafop}[4]{\opw^{}(#2; #3, #4)}

\newcommandx{\scafopcv}[5]{\opcvw_{(#1)}^{}(#3; #4, #5)}

\newcommandx{\scafopabv}[2]{\opw^{#1}(#2)}
\newcommandx{\locscafopabv}[3]{\opw_{(#1)}^{#2}(#3)}
\newcommandx{\locscafopadjabv}[3]{\optw_{(#1)}^{#2}(#3)}
\newcommandx{\locscafopcvabv}[2]{\opcvw_{(#1)}(#2)}

\newcommandx{\cvplus}[1]{\cvar{#1}{+}}
\newcommandx{\pplus}{\param^{+}}
\newcommandx{\pch}[2]{\param_{#1}^{#2}}

\def\randStatew{Z}
\def\randStatezw{z}
\newcommand{\locRandStatew}[1]{Z_{(#1)}}
\newcommand{\locRandState}[2]{\randStatew_{(#1)}^{#2}}

\newcommandx{\locRandDist}[1]{\mcD^{#1}}
\newcommand{\randStatez}[2]{\randStatezw_{(#1)}^{#2}}

\def\strcvx{\mu}
\def\lip{L}
\def\thirdlip{Q}
\def\fourthlip{G}

\def\fw{f}
\newcommandx{\nfw}[1]{f_{(#1)}}
\newcommandx{\nfsw}[2]{F^{#2}_{(#1)}}
\newcommandx{\f}[1]{f(#1)}
\newcommandx{\nf}[2]{f_{(#1)}(#2)}
\newcommandx{\nfs}[3]{F^{#3}_{(#1)}(#2)}

\def\gfw{\nabla \fw}
\newcommandx{\ngfw}[1]{\nabla \nfw{#1}}
\newcommandx{\ngfsw}[2]{\nabla \nfsw{#1}{#2}}

\newcommandx{\gf}[1]{\gfw(#1)}
\newcommandx{\gnf}[2]{\ngfw{#1}(#2)}
\newcommandx{\gnfs}[3]{\nabla \nfs{#1}{#2}{#3}}

\newcommandx{\nhfw}[1]{\nabla^2 \nfw{#1}}
\newcommandx{\nhfsw}[2]{\nabla^2 \nfsw{#1}{#2}}

\newcommandx{\hf}[1]{\nabla^2 \f{#1}}
\newcommandx{\hnf}[2]{\nabla^2 \nf{#1}{#2}}
\newcommandx{\hnfs}[3]{\nabla^2 \nfs{#1}{#2}{#3}}

\newcommandx{\hhf}[1]{\nabla^3 \f{#1}}
\newcommandx{\hhnf}[2]{\nabla^3 \nf{#1}{#2}}
\newcommandx{\hhnfs}[3]{\nabla^3 \nfs{#1}{#2}{#3}}

\newcommandx{\avghnf}[2]{\bar{D}^{2,#2}_{(#1)}}
\newcommandx{\avghhnf}[2]{\bar{D}^{3,#2}_{(#1)}}

\newcommandx{\hhhnf}[2]{\nabla^4 \nf{#1}{#2}}
\newcommandx{\avghhhnf}[2]{\bar{D}^{4,#2}_{(#1)}}

\newcommandx{\locreste}[2]{{\mathcal{R}}_{(#1)}^{#2}}
\newcommandx{\avgreste}[3]{\bar{\mathcal{R}}_{(#1)}^{#2}\left( #3 \right)}
\newcommandx{\resterand}[3]{\mathrm{r}_{(#1)}^{#2}\left( #3 \right)}
\newcommandx{\reste}[3]{\mathcal{R}_{(#1)}^{#2}\left( #3 \right)}
\newcommandx{\restebis}[3]{\mathcal{Q}_{(#1)}^{#2}\left( #3 \right)}

\def\Aw{\mathbf{A}}
\def\bw{\mathbf{b}}
\newcommandx{\nbarA}[1]{\bar{\Aw}\!_{(#1)}}
\newcommandx{\barA}{\bar{\Aw}}
\newcommandx{\nbarb}[1]{\bar{\bw}_{(#1)}}
\newcommandx{\nA}[2]{\Aw\!_{(#1)}(#2)}
\newcommandx{\nb}[2]{\bw_{(#1)}(#2)}

\def\heterboundw{\zeta}
\def\heterboundgrad{\heterboundw_1}
\def\heterboundhess{\heterboundw_2}

\newcommandx{\bias}[1]{\mathbf{\rho_{(#1)}}}

\newcommandx{\statdist}[1]{\pi_{(#1)}}
\newcommandx{\statdistlim}[1]{\bar{\param}^{(#1)}}

\def\nagent{N}
\def\nrounds{T}
\def\nlupdates{H}
\def\step{\gamma}

\def\contractw{\Gamma}
\newcommandx{\loccontractw}[1]{\Gamma_{\!(#1)}}
\newcommandx{\globcontractw}{\bar{\Gamma}}%
\def\expcontractw{\bar{\Gamma}}
\newcommandx{\loccontract}[2]{\contractw^{#2}_{(#1)}}
\newcommandx{\avgcontract}[1]{\contractw^{(\text{avg})}_{(#1)}}
\newcommandx{\globcontract}[1]{\contractw_{(#1)}}
\newcommandx{\weightedglobcontract}[2]{\contractw_{(#2,#1)}}
\newcommandx{\expweightedglobcontract}[1]{\expcontractw_{(#1)}}
\newcommandx{\exploccontract}[2]{\expcontractw^{#2}_{(#1)}}

\newcommandx{\idmcontract}[1]{G_{(#1)}}

\newcommandx{\locmat}[2]{\mathrm{C}_{\!(#1)}^{#2}}
\newcommandx{\shiftedlocmat}[2]{\mathrm{\widetilde C}_{\!(#1)}^{#2}}
\newcommandx{\diffcontractc}[1]{\Delta^{\Gamma}_{(#1)}}

\newcommand{\biasparam}{\bar{\boldsymbol{b}}^{\param}}
\newcommand{\biascvar}[1]{\bar{\boldsymbol{b}}^{\cvarw}_{(#1)}}
\newcommand{\covparam}{\bar{\boldsymbol{\Sigma}}^{\param}}
\newcommand{\covparamcvar}[1]{\bar{\boldsymbol{\Sigma}}^{\param,\cvarw}_{(#1)}}
\newcommand{\covcvarparam}[1]{\bar{\boldsymbol{\Sigma}}^{\cvarw,\param}_{(#1)}}
\newcommand{\covcvar}[1]{\bar{\boldsymbol{\Sigma}}^{\cvarw}_{(#1)}}
\newcommand{\covonestep}{\bar{\boldsymbol{\Sigma}}^{\epsilon}}
\newcommandx{\loccovonestep}[1]{\bar{\boldsymbol{\Sigma}}^{\epsilon}_{(#1)}}

\newcommand{\distRandState}[1]{\nu_{(#1)}}
\newcommand{\updatefuncnoise}[3]{\varepsilon_{(#1)}^{#2}(#3)}

\def\sqoptvar{\sigma_\star}
\def\optvar{\sqoptvar^2}
\def\smoothcstvar{\beta}

\def\noisew{\varepsilon}

\newcommand{\locnoiseabv}[2]{\noisew_{(#1)}^{#2}}

\newcommand{\boundlocparam}[1]{\rm{C}_{(#1)}^{\param}}
\newcommand{\boundcvar}[1]{\rm{C}_{(#1)}^{\cvarw}}

\newcommandx{\bcovparam}{\mathrm{C}^{\theta}}
\newcommandx{\bcovparamcvar}{\mathrm{C}^{\theta,\cvarw}}
\newcommandx{\bcovcvareq}{\mathrm{C}^{\cvarw, =}}
\newcommandx{\bcovcvarneq}{\mathrm{C}^{\cvarw, \neq}}
\def\bcovallcvar{\mathrm{C}^\cvarw}
\def\globboundnoise{\mathrm{\varsigma}^{\epsilon}}

\def\opcov{\mathbf{A}}

\newcommandx{\noisecovmat}[1]{\mathcal{C}(#1)}
\newcommandx{\locnoisecovmat}[2]{\mathcal{C}_{#1}(#2)}

\def\rmR{\mathrm{R}}

\newcommand{\stationaryparam}[1]{\hat{\param}^{#1}}
\newcommand{\stationarycvar}[2]{\hat{\cvarw}_{(#1)}^{#2}}
\def\stationarybigX{\hat \bigX}

\ExplSyntaxOn

\NewDocumentCommand{\definealphabet}{mmmm}
 {%
  \int_step_inline:nnn { `#3 } { `#4 }
   {
    \cs_new_protected:cpx { #1 \char_generate:nn { ##1 }{ 11 } }
     {
      \exp_not:N #2 { \char_generate:nn { ##1 } { 11 } }
     }
   }
 }

\ExplSyntaxOff

\definealphabet{bb}{\mathbb}{A}{Z}
\definealphabet{mc}{\mathcal}{A}{Z}
\definealphabet{mc}{\mathcal}{a}{z}
\definealphabet{mf}{\mathfrak}{A}{Z}
\definealphabet{mf}{\mathfrak}{a}{z}
\definealphabet{ms}{\mathsf}{A}{Z}
\definealphabet{ms}{\mathsf}{a}{z}

\newcommand*{\eg}{e.g.\@\xspace}
\newcommand*{\ie}{i.e.\@\xspace}

\makeatletter
\newcommand*{\etc}{%
    \@ifnextchar{.}%
        {etc}%
        {etc.\@\xspace}%
}
\makeatother

\newtheorem{assum}{\textbf{A}\hspace{-2pt}}
\crefname{assum}{A\hspace{-2pt}}{A\hspace{-2pt}}

\allowdisplaybreaks

\icmltitlerunning{Scaffold with Stochastic Gradients: New Analysis with Linear Speed-Up}

\begin{document}

\twocolumn[
\icmltitle{Scaffold with Stochastic Gradients: New Analysis with Linear Speed-Up}

\icmlsetsymbol{equal}{*}

\begin{icmlauthorlist}
\icmlauthor{Paul Mangold}{polytechnique}
\icmlauthor{Alain Durmus}{polytechnique}
\icmlauthor{Aymeric Dieuleveut}{polytechnique}
\icmlauthor{Eric Moulines}{polytechnique}
\end{icmlauthorlist}

\icmlaffiliation{polytechnique}{Ecole Polytechnique, CMAP, UMR 7641, France}

\icmlcorrespondingauthor{Paul Mangold}{paul.mangold@polytechnique.edu}

\icmlkeywords{Machine Learning, ICML, Optimization, Federated Learning}

\vskip 0.3in
]

\printAffiliationsAndNotice{}  %

\begin{abstract}
This paper proposes a novel analysis for the Scaffold algorithm, a popular method  for dealing with data heterogeneity in federated learning. While its convergence in deterministic settings---where local control variates mitigate client drift---is well established, the impact of stochastic gradient updates on its performance is less understood.
To address this problem, we first show that its global parameters and control variates define a Markov chain that converges to a stationary distribution in the Wasserstein distance.
Leveraging this result, we prove that Scaffold achieves linear speed-up in the number of clients up to higher-order terms in the step size.
Nevertheless, our analysis reveals that Scaffold retains a higher-order bias, similar to FedAvg, that does not decrease as the number of clients increases. 
This highlights opportunities for developing improved stochastic federated learning algorithms.

\end{abstract}

\section{Introduction}
\label{sec:intro}
This paper focuses on the federated optimization, in which $ \nagent $ agents collaborate  to solve  a problem of the form
\begin{align}
    \label{pb:smooth-fl}
    \paramlim \in \argmin_{\param \in \rset^d} 
       f(\theta) =  \frac{1}{\nagent}
        \sum_{c=1}^\nagent
        \nf{c}{\param}
        \eqsp,
\end{align}
where for each $c \in \iint{1}{\nagent}$, $\smash{\nf{c}{\param} = \PE[ \nfs{c}{\param}{ \locRandState{c}{} } ]}$ is a local risk function of agent $c$ for some function $\smash{ (z_{(c)},\theta) \mapsto F_{(c)}^{z_{(c)}}(\theta) }$ and local observation $\locRandState{c}{}  $ with distribution \( \distRandState{c} \) over a measurable space \(\smash{(\msZ, \mcZ)} \).

One of the most popular methods for solving \eqref{pb:smooth-fl} is \FedAvg \cite{mcmahan2017communication}, where clients perform multiple local stochastic gradient updates, and send their updated parameters to a central server, that aggregates them.
Although \FedAvg's local training reduces the number of communications in certain federated learning settings, client heterogeneity can significantly hinder its convergence.
When the number of local iterations increases, clients lean towards their local minimums, which differ from the global one due to heterogeneity.
This phenomenon, called \emph{client drift}, can induce large bias in \FedAvg.
To control this bias, clients must communicate frequently, requiring at least $\Omega(1/\epsilon)$ communication rounds to reach a mean squared error of order $\epsilon^2$ when objective functions are strongly-convex \citep{karimireddy2020scaffold}.

A key method for mitigating client drift is \Scaffold \cite{karimireddy2020scaffold}.
In this algorithm, each client updates its local model by performing gradient updates, adjusted using local control variates. 
After each aggregation step, clients update their local control variates based on the global model received from the server, effectively removing heterogeneity bias.
\Scaffold was first theoretically studied by \citet{karimireddy2020scaffold}, reducing communications from $O(1/\epsilon)$ to $O(\log(1/\epsilon))$ for strongly-convex objectives, where $\epsilon >0$ is a precision target.
Later, \citet{mishchenko2022proxskip,hu2023tighter} proved that (a variant of) \Scaffold reaches $O(\log(1/\epsilon))$ communication cost with an improved dependence on the problem's condition number.
Unfortunately, in all these results, \emph{the number of gradients computed by each client does not decrease with the number of clients}.\footnote{We note that, although \citet{karimireddy2020scaffold} obtain such speed-up, they do using a global step size, which significantly departs from common practice. See discussions in \Cref{rmq:scaffold-without-global-it}.}
Yet, a fundamental promise of federated learning is to reduce training cost through collaboration, a phenomenon called linear speed-up \citep{yu2019linear}.

{\small
\begin{table*}[t]
\begin{threeparttable}
    \caption{Communications and local iterations required for \Scaffold to reach $\PE[ \norm{ \globparam{t} - \paramlim }^2 ] \le \epsilon^2$, for $\epsilon > 0$, according to multiple analyses of \Scaffold with stochastic gradients for $\strcvx$-strongly convex and $\lip$-smooth functions.}
    \label{table:results}
\setlength\tabcolsep{0pt}
\def\arraystretch{1.2}
    \centering
    \begin{tabular*}{\linewidth}{@{\extracolsep{\fill}}ccccccc}
    \toprule
         &  Communication & Local Iterations & Linear Speed-Up & 
         \multicolumn{2}{c}{Acceleration$^{(3)}$}
         & 
         General objective \\[-0.3em]
         & & & & \small ~~Det. & \small Sto. &
          \\[-0.2em]
       \midrule
       \citet{karimireddy2020scaffold}$^{(1)}$
       & $O(\log(1/\epsilon))$
       & $O(1/\epsilon^2 )$
       & \hphantom{$^{(1)}$}\newcrossmark$^{(1)}$ & \newcrossmark  & \newcheckmark  & \newcheckmark 
       \\
       \citet{mishchenko2022proxskip}$^{(2)}$
       & $O(1 / \epsilon)$
       & $O(1 / \epsilon)$
       & \newcrossmark & \newcheckmark  & \newcrossmark  & \newcheckmark 
       \\      \citet{hu2023tighter}$^{(2)}$
       & $O(\log(1/\epsilon))$
       & $O(1/\epsilon^2)$
       & \newcrossmark & \newcheckmark & \newcheckmark & \newcheckmark 
       \\\citet{mangold2024scafflsa}
       & $O(\log(1/\epsilon))$
       & $O(1/ \nagent \epsilon^2)$
       & \newcheckmark & \newcrossmark & \newcheckmark & \hphantom{$^{(4)}$}\newcrossmark$^{(4)}$ 
       \\
       \midrule
       \textbf{Ours}
       & $O(\log(1/\epsilon))$
       & $O(1 / \nagent \epsilon^2)$
       & \newcheckmark & \newcrossmark & \newcheckmark & \newcheckmark 
       \\
    \bottomrule
    \end{tabular*}
    \begin{tablenotes}
      \small
      \item ${(1)}$ they obtain a linear speed-up by introducing a global step size: in practical implementations, there is no global step size and their analysis loses linear speed-up (see~\Cref{rmq:scaffold-without-global-it}); ${(2)}$ based on a stochastic communication scheme; ${(3)}$ acceleration means that the algorithm benefits from local steps, when gradients are deterministic (Det.), or stochastic (Sto.); (4) only holds for quadratic functions.
      \end{tablenotes}
      \vspace{-0.5em}
\end{threeparttable}
\end{table*}
}

In this paper, we show for the first time, to our knowledge, that \emph{S{\scriptsize CAFFOLD} achieves linear speed-up}.
To this end, we develop a novel point of view on \Scaffold, showing that its global iterates and control variates jointly form a Markov chain, similarly to \SGD \citep{dieuleveut2020bridging} and \FedAvg \citep{mangold2024refined}.
For strongly-convex and smooth objectives, we show that this Markov chain converges geometrically to a unique stationary distribution.
A careful examination of the pairwise covariances of the global parameters and control variate reveals that, in this stationary distribution, \Scaffold's global parameters' variance reduces linearly with the number of clients, up to a maximum number of clients.
We then leverage this result to give a new non-asymptotic convergence rate for \Scaffold, highlighting the speed-up property.
Our analytical framework also allows to derive first-order (in the step size) expansions of this covariances, and unveils that, despite its bias-correction mechanism, \Scaffold's global iterates still suffer from a small bias.
Our contribution are:%
\begin{itemize}[leftmargin=8pt, itemsep=0pt]
    \item \textbf{\Scaffold's iterates converge.} The global iterates and control variates of \Scaffold form a Markov chain that converges linearly to a stationary distribution in Wasserstein distance, with a faster rate with more local steps.

    \item \textbf{\Scaffold has linear speed-up.}
    We give a new non-asymptotic convergence rate for \Scaffold, showing that the number of gradients computed by each client to reach a given precision decreases linearly with the number of clients (up to a limit that we characterize).
    To our knowledge, this is the first result of this kind for \Scaffold; see \Cref{table:results} for a comparison with existing works.

    \item \textbf{\Scaffold is still biased.}
    We give first-order expansions, in the step size, of the covariances of \Scaffold's iterates in the stationary distribution.
    Surprisingly, while \Scaffold corrects \emph{heterogeneity bias}, it still suffers from another bias due to its stochastic updates. 
\end{itemize}

\textbf{Notations.}
We denote by \( \nabla \fw \) the gradient of a differentiable function \( \fw \colon \rset^d \to \rset \).
If \( \fw \) is \( i \)-times differentiable for \( i \geq 1 \), we denote its \( i \)-th derivative by \( \nabla^i \fw \).  
We use \( \pscal{\cdot}{\cdot} \) to denote the Euclidean dot product.
Vectors are columns, and their Euclidean norm is $\norm{\cdot}$.
For matrices, $\norm{\cdot}$ is the operator norm, \( \Id \) is the identity matrix in \( \rset^d \).  
For two matrices \( A, B \), we define the Kronecker-type linear operator \( A \otimes B \) as $A \otimes B : M \mapsto A M B$
where \( A, M, \) and \( B \) have compatible dimensions for multiplication. For a tensor \( X \), we denote by \( X^{\otimes k} \) its \( k \)-th tensor power.  
For a sequence of matrices \( M_1, \dots, M_k \), we define their ordered product as  
$\smash{\prod_{i=1}^k M_i = M_k M_{k-1} \cdots M_1}$. 
Let \( \mathcal{B}(\rset^d) \) be the Borel \( \sigma \)-algebra of \( \rset^d \).
For two probability measures $\rho_1,\rho_2$ over $\mathcal{X}$ such that {\small $
\int \rho_i(\rmd \bigX) \| \bigX \|^2_{\Lambda} < \infty$}, $i=1,2$, we define the second-order Wasserstein distance as $\wasserstein^2(\rho_1, \rho_2) = \inf_{\xi \in \Pi(\rho_1, \rho_2)} \int \norm{ \bigX - \bigX' }[\Lambda]^2 \xi(\rmd \bigX, \rmd \bigX')$, with $\Pi(\rho_1, \rho_2)$ the set of probability measures on $\mathcal{X} \times \mathcal{X}$ such that $\xi(\msA \times \mathcal{X}) = \rho_1(\msA)$, $\xi(\mathcal{X} \times \msA) = \rho_2(\msA)$ for $\msA \in \mathcal{B}(\mathcal{X})$.

\section{Federated Learning and \Scaffold}
\label{sec:assumptions}
\label{sec:scaffold-algo}

The main challenge in federated learning arises from the fact that each client \( c \in \iint{1}{\nagent} \) only has access to its own local function \( \nfw{c} \), rather than the full sum in \eqref{pb:smooth-fl}. Since these functions typically differ across clients, this induces \emph{heterogeneity}, making optimization more complex. 

\paragraph{Assumptions.}
Throughout this paper, we consider the following assumptions.
The first assumptions \Cref{assum:strong-convexity}, \Cref{assum:smoothness} and \Cref{assum:third-derivative} define the regularity of the local objective functions.
\begin{assum}[Strong Convexity]
\label{assum:strong-convexity}
For every $c \in \iint{1}{\nagent}$, the function $\nfw{c}$ is twice differentiable and $\strcvx$-strongly-convex.
In particular, we have
$\hnf{c}{\param} \succcurlyeq \strcvx \Id$ for any $\param \in \rset^d$.
\end{assum}
\begin{assum}[Smoothness]
\label{assum:smoothness}
For every $c \in \iint{1}{\nagent}$ and $z \in \msZ$, the function $\smash{\nfsw{c}{z}}$ is twice differentiable and $\lip$-smooth.
In particular, we have $\hnfs{c}{\param}{z} \preccurlyeq \lip \Id$ for $\param \in \rset^d$.
\end{assum}

\begin{assum}[Third Derivative]
\label{assum:third-derivative}
For every $c \in \iint{1}{\nagent}$, $z \in \msZ$, the function $\nfw{c}$ is thrice differentiable with bounded third derivative, \ie, there exists $\thirdlip \geq 0$ such that for any $u \in \rset^d$ and  $\theta \in \rset^d$, $\norm{ \hhnf{c}{\param} u^{\otimes 2} } \le \thirdlip \norm{ u }^2$.
\end{assum}
These assumptions are classical in stochastic optimization \citep{nesterov2013introductory,dieuleveut2016nonparametric}. 
We discuss the main consequences of \Cref{assum:strong-convexity} and \Cref{assum:smoothness} in \Cref{sec:monotonicity-cocoercivity}. 

To measure heterogeneity of the problem, we rely on the gradients and Hessians of local functions at the solution.
\begin{assum}[Heterogeneity Measure] 
\label{assum:heterogeneity}
  There exist $\heterboundgrad, \heterboundhess \ge 0$ such that, with $\paramlim$ as in \eqref{pb:smooth-fl}
  \begin{equation*}
  \textstyle
      \frac{1}{\nagent} \sum_{c=1}^\nagent 
      \norm{ \nabla^i \nf{c}{\paramlim} - \nabla^i \fw(\paramlim) }^2
    \le \heterboundw_{i}^2
    \eqsp
    \text{ for } i \in \{1 , 2\}
    \eqsp.
  \end{equation*}
\end{assum}
Finally, for a parameter $\param \in \rset^d$ and  $z \in \msZ$, we define the stochastic part of the gradient and its covariance as
\begin{align}
\label{eq:def-epsilon-noise}
\updatefuncnoise{c}{z}{\param}
& \eqdef \gnfs{c}{\param}{z} - \gnf{c}{\param}
\eqsp,
\\
\label{eq:covariance-gradient-noise}
\locnoisecovmat{c}{\param} 
& \eqdef \PE\big[ \updatefuncnoise{c}{z}{\param}\updatefuncnoise{c}{z}{\param}^\top \big] \eqsp.
\end{align}
We assume in \Cref{assum:smooth-var} that $\updatefuncnoise{c}{z}{\param}$ has bounded sixth moment.
\begin{assum}[Gradient's Variance]
\label{assum:smooth-var}
There exist constants $\optvar, \smoothcstvar \ge 0$ such that for $\param \in \rset^d$, $p \in \{1,2,3\}$, and $c\in \iint{1}{\nagent}$, \begin{align*}
    \PE^{1/p}\big[ \norm{ \updatefuncnoise{c}{\locRandStatew{c}}{\param } }^{2p} \big] 
    & \le 
    \optvar
    + \smoothcstvar \norm{ \param - \paramlim }^2 
    \eqsp,
\end{align*}
where $\locRandStatew{c}$ has values in $\msZ$ and distribution  $\distRandState{c}$.

\begin{algorithm}[tb]
\setstretch{1.1}
\caption{\Scaffold}
\label{algo:scaffold}
\textbf{Input}: initial $\globparam{0} \in \rset^d$ and $\cvar{1}{0}, \dots, \cvar{\nagent}{0} \in \rset^d$, 
step size $\step > 0$, number of rounds $\nrounds > 0$, number of clients $\nagent > 0$, number of local steps $\nlupdates > 0$ 
\begin{algorithmic}[1] %
\FOR{$t=0$ to $\nrounds-1$}
\FOR{$c=1$ to $\nagent$}
\STATE Initialize $\locparam{c}{t,0} = \globparam{t}$
\FOR{$h=0$ to $\nlupdates-1$}
\STATE Receive random state $\locRandState{c}{t,h+1}$
\STATE Set $\locparam{c}{t,h+1} = \locparam{c}{t,h} - \step \Big\{ \gnfs{c}{\locparam{c}{t,h}}{\locRandState{c}{t,h+1}} + \cvar{c}{t} \Big\}$
\ENDFOR
\ENDFOR
\STATE
Update: $\globparam{t+1} = \frac{1}{\nagent} \sum_{c=1}^{\nagent} \locparam{c}{t,\nlupdates}$
\STATE
Update: $\cvar{c}{t+1} = \cvar{c}{t} + \frac{1}{\step \nlupdates} ( \locparam{c}{t,\nlupdates} - \globparam{t+1})$
\ENDFOR
\STATE \textbf{Return: } $\globparam{\nrounds}$    
\end{algorithmic}
\end{algorithm}

\end{assum}

\textbf{\FedAvg.}
A now very popular algorithm to solve \eqref{pb:smooth-fl} is Federated Averaging (\FedAvg) \citep{mcmahan2017communication}.
This method leverages local training to reduce communications, by letting each client perform a number of local stochastic gradient updates. 
Each final iterate of these updates are then sent to a central server, which aggregates the model received by all clients. More precisely, \FedAvg defines a sequence of global iterates $(\vartheta_t)_{t\in\nset}$ as follows.
At a global time step $t \ge 0$, each client $c \in \iint{1}{\nagent}$ performs $\nlupdates > 0$ local iterations, starting from $\vlocparam{c}{0}=\vartheta_t$, where $\vartheta_t$ is the current global parameter received from the server. This writes as, for $h \in \iint{0}{\nlupdates-1}$, \begin{align*}
\vlocparam{c}{t,h+1} = \vlocparam{c}{t,h} - \step \gnfs{c}{\vlocparam{c}{t,h}}{\locRandState{c}{t,h+1}\!\!\!}
\eqsp,
\end{align*}
where $\{\locRandState{c}{t,h+1}\}_{h=1}^H$ are i.i.d.~random variables independent among clients and from the previous iterations,  with distribution $\nu_{(c)}$.
After these local updates, the parameters are aggregated by the server $\smash{\vartheta^{t+1} = \nagent^{-1} \sum_{c=1}^\nagent \vlocparam{c}{t,\nlupdates}}$.

\paragraph{\Scaffold.}
The \Scaffold algorithm \citep{karimireddy2020scaffold} uses control variates to mitigate \emph{client drift} by replacing the local gradient updates of \FedAvg  for $h \in \iint{0}{\nlupdates-1}$,  by
\begin{align}
\label{eq:update-loc-scaffold}
\locparam{c}{t,h+1} = \locparam{c}{t,h} - \step \Big( \gnfs{c}{\locparam{c}{t,h}}{\locRandState{c}{t,h+1}} + \cvar{c}{t} \Big)
\eqsp.
\end{align}
These parameters are then aggregated by a central server as in \FedAvg: $\theta^{t+1} = \nagent^{-1} \sum_{c=1}^\nagent \locparam{c}{t,\nlupdates}$.
After aggregation, each client $c$ locally updates its control variate as 
\begin{align}
\label{eq:update-cvar-scaffold}
\cvar{c}{t+1} = \cvar{c}{t} + \frac{1}{\step \nlupdates} ( \locparam{c}{t,\nlupdates} - \globparam{t+1})
\eqsp.
\end{align}
We give the pseudo-code of this algorithm in~\Cref{algo:scaffold}.
Learning $\cvar{c}{t}$ 
corresponds to estimating a linear correction of the gradient of the local functions so that the corrected gradient is zero at $\paramlim$. The \emph{ideal control variate} for client $c$ is thus $\cvarlim{c} = - \gnf{c}{\paramlim}$, as this correction ensures that all clients converge toward the same optimum.

\begin{remark}
\label{rmq:scaffold-without-global-it}
In this paper, we aim to study the \Scaffold algorithm as it is commonly used.
Thus, contrarily to \citep{karimireddy2020scaffold,yang2021achieving}, we do not consider two-sided step sizes.
While this yields the desired linear speed-up by dividing the local step size by $\sqrt{\nagent}$, and increasing the global one, it essentially reduces the algorithm to mini-batch \SGD, and does not give much insights on \Scaffold itself.
Thus, we consider in \Cref{table:results} the rate of \citet{karimireddy2020scaffold} without global step size.
\end{remark}

\section{Related Work}
\label{sec:related-work}
\paragraph{Analysis of \FedAvg.} Early analyses of \FedAvg were conducted under homogeneity assumptions on the gradients~\citep{stich2019local, wang2018cooperative, haddadpour2019convergence,patel2019communication, yu2019parallel, li2019communication, woodworth2020minibatch}. Subsequent studies have shown that \FedAvg exhibits a fundamental bias in heterogeneous settings~\citep{li2019convergence,malinovskiy2020local, charles2021convergence, pathak2020fedsplit,karimireddy2020scaffold}: due to client drift, the iterates of \FedAvg do not converge to the true solution \( \theta_{\star} \), but to a biased limit point.

In fact, even in homogeneous settings, \FedAvg remains biased due to its stochastic updates.
This appears in the analyses of \citet{khaled2020tighter, woodworth2020local, glasgow2022sharp, wang2024Unreasonable}.

\paragraph{Heterogeneity mitigation.}
\citet{karimireddy2020scaffold} proposed \Scaffold, which reduces client drift with control variates, alike variance reduction methods~\cite{schmidt2017minimizing}, and proved its convergence. Subsequently, \citet{mitra2021linear,gorbunov2021local} established similar  rates 
in the smooth and strongly convex case. 
However, in all these works, the number of communication rounds required to achieve mean squared error of order $\epsilon^2$, scales as $O(\kappa \log(1/\epsilon))$, where $\kappa$ is the problem's condition number.

\citet{mishchenko2022proxskip} then introduced \ProxSkip, which, in a deterministic setting, achieves accelerated communication complexity, reducing the average number of communication rounds to $O(\sqrt{\kappa}\log(\epsilon^{-1}))$ for reaching MSE of order $\epsilon^2$.
However, when gradients are stochastic, their analysis requires $O(1/\epsilon)$ rounds.
Later on, \citet{hu2023tighter} fixed this, reaching $O(\sqrt{\kappa} \log(1/\epsilon)$ rounds even in the stochastic setting.
Nonetheless, neither of these analyses achieve linear speed-up with respect to the number of clients. %
Several extensions of these methods have been proposed \citep{malinovsky2022variance, condat2022provably, condat2022randprox, sadiev2022communication}. However, the sample complexity results established in these works do not exhibit linear speed-up either.
A notable exception is the work of \citet{mangold2024scafflsa}, who achieves linear speed-up for \Scaffold \emph{for quadratic objectives}; they consider an extended version of \Scaffold for linear approximation, named \Scafflsa, requiring $O(\kappa^2 \log(1/\epsilon))$ communications with a number of local updates scaling in $O(1/\nagent \epsilon^2)$, effectively achieving linear speed-up.
In this work, we present a more general analysis that holds beyond the quadratic setting.

\paragraph{SGD in a Markovian setup.}
Unlike SGD with a diminishing step size, which converges to the true optimum under convexity assumptions, constant step-size \SGD\ does not converge pointwise and instead oscillates around $\theta_{\star}$~\cite{chee2018convergence}, introducing an inherent bias. To address this problem,
\citet{dieuleveut2020bridging}, following a stream of works by \cite{pflug1986stochastic,fort1999asymptotic,bach2013non}, analyze SGD with a constant step size as a Markov chain, leveraging randomly perturbed dynamical systems to characterize its convergence and limiting behavior. Recently, \citet{mangold2024refined} proposed to view \FedAvg's iterates as a Markov chain.
They establish that \FedAvg's iterates converge towards a unique stationary distribution, and give explicit first-order expansion of the bias in $O(\step \nlupdates)$.  
This bias decomposes into two  components: one due to heterogeneity, and one due to stochasticity of the local gradients.
Remarkably, this second bias vanishes when optimizing quadratic functions.

\section{New Convergence Rate for \Scaffold}
\label{sec:convergence-scaffold}
In this section, we present our first main theoretical contribution: \emph{S{\scriptsize CAFFOLD} achieves linear speed-up with respect to the number of agents}.
To establish this result, we introduce a new analytical framework for the study of \Scaffold.

First, we show in \Cref{sec:convergence-global-iterates} that the global iterates and control variates of \Scaffold\ define a Markov chain. We then establish that this Markov chain geometrically converges to a unique stationary distribution in Wasserstein distance. 
Next, we analyze the covariance structure of this stationary distribution in \Cref{sec:bound-variance-stationary}. The detailed analysis of this covariance matrix  provides important insights into the behavior of \Scaffold\ in the stationary regime.
Finally, based on these results, we derive a non-asymptotic convergence rate in \Cref{sec:non-asymptotic-rate}, proving the linear speed-up for a range of step-sizes and horizons.

\subsection{Convergence of Global Iterates}
\label{sec:convergence-global-iterates}
\textbf{Iterates of \Scaffold.} 
We define the following operators, that generate the iterates of \Scaffold.
For a  value $\param \in \rset^d$, define the local update operator on client~$c$ as
\begin{equation*} 
\locstoscafop{c}{}{\param}{\cvar{c}{}}{\randStatez{c}{}}
=
\param - \step \{ \gnfs{c}{\param}{\randStatez{c}{}} + \cvar{c}{} \} \eqsp,
\end{equation*}
for $\randStatez{c}{} \in \msZ$.
Set {$\locstoscafop{c}{0}{\param}{\cvar{c}{}}{z} = \param$} and define recursively the local parameter updates 
\begin{align*}
& \locstoscafop{c}{h+1}{\param}{\cvar{c}{}}{\randStatez{c}{1:h\!+\!1}}
\! = \!
\locstoscafop{c}{}{
\locstoscafop{c}{h}{\param}{\cvar{c}{}}{\randStatez{c}{1:h}}
}{\cvar{c}{}}{\randStatez{c}{h\!+\!1}}
~,
\end{align*}
\normalsize where {$\randStatez{c}{1:h}= [\randStatez{c}{1},\dots, \randStatez{c}{h}]$}, for $c \in [\nagent]$ and $h \in [\nlupdates]$. %
This allows to define the global update operator
\begin{align*} \textstyle
\globstoscafop{\nlupdates}{\param}{\cvar{1:\nagent}{}}{\randStatez{1:\nagent}{1:\nlupdates}}
=
\frac{1}{\nagent} \sum_{c=1}^\nagent
\locstoscafop{c}{\nlupdates}{\param}{\cvar{c}{}}{\randStatez{c}{1:\nlupdates}}
\eqsp,
\end{align*}
Similarly, for $\param \in \rset^d$, we define the operator that updates the control variates as
 \begin{multline*}
\scafopcv{c}{\nlupdates}{\cvar{c}{}}{\param}{\randStatez{1:\nagent}{1:\nlupdates}} =
\cvar{c}{}
\\ 
~~~+ \frac{1}{\step \nlupdates}
\big(
\locstoscafop{c}{\nlupdates}{\param}{\cvar{c}{}}{\randStatez{c}{1:\nlupdates}}
-
\globstoscafop{\nlupdates}{\param}{\cvar{1:\nagent}{}}{\randStatez{1:\nagent}{1:\nlupdates}}
\big)
\eqsp.
\end{multline*}
\normalsize %
Thus, we can define the update of the \Scaffold algorithm
\begin{multline*}
\textstyle
\opscaffold \colon
\big( \param, \cvar{1:\nagent}{}; \randStatez{1:\nagent}{1:\nlupdates} \big)
\mapsto 
\big( \globstoscafop{\nlupdates}{\param}{\cvar{1:\nagent}{}}{\randStatez{1:\nagent}{1:\nlupdates}},\\
\textstyle
\scafopcv{1}{\nlupdates}{\cvar{1}{}}{\param}{\randStatez{1:\nagent}{1:\nlupdates}}, \dots, \scafopcv{\nagent}{\nlupdates}{\cvar{\nagent}{}}{\param}{\randStatez{1:\nagent}{1:\nlupdates}})  \big)
\eqsp.
\end{multline*}
Note that for all $\randStatez{1:\nagent}{1:\nlupdates}$, $\opscaffold(\cdot,\randStatez{1:\nagent}{1:\nlupdates})$  is a mapping from
\begin{equation*}
\mathcal{X} = \{ (\bigX_{(0)},  \dots, X_{(\nagent)} ) \in \rset^{(\nagent+1)d} \, :\, \textstyle{\sum_{c=1}^\nagent} \bigX_{(c)} = 0
\} \eqsp,
\end{equation*}
into itself.
We equip $\mathcal{X}$ with the norm $\norm{ \bigX }[\Lambda]^2 = \pscal{ \bigX }{ \Lambda \bigX }$, where $\Lambda = (\Id, \tfrac{\step^2\nlupdates^2}{\nagent} \Id, \dots, \tfrac{\step^2\nlupdates^2}{\nagent} \Id)$, or more explicitly,
\begin{align}
\label{eq:def-lambda-norm-main}
\norm{ \bigX}[\Lambda]^2
& =
\norm{ \bigX_{(0)} }^2
+ \frac{\step^2 \nlupdates^2}{\nagent} 
\sum_{c=1}^\nagent
\norm{ \bigX_{(c)} }^2
\eqsp.
\end{align}
With these notations, the \Scaffold\ updates of the parameters and the control variates---see \Cref{algo:scaffold}---writes 
\begin{equation}
\label{eq:scaffold-update}
\bigX^{t+1}= \opscaffold\Big(\bigX^t ; \locRandState{1:\nagent}{t+1,1:\nlupdates}\Big) 
~,
\end{equation}
where $\bigX^t= [\globparam{t}\!,\cvar{1}{t}, \dots, \cvar{\nagent}{t}]$ and $\{\locRandState{1:\nagent}{t,1:\nlupdates}\}_{t\in\nset}$ is an \iid\ sequence with $\locRandState{c}{t,h} \sim \distRandState{c}$ for $c \in \iint{1}{\nagent}$ and $h \in \iint{0}{\nlupdates}$.

\paragraph{\Scaffold's iterates as a Markov chain.}
\Scaffold\ updates form an iterated random function, a specific class of Markov chains that have been extensively studied (see \citet{diaconis1999iterated} and the references therein). The Markov property is clear: given the present state of the $\bigX^t= (\globparam{t},\cvar{1}{t}, \dots, \cvar{\nagent}{t})$, the conditional distribution of the future state does not depend on the past. 
Hence \Scaffold's global iterates define a time-homogeneous Markov chain on $\mathcal{X}$ equipped with its Borel $\sigma$-algebra $\mathcal{B}(\mathcal{X})$. We denote by $\markovkernel$  the corresponding Markov kernel on $\mathcal{X}$. 
We define, for $t \geq 1$, the iterates of $\markovkernel$ as $\markovkernel^t$.
For any probability measure $\rho$ on $\mathcal{X}$ and $t \in \nset$, the distribution of \Scaffold's iterates $\bigX^t$ started from $\bigX^0 \sim \rho$ is $\rho \markovkernel^t$.
We show below that the iterates of \Scaffold converge to a unique stationary distribution.
This requires a contraction in average (see \citet{diaconis1999iterated}, Theorem~1): 
the next lemma shows that $\opscaffold$ defines a contractive map over $\mathcal{X}$.
\begin{restatable}{lemma}{contractscaffoldglobalupdate}
\label{lem:contraction-scaffold-global-update}
\label{sec:contraction-coupling}
Assume \Cref{assum:strong-convexity} and \Cref{assum:smoothness}.
Let $\randStatew = \locRandState{1:\nagent}{1:\nlupdates}$ be i.i.d.~random variables satisfying \Cref{assum:smooth-var}.
Let the step size $\step>0$ and number of local updates $\nlupdates >0$ satisfy $\step \leq 1/(2\lip)$ and $\step \nlupdates (\lip + \strcvx) \le 1$.
Then, for any $\param, \param' \in \rset^d$ and $\{ \cvar{c}{}, \cvar{c}{\prime} \}_{c=1}^{\nagent} \in \rset^d$ such that $\smash{\sum_{c=1}^\nagent \cvar{c}{} = \sum_{c=1}^\nagent \cvar{c}{\prime} = 0}$, it holds that
\begin{align}
\!\!\!\PE\Big[\norm{ \opscaffold(\bigX; \randStatew) \!- \opscaffold(\bigX'; \randStatew) }[\Lambda]^2\Big] 
\!\le\!
\Big(1 \!-\! \frac{\step \strcvx}{4} \Big)^{\!\nlupdates} \norm{ \bigX \!-\! \bigX' }[\Lambda]^2
\eqsp,
\end{align}
with $\bigX \!=\! ( \param, \cvar{1}{}, \dots, \cvar{\nagent}{} )$, and $\bigX' \!=\! ( \param', \cvar{1}{\prime}, \dots, \cvar{\nagent}{\prime} )$.
\end{restatable}
We prove this lemma in \Cref{sec:proof-contract-scaffold-update}.
A major consequence of this lemma is that \Scaffold's iterates and control variates converge to a unique stationary distribution.

\begin{restatable}{theorem}{convergencescaffoldstatdist}
\label{thm:convergence-scaffold-stat-dist}
Assume \Cref{assum:strong-convexity}, \Cref{assum:smoothness}, and \Cref{assum:smooth-var}. 
Let $\step > 0$, $\nlupdates > 0$, such that $\step \leq 1/(2\lip)$ and $\step \nlupdates (\lip + \strcvx) \le 1$.
Let $\smash{\{\bigX^t \}_{t=0}^\infty}$, with $\smash{\bigX^t = (\globparam{t}, \cvar{1}{t}, \dots, \cvar{\nagent}{t})}$, be S{\scriptsize CAFFOLD}'s iterates with step size $\step$ and $\nlupdates$ local steps and $\bigX^0 \sim \rho$, where $\rho$ is a probability measure on $\mathcal{X}$ such that  $\int  \! \| \bigX \|^2_{\Lambda} \rho(\rmd \bigX) < \infty$.
Then, the distribution $\rho \markovkernel^t$ of $\bigX^t$ converges to a unique stationary distribution $\statdist{\step, \nlupdates}$ satisfying
$\int  \| \bigX \|^2_{\Lambda} \statdist{\step,\nlupdates}(\rmd \bigX) < \infty$, and for any  $t \in \nset$,
  \begin{align*}
    \wasserstein^2(\rho \markovkernel^{t}, 
    \statdist{\step, \nlupdates})
    & \le
      \left(1 - \frac{\step \strcvx}{4}\right)^{\nlupdates t}
      \wasserstein^2(\rho, \statdist{\step, \nlupdates})
      \eqsp.
  \end{align*}
\end{restatable}
We prove this theorem in \Cref{sec:proof-contract-scaffold-update}.
In the following, we indifferently write $\statdist{\step,\nlupdates}(\rmd \param, \rmd \Cvarw)$ and $\statdist{\step,\nlupdates}(\rmd \bigX)$.

\Cref{thm:convergence-scaffold-stat-dist} shows that the Markov kernel $\markovkernel$ is geometrically ergodic in $2$-Wasserstein distance. Moreover, the distribution of $\bigX^t$ converges to the limiting distribution $\statdist{\step, \nlupdates}$ at a linear rate $(1 - \step \mu/4)$, with the exponent given by the number of \textit{effective} steps $\nlupdates \times t$. 
As with the deterministic algorithm, for a given step size $\step$, a larger number of local steps $\nlupdates$ speeds up the convergence to stationarity. We will show below that it leads to additional bias.
Define the optimal vector
$ \bigXlim = (\paramlim, \cvarlim{1}, \dots, \cvarlim{\nagent})$, where the optimal control variates are given by $\cvarlim{c} = - \gnf{c}{\paramlim}$.

\begin{restatable}{lemma}{lemdescentnoise}
\label{lem:descent-noise}
Assume \Cref{assum:strong-convexity}, \Cref{assum:smoothness}, \Cref{assum:smooth-var}. 
Let $\randStatew = \locRandState{1:\nagent}{1:\nlupdates}$ be  i.i.d. random variables satisfying \Cref{assum:smooth-var}.
Assume the step size $\step$ and the number of local updates $\nlupdates$ satisfy $\step \nlupdates (\lip + \strcvx) \leq 1$.  Then, for all 
$\param \in \mathbb{R}^d$ and $\{ \cvar{c}{}\}_{c=1}^\nagent \subset \mathbb{R}^d$ such that
$\sum_{c=1}^\nagent \cvar{c}{} = 0$,
\begin{align*}
& \PE\Big[\norm{ \opscaffold(\bigX; \randStatew)- \bigXlim }[\Lambda]^2 \Big]
\skipquad
\le
\left(1 - \frac{\step \strcvx}{4} \right)^\nlupdates  \norm{ \bigX - \bigXlim }[\Lambda]^2
+ 2 \step^2 \nlupdates \optvar
~,
\end{align*}
with the global iterate vector
$\bigX = (\param, \cvar{1}{}, \dots, \cvar{\nagent}{})$.
\end{restatable}
We prove this lemma in \Cref{sec:bound-global-scaffold}.
Thus, for any \( \bigX \in \mathcal{X} \), a single iteration of \Scaffold\ brings \( \bigX \) closer to a neighborhood of the optimal solution \( \bigXlim \), as long as \( \norm{\bigX - \bigXlim}[\Lambda]^2 \) is sufficiently large. In Markov chain theory, this implies that \( \norm{\bigX - \bigXlim}[\Lambda]^2 \) serves as a Foster-Lyapunov function for the kernel \( \markovkernel \). From this Foster-Lyapunov condition, we may retrieve a first rough bound on the fluctuation of the estimator around $\bigXlim$.
\begin{restatable}{theorem}{crudeboundtwox}
\label{thm:crude-bound-2-X}
Assume \Cref{assum:strong-convexity}, \Cref{assum:smoothness} and \Cref{assum:smooth-var}.
Let $\step > 0$ be the step size and $\nlupdates > 0$ the number of local updates.
Assume that $\step \le 1 / 4 \lip$ and $\step \nlupdates (\lip + \strcvx) \le 1$.
Then, for any $\nrounds > 0$ and any $\bigX^0 \in \mathcal{X}$, the iterates and control variates of S{\scriptsize CAFFOLD}, $\bigX^\nrounds = (\globparam{\nrounds}, \cvar{1}{\nrounds}, \dots, \cvar{\nagent}{\nrounds})$, satisfy the inequality
\begin{align*}
\PE&\left[ \norm{ \bigX^\nrounds - \bigXlim }[\Lambda]^2 \right]
\skipquad \le
\Big(1 - \frac{\step \strcvx}{4}\Big)^{\nlupdates \nrounds}  \norm{ \bigX^0 - \bigXlim }[\Lambda]^2
+ \frac{8 \step}{\strcvx} \optvar
\eqsp, 
\end{align*}
where $\bigXlim$ is the global optimal vector.
\end{restatable}
The proof of this theorem is given in \Cref{sec:bound-global-scaffold}.
This preliminary bound is very similar to the ones established in \citet[Lemma 14]{karimireddy2020scaffold} and \citet[Theorem~5.5]{mishchenko2022proxskip} (for \ProxSkip). We include it for completeness, to underline that a major limitation is that it does not achieve linear speedup in the number of clients. 
Nonetheless, this result is crucial to bound the higher-order terms that appear in all our subsequent analysis.
Indeed, taking $T \to \infty$, a consequence of \Cref{thm:convergence-scaffold-stat-dist}  and \Cref{thm:crude-bound-2-X} is $\int \norm{ \bigX^\nrounds - \bigXlim }[\Lambda]^2 \pi^{(\gamma,H)}(\rmd \bigX) \leq 8 \gamma \sigma^2_{\star}/\mu$, which gives the following Corollary.
\begin{corollary}
\label{cor:crude-bounds-global-main}
Assume \Cref{assum:strong-convexity}, \Cref{assum:smoothness} and \Cref{assum:smooth-var}.
Let $\randStatew = \locRandState{1:\nagent}{1:\nlupdates}$ be i.i.d. random variables satisfying \Cref{assum:smooth-var}.
Let $\step > 0$ be the step size and $\nlupdates > 0$ the number of local updates of S{\scriptsize CAFFOLD}.
Assume that $\step \lesssim 1/\lip$ and $\step \nlupdates (\lip + \strcvx) \lesssim 1$.
Then, for all $h \in \iint{0}{\nlupdates}$,  it holds that
\begin{gather*}
  \int  \norm{ \param - \paramlim }^2 \statdist{\step, \nlupdates}(\rmd \param, \rmd \Cvarw)
  \le \frac{8 \step}{\strcvx} \optvar
    \eqsp,
\end{gather*}
where $\Cvarw = (\cvar{1}{}, \dots, \cvar{\nagent}{}) \in \rset^{\nagent \times d}$.
\end{corollary}
We give a proof and a more complete version of this Corollary in \Cref{sec:bound-global-scaffold}, \Cref{cor:crude-bounds-global}. 
We may also obtain a similar bound for local updates and control variates.
\begin{restatable}{lemma}{crudeboundlocalandcvar}
\label{lem:crude-bound-local-and-cvar}
Assume \Cref{assum:strong-convexity}, \Cref{assum:smoothness}.
Let $\randStatew = \locRandState{1:\nagent}{1:\nlupdates}$ be i.i.d. random variables satisfying \Cref{assum:smooth-var}.
Assume the step size $\step$ and the number of local updates $\nlupdates$ satisfy $\step \nlupdates (\lip + \strcvx) \hidleq 1\hiddencst{/12}$. Under these conditions, for any $h\in \iint{0}{\nlupdates}$ and $c\in \iint{1}{\nagent}$, it holds that,
\begin{gather}
\label{eq:bound-local-iterate-h}  
\int\! \PE\!\left[ \norm{  \locstoscafop{c}{h}{\param}{\cvar{c}{}}{\locRandState{c}{1:h}} \!-\! \paramlim }^2 \right]   \statdist{\step, \nlupdates}(\rmd \theta, \rmd \Cvarw)
  \hidleq 
\frac{\hiddencst{28} \step\optvar}{\strcvx} 
~, \nonumberinmain
\\
\label{eq:bound-local-xi}
\int \norm{ \cvar{c}{} - \cvarlim{c}  }^2 \statdist{\step, \nlupdates}(\rmd \theta, \rmd \Cvarw)
\hidleq
\frac{\hiddencst{54} \lip \optvar}{\strcvx \nlupdates} 
\eqsp. \nonumberinmain
\end{gather}
\end{restatable}
The proof is postponed to \Cref{sec:bound:stationary}. We use \( \lesssim \) to omit numerical constants, which are provided in the full proof. Notably, for any agent $c$, the variances of \emph{local iterates} after $h\le H$ local iterations do not scale with \( 1/\nagent \). However, it is crucial to highlight that the fluctuations of the control variate scale inversely with~\( \nlupdates \). 
We also give derive analog variants of \Cref{cor:crude-bounds-global-main} and \Cref{lem:crude-bound-local-and-cvar} for moments $2, 4$, and $6$ in \Cref{lem:descent-noise-powsix}.

\subsection{Bounding the Variance of the global iterates}
\label{sec:bound-variance-stationary}

We now derive an upper bound on the variance of $\param - \paramlim$ under the stationary distribution.
In particular, we show that this variance is proportional to $1/\nagent$, up to a higher order term in the step size.
To this end, we track the relations between the covariance matrices of the global parameters and control variates, defined  for any $c,c'\in[\nagent]$ as
\begin{equation*}
\begin{aligned}
 \covparam &
\eqdef\! \int \big(\param - \paramlim\big)^{\otimes 2} \statdist{\step, \nlupdates}(\rmd \param, \rmd \Cvarw)  \eqsp,
\\
\covcvar{c,c'}
& \eqdef{}\!
\int \big( \cvar{c}{} \!- \cvarlim{c} \big)
\big( \cvar{c'}{} \!- \cvarlim{c'} \big)^\top \statdist{\step, \nlupdates}( \rmd \param, \rmd \Cvarw )
\eqsp,
\\
\covparamcvar{c}
& \eqdef{}\!
\int 
\big( \param - \paramlim \big)\big( \cvar{c}{} \!- \cvarlim{c} \big)^\top 
\statdist{\step, \nlupdates}( \rmd \param, \rmd \Cvarw )
\eqsp.
\end{aligned}
\end{equation*}
We emphasize that the parameter and control variates are inherently correlated. Local gradient noise introduced in the updates of the local parameters~\eqref{eq:update-loc-scaffold} propagates to the control variates via their update~\eqref{eq:update-cvar-scaffold}.  
We refer to \Cref{lem:expansion-cov-first-order-main} for a detailed discussion on these covariance matrices. There, we provide exact first-order expansions, offering a precise characterization of their structure and interactions.

Now, we derive an upper bound on the global parameter's covariance $\covparam$.
To this end, define
\begin{gather*}
\textstyle
\bcovparam = \norm{ \covparam }
\eqsp,
\quad
\bcovparamcvar = \frac{1}{\nagent} \sum_{c=1}^\nagent \norm{ \covparamcvar{c} }
\eqsp,
\\
\textstyle
\bcovallcvar = \frac{1}{\nagent^2} \sum_{c,c'=1}^\nagent \norm{ \covcvar{c,c'} }
\eqsp.
\end{gather*}
We also define the following quantity, related to the variance of noise added by clients during local updates,
\begin{align*}
\globboundnoise
& =
\frac{1}{\nagent} \sum_{c=1}^\nagent \sum_{h=0}^{\nlupdates-1}
\bnorm{ \int \PE[ \mathcal{C}_c^h(\param) ] \statdist{\step, \nlupdates }(\rmd \param, \rmd \Cvarw) }
\eqsp,
\end{align*}
where $\mathcal{C}_c^h(\param) = \locnoisecovmat{c}{\locstoscafop{c}{h}{\param}{\cvar{c}{}}{\locRandState{c}{1:h}}}$, and $\locnoisecovmat{c}{\param}$ is the covariance of the local gradient noise as defined in \eqref{eq:covariance-gradient-noise}.
The next lemma relates $\bcovparam$, $\bcovparamcvar$ and $\bcovallcvar$. We present it in a simplified form to highlight the main dependencies. 
\begin{lemma}
\label{lem:ineq-bound-covariances}
Assume \Cref{assum:strong-convexity}, \Cref{assum:smoothness}, \Cref{assum:smooth-var}.
Assume the step size $\step$ and the number of local steps $\nlupdates$ satisfy $\step \nlupdates(\lip+\strcvx) \lesssim 1$, then
\begin{align*}
 \step \strcvx \nlupdates \bcovparam
& \lesssim
{\step^2\nlupdates^2\lip}\bcovparamcvar
+ \step^4\nlupdates^4\lip^2\bcovallcvar
+ \frac{\step^2}{\nagent} \globboundnoise
+ \mathrm{r}^{\param} 
\eqsp,
\\
\bcovparamcvar
& \lesssim
\heterboundhess \bcovparam
+ {\step^3\nlupdates^3} \lip^2
\bcovallcvar 
+ \frac{\step}{\nagent \nlupdates} \globboundnoise
+ \mathrm{r}^{\param,\cvarw}
\eqsp,
\\
\bcovallcvar
& \lesssim
\heterboundhess^2 \bcovparam
+
\heterboundhess
\step\nlupdates\lip
\bcovparamcvar
+ \frac{1}{\nagent \nlupdates^2} \globboundnoise
+ \mathrm{r}^{\cvarw}
\eqsp,
\end{align*}
where $\heterboundhess$ is the heterogeneity coefficient defined in \Cref{assum:heterogeneity}, $\mathrm{r}^{\param}$, $\mathrm{r}^{\param,\cvarw}$, and $\mathrm{r}^{\cvarw}$ are higher-order terms.
\end{lemma}
We prove this lemma and give exact expressions in \Cref{sec:app:upper-bound-cov-mat}-\Cref{lem:ineq-bound-covariances-appendix}.
Using these inequalities, we derive the next theorem, that gives an upper bound on $\bcovparam$.
\begin{restatable}{theorem}{boundcovstationary}
\label{thm:bound-cov-stationary}
Assume \Cref{assum:strong-convexity}, \Cref{assum:smoothness}, \Cref{assum:third-derivative}, \Cref{assum:heterogeneity}, and \Cref{assum:smooth-var}.
Furthermore, assume that $\step\nlupdates\lip \heterboundhess \hidleq \strcvx\hiddencst{/10}$, $\step \nlupdates (\lip+\strcvx) \hidleq 1\hiddencst{/48}$ and ${\step \smoothcstvar} \hidleq \strcvx\hiddencst{/19}$.
Then, it holds that
\begin{align*}
\bcovparam
\hidleq
\frac{\hiddencst{10} \step }{\nagent \strcvx} \optvar 
+ \frac{\hiddencst{6 \cdot 15080}\step^{3/2} \thirdlip}{\strcvx^{5/2}}
\sqoptvar^3
+
\frac{\hiddencst{48 \cdot 600^2} \step^3 \nlupdates \thirdlip^2}{\strcvx^3} \sqoptvar^4
\eqsp.
\end{align*}
\end{restatable}
We prove this theorem in \Cref{sec:app:upper-bound-cov-mat}.
Recall that $\thirdlip$ is the upper bound on the third derivative, which is defined in \Cref{assum:third-derivative} and it vanishes in the quadratic case. We recover in such case the bound on the covariance of the parameter derived in \cite{mangold2024scafflsa}. 
A crucial feature of this result, is that the covariance of the parameters' error $\covparam$ is proportional to $\step/N$, up to higher-order terms in the step size.
To our knowledge, this is the first time the variance of \Scaffold with general objective function is shown to decrease with the number of clients.
It is in stark contrast with existing analyses of \Scaffold \citep{mishchenko2022proxskip} where variance only scales in $\step$.

\subsection{A Non-Asymptotic Rate with Linear Speed-Up}
\label{sec:non-asymptotic-rate}

We now state our main result, showing that our bounds from \Cref{sec:bound-variance-stationary} can be used to obtain non-asymptotic rates for \Scaffold.
This can be achieved by using the convergence of \Scaffold to its stationary distribution through a synchronous coupling method.

\begin{restatable}{theorem}{convergenceratescaffold}%
\label{thm:convergence-scaffold-rate}
Assume \Cref{assum:strong-convexity}, \Cref{assum:smoothness}, \Cref{assum:third-derivative}, \Cref{assum:heterogeneity}, and \Cref{assum:smooth-var}.
Furthermore, assume that $\step\nlupdates\lip \heterboundhess \hidleq \strcvx\hiddencst{/10}$, $\step \nlupdates (\lip+\strcvx) \hidleq 1\hiddencst{/48}$ and ${\step \smoothcstvar} \hidleq \strcvx\hiddencst{/19}$. 
Then, the mean squared error of S{\scriptsize CAFFOLD}'s global iterates, initialized with $\globparam{0} \in \rset^d$ and $\cvar{1}{} = \cdots = \cvar{\nagent}{} = 0 \in \rset^d$ is
\begin{align*}
\mainalign \PE\left[ \norm{ \globparam{\nrounds} - \paramlim }^2 \right] 
  \skipquad
  \restatealign
  \hidleq
\left( 1 - \frac{\step \strcvx}{4} \right)^{\nlupdates \nrounds}
\left\{ 
2\norm{  \param - \paramlim }^2
+
2\step^2 \nlupdates^2 \heterboundgrad^2
+
\frac{ \hiddencst{64} \optvar }{\lip\strcvx}
\right\}
\\ & \qquad
+ 
\frac{\hiddencst{20 d} \step }{\nagent \strcvx} \optvar 
+ \frac{ \hiddencst{12 \cdot 15080 d } \step^{3/2} \thirdlip}{\strcvx^{5/2}}
\sqoptvar^3
+
  \frac{\hiddencst{96 \cdot 600^2 d } \step^3 \nlupdates \thirdlip^2}{\nagent\strcvx^3} \sqoptvar^4
  \eqsp.
\end{align*}
\end{restatable}
To prove this theorem, we decompose $\globparam{\nrounds} - \paramlim = \globparam{\nrounds} - \stationaryparam{\nrounds} + \stationaryparam{\nrounds} - \paramlim$, where $\stationaryparam{\nrounds}$ is obtained by running \Scaffold with the same realization of noise as $\globparam{\nrounds}$ but starting from $\stationaryparam{0}$ in the stationary distribution.
We then obtain a bound on the error by bounding $\smash{\PE[ \norm{  \globparam{\nrounds} - \stationaryparam{\nrounds} }^2 ]}$ and $\smash{\PE[ \norm{ \stationaryparam{\nrounds} - \paramlim }^2 ] }$ separately, using \Cref{lem:contraction-scaffold-global-update} and \Cref{thm:bound-cov-stationary} respectively.
We give a detailed proof in \Cref{sec:non-asymptotic-rates}.

This theorem converts our asymptotic bound on \Scaffold's error in the stationary regime into a non-asymptotic bound, \emph{where the variance term scales in $1/\nagent$}, up to higher-order factors in $O(\step^{3/2} + \step^3 \nlupdates)$.
This gives the following sample and communication complexity for \Scaffold.

\begin{restatable}{corollary}{samplecommcomplexityscaffold}%
\label{cor:complexity-with-speed-up}
Let $\epsilon > 0$.
With \Cref{thm:convergence-scaffold-rate}'s assumptions, we can set $\step \lesssim{} \frac{\nagent \strcvx \epsilon^2}{\optvar}$ and $\nlupdates
\lesssim \frac{\optvar}{\nagent \lip \strcvx \epsilon^2} \min( 1, \sfrac{\strcvx}{\heterboundhess})$.
Then, if the number of clients is ${\nagent \lesssim \min( \frac{\strcvx^{2/3}}{\thirdlip^{2/3} \epsilon^{2/3}}}, \frac{\lip^{1/2}\strcvx^{1/2}}{\thirdlip \epsilon})$, S{\scriptsize CAFFOLD} guarantees $\PE[ \norm{ \globparam{\nrounds} - \paramlim }^2] \le \epsilon^2$ for $\nrounds \gtrsim 
\frac{\lip}{\strcvx} \max(1, \heterboundhess/\strcvx) 
\log\left( 
\frac{\norm{ \globparam{0} - \paramlim }^2 + \heterboundgrad^2 / \lip^2}{\epsilon^2}
\right)$, and the number of stochastic gradients computed by each client is 
\begin{align*}
\# \text{grad per client}
\lesssim \frac{\optvar \min( 1, \sfrac{\strcvx}{ \heterboundhess})}{\nagent \strcvx^2 \epsilon^2} 
\log\left( 
\frac{\psi_0}{\epsilon^2}
\right)
~.
\end{align*}
where $\psi_0 = \norm{ \globparam{0} - \paramlim }^2 + \heterboundgrad^2 / \lip^2 + \optvar/ (\lip \strcvx)$.
\end{restatable}
We prove this corollary in \Cref{sec:non-asymptotic-rates}.
This result combines two crucial features: (i) \emph{S{\scriptsize CAFFOLD} has linear speed-up} up to a given number of clients: the number of gradients computed by each client scales in $1/\nagent$; and (ii) \emph{S{\scriptsize CAFFOLD} accelerates stochastic gradients}: the number of rounds required for convergence depends logarithmically on the desired precision $\epsilon$.
In comparison, in heterogeneous settings, \FedAvg's number of communication scales polynomially in $1/\epsilon$.
To our knowledge, this is the first time that \Scaffold is proven to have linear speed-up (without relying on global step sizes), while guaranteeing acceleration with stochastic gradients.
\begin{remark}
In our analysis, we show that the number of rounds scales in $\log(1/\epsilon)$, with a multiplicative factor $\lip/\strcvx$.
Additionally, \citet{hu2023tighter} proved that this constant can be reduced to $\sqrt{\lip/\strcvx}$, but without linear speed-up in the number of clients.
It is an intriguing open question to determine whether \Scaffold can preserve this reduction from ${\lip/\strcvx}$ to $\sqrt{\lip/\strcvx}$ while guaranteeing this linear speed-up.
\end{remark}

\section{Explicit Expression for Bias and Variance}
\label{sec:scaffold-bias-variance}
\begin{figure*}[t]
   \captionsetup[subfigure]{justification=centering}
    \centering
    \raisebox{2.5\height}{\rotatebox[origin=c]{90}{MSE}}
    \begin{subfigure}[t]{0.24\textwidth}
        \centering
        \includegraphics[width=\linewidth]{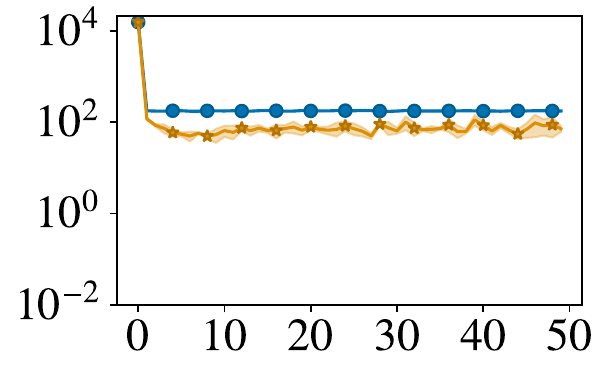}
    \end{subfigure}~%
    \begin{subfigure}[t]{0.24\textwidth}
        \centering
        \includegraphics[width=\linewidth]{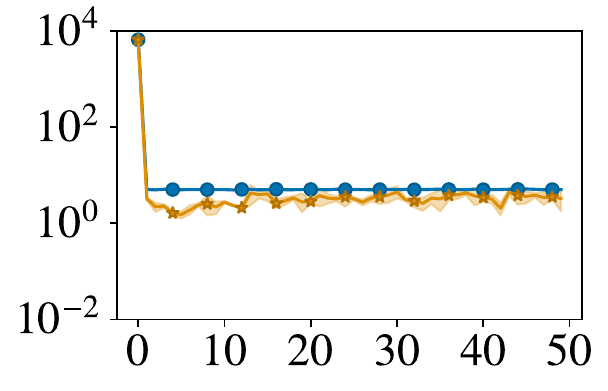}
    \end{subfigure}~%
    \begin{subfigure}[t]{0.24\textwidth}
        \centering
        \includegraphics[width=\linewidth]{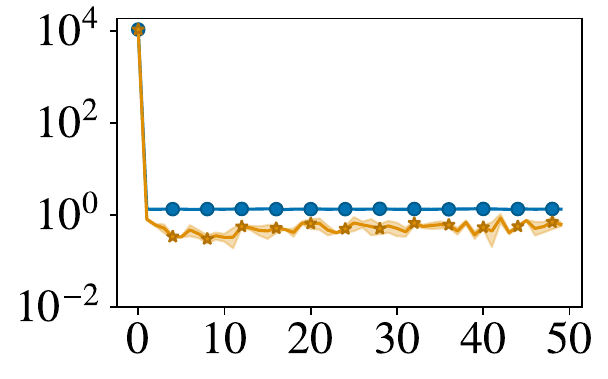}
    \end{subfigure}~%
    \begin{subfigure}[t]{0.24\textwidth}
        \centering
        \includegraphics[width=\textwidth]{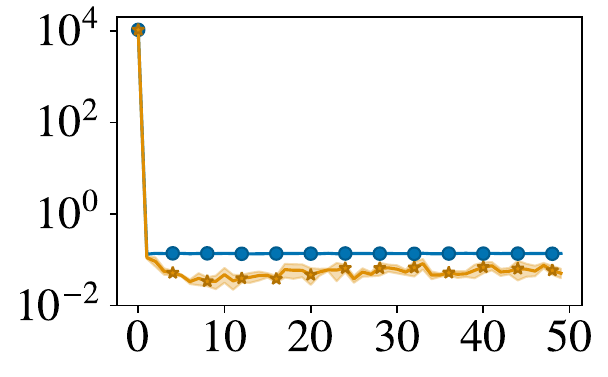}
    \end{subfigure}
        \centering
    \raisebox{2.5\height}{\rotatebox[origin=c]{90}{MSE}}
    \begin{subfigure}[t]{0.24\textwidth}
        \centering
        \includegraphics[width=\linewidth]{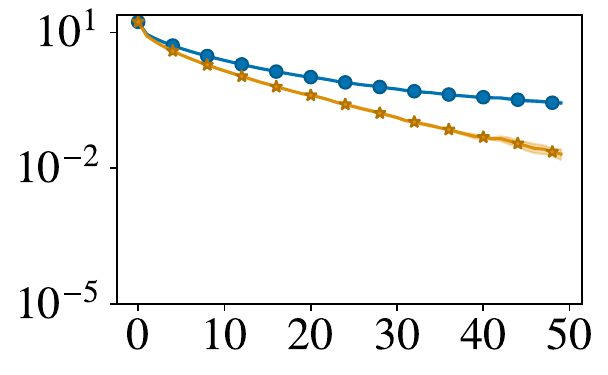}
        Communication rounds
        
        \caption{$N=10$}
        \label{fig:lin_n10}
    \end{subfigure}~%
    \begin{subfigure}[t]{0.24\textwidth}
        \centering
        \includegraphics[width=\linewidth]{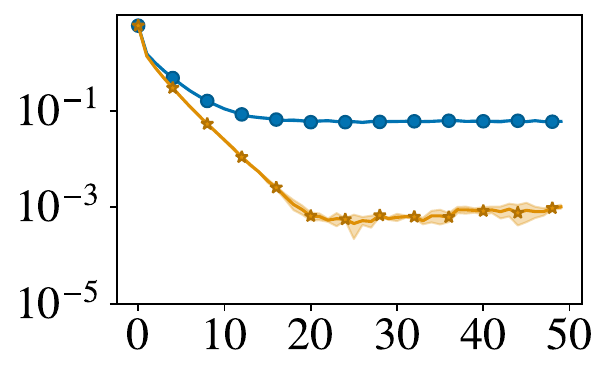}
        Communication rounds
        
        \caption{$N=100$}
        \label{fig:lin_n100}
    \end{subfigure}~%
    \begin{subfigure}[t]{0.24\textwidth}
        \centering
        \includegraphics[width=\linewidth]{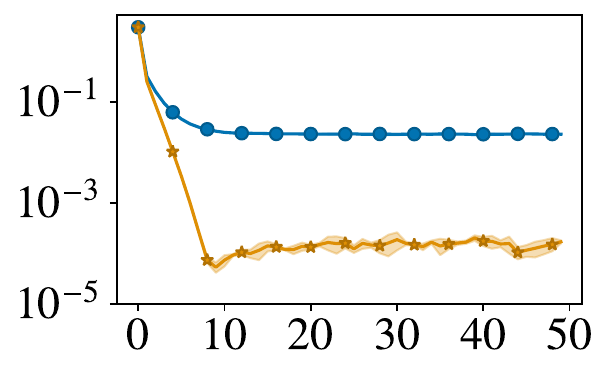}
        Communication rounds
        
        \caption{$N=1000$}
        \label{fig:log_n10}
    \end{subfigure}~%
    \begin{subfigure}[t]{0.24\textwidth}
        \centering
        \includegraphics[width=\textwidth]{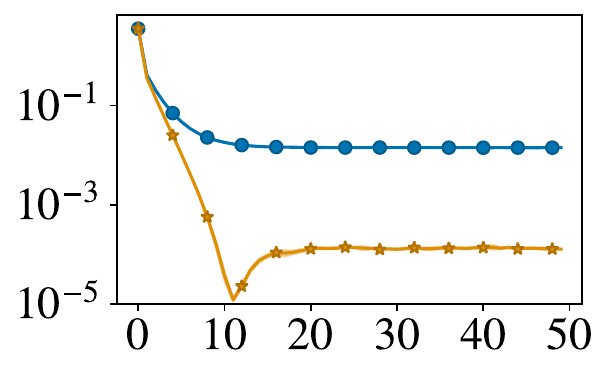}
        Communication rounds
        
        \caption{$N=10000$}
        \label{fig:log_n100}
    \end{subfigure}
    \caption{Mean squared error $\PE[ \norm{ \globparam{t} - \paramlim}^2 ]$ as a function of the number of communications, with $\nlupdates = 100$ and $\step = 0.05$, for linear regression (top row) and logistic regression (bottom row) problems. For each curve, we plot the average over $3$ runs and the standard deviation.}
    \label{fig:results-experiments}
\end{figure*}

The analysis framework that we put in place in \Cref{sec:convergence-scaffold} is guided by the study of the covariances of the global parameters and control variates of \Scaffold.
We now provide novel insights on the behaviour of \Scaffold in the stationary regime.
In \Cref{sec:var-global-iterates}, we give exact first-order (in the step size) expression for the covariance matrices defined in \Cref{sec:bound-variance-stationary}.
Surprisingly, this study uncovers that \Scaffold's global parameters \emph{are still biased}, and we describe this bias in \Cref{sec:non-vanishing-bias-scaffold}.

\subsection{Variance of the Global Iterates}
\label{sec:var-global-iterates}
In \Scaffold, the only source of randomness comes from the stochasticity of the gradient updates.
These stochastic updates then propagate in the global iterates and control variates of the algorithm.
Our analysis framework allows us to give the following expressions of these covariances, as a function of the gradient's covariance at the solution $\paramlim$.
\begin{restatable}{lemma}{expansioncovfirstorder}
\label{lem:expansion-cov-first-order-main}
Assume \Cref{assum:strong-convexity}, \Cref{assum:smoothness}, \Cref{assum:third-derivative}, \Cref{assum:heterogeneity}, \Cref{assum:smooth-var}.
Furthermore, assume that the step size $\step$ and number of local updates $\nlupdates$ satisfy $\step \nlupdates \lip \heterboundhess \hidleq \strcvx\hiddencst{/10}$ and $\step \nlupdates (\lip+\strcvx) \hidleq 1\hiddencst{/48}$ and ${\step \smoothcstvar} \hidleq \strcvx\hiddencst{/19}$.
Then, it holds that, for $c \neq c' \in \iint{1}{\nagent}$,
\begin{align*}
\covparam
& =
\frac{\step}{\nagent} \opcov \noisecovmat{\paramlim}
+ O(\step^2 \nlupdates + \step^{3/2})
\eqsp,
\\
\covparamcvar{c}
& =
\frac{\step}{\nagent} \opcov \noisecovmat{\paramlim} (\hnf{c}{\paramlim} - \hf{\paramlim})
\skipquad + \frac{\step}{\nagent}
\left( \locnoisecovmat{c}{\paramlim} - \noisecovmat{\paramlim}  \right)
+ O(\step^2 \nlupdates + \step^{3/2})
\eqsp,
\\
\covcvar{c,c} & =
\Big( 1 - \frac{2}{\nagent} \Big)
\frac{1}{\nlupdates} \locnoisecovmat{c}{\paramlim}
+ \frac{1}{\nagent\nlupdates} \noisecovmat{\paramlim}
+ O(\step)
\eqsp,
\\
\covcvar{c,c'} & =
\frac{1}{\nagent\nlupdates} ( \noisecovmat{\paramlim} 
- \locnoisecovmat{c}{\paramlim} 
- \locnoisecovmat{c'}{\paramlim})
+ O(\step)
\eqsp,
\end{align*} 
where $\opcov = \left( \Id \otimes \hf{\paramlim} +\hf{\paramlim} \otimes \Id \right)^{-1}$, $\locnoisecovmat{c}{\paramlim} = \PE[ ( \updatefuncnoise{c}{\locRandStatew{c}}{\paramlim } )^{\otimes 2} ]$ and $\noisecovmat{\paramlim} = \frac{1}{\nagent} \sum_{c=1}^\nagent \locnoisecovmat{c}{\paramlim}$.
\end{restatable}
We prove this lemma in \Cref{sec:proof-expansion-cov-first-order-main}.
This result confirms our finding that, in the stationary regime of \Scaffold, the covariances $\covparam$ and $\covparamcvar{c}$ both scale in $\step / \nagent$.
However, this is not the case for the control variates, which do not even scale in the step size $\gamma$.
More remarkably, we show that, for any $c$, the covariance $\covcvar{c,c}$ of $\cvar{c}{}$ does not decrease in $1/\nagent$.
Fortunately, the covariances of pairs of distinct control variates recovers this $1/\nagent$, which is the reason why \Scaffold enjoys linear speed-up.

\begin{remark}
We note that, in the analysis of \ProxSkip \citep{mishchenko2022proxskip}, they use a Lyapunov function similar to \eqref{eq:def-lambda-norm-main}, based on the average of the $\step^2 \nlupdates^2 \norm{ \cvar{c}{} - \cvarlim{c} }^2$.
\Cref{lem:expansion-cov-first-order-main} shows that this Lyapunov function \emph{cannot achieve linear speed-up}, as its terms only scale in $O(\step^2 \nlupdates)$.
\end{remark}

\subsection{Non-Vanishing Bias of \Scaffold}
\label{sec:non-vanishing-bias-scaffold}
Quite surprisingly, our analysis highlights that \Scaffold is still biased.
We now give an expression of \Scaffold's bias, \ie, the expected error in the stationary distribution
\begin{align}
\biasparam
& \eqdef
\int (\param - \paramlim) \statdist{\gamma, \nlupdates}{(\rmd \param, \rmd \Cvarw)}
\eqsp.
\end{align}
We require the fourth derivative of $\nfw{c}$ to be bounded.
\begin{assum}[Fourth Derivative]
\label{assum:fourth-derivative}
For $c \in \iint{1}{\nagent}$, the function $\nfw{c}$ is 4 times differentiable and satisfies, for any $\param \in \rset^d$ and $u \in\rset^d$, $\norm{ \hhhnf{c}{\param}u^{\otimes 3} } \le \fourthlip \norm{ u }^3$.
\end{assum}
Given this assumption, we obtain the following theorem.
\begin{restatable}{theorem}{biasscaffoldexprssion}
\label{thm:bias-scaffold}
Assume \Cref{assum:strong-convexity}, \Cref{assum:smoothness}, \Cref{assum:third-derivative}, \Cref{assum:heterogeneity}, \Cref{assum:smooth-var}, \Cref{assum:fourth-derivative}.
Furthermore, assume that the step size $\step$ and number of local updates $\nlupdates$ satisfy $ \step (\nlupdates - 1) \lip \heterboundhess \hidleq \strcvx\hiddencst{/10} $ and $ \step \nlupdates (\lip+\strcvx) \hidleq  1\hiddencst{/12} $ and $ \step \smoothcstvar \hidleq \strcvx \hiddencst{/19}$.
Then, the bias of S{\scriptsize CAFFOLD} is
\begin{align*}
\!\!\!\!\biasparam\!\! 
 =\!
- \frac{\step}{2\nagent} \hf{\paramlim}^{\!-1} \hhf{\paramlim} 
\opcov \noisecovmat{\paramlim}
+ O(\step^2\! \nlupdates + \step^{3/2})
~.\!\!\!\!
\end{align*}
\end{restatable}
We refer to \Cref{sec:proof-thm-bias-scafold} for a proof of this theorem.
Even though \Scaffold eliminates heterogeneity bias, its global iterates remain biased.
This bias scales with $\step/\nagent$ times the local gradient's variance.
It is not due to heterogeneity, but solely to the stochasticity of the local updates. 
In fact, we even recognize the bias of \FedAvg with homogeneous functions, as presented in \citet{mangold2024refined}'s Theorem~3.
We note that this bias scales with the local gradients' covariances, suggesting that \Scaffold may not be appropriate in problems with very noisy gradients.

\section{Numerical Results}
\label{sec:numerical}
\textbf{Experimental setup.} We illustrate our theoretical findings on $\ell_2$ regularized linear and logistic regression.
For linear regression, we use \texttt{make\_regression} function from scikit-learn \citep{pedregosa2011scikit} to generate two different datasets with $100 \nagent$ records and $20$ features; to simulate heterogeneity, we use different seeds and \texttt{n\_informative=2} and \texttt{n\_informative=10} respectively.
The first dataset is split evenly among the first $\nagent/2$ clients, while the second one is split evenly across the other half of clients.
For logistic regression, we repeat the same procedure with the \texttt{make\_classification} function with two different seeds.
Using this procedure, we generate a regression and a classification task, where each client has $200$ records, and where the distribution is heterogeneous.
In both settings, we run \Scaffold with $\step = 0.05$ and $\nlupdates = 100$, $\nrounds = 100$ and $\nagent \in \{ 10, 100, 1000, 10000 \}$. We estimate the gradients using batches of size $10$, and compare the result with \FedAvg with the same parameters.
The code is available online at \url{https://github.com/pmangold/scaffold-speed-up}.

\textbf{\Scaffold has linear speed-up}.
For each value of $\nagent$, we run both \Scaffold and \FedAvg and report the results in \Cref{fig:results-experiments}.
As expected, \Scaffold consistently outperforms \FedAvg in all settings.
In conformity with our theory, \Scaffold benefits from the presence of more clients: as the number of clients increases, the error in stationary regime decrease, both in linear (top row) and logistic (bottom row) regression. 

\textbf{Linear speed-up with many clients.}
Remarkably, the linear speed-up remains for number of clients gets large (up to $1,000$), suggesting that the condition on the maximal number of clients until which the linear speed-up holds in \Cref{cor:complexity-with-speed-up} is not overly restrictive.
Nonetheless, there is no more improvement from $\nagent = 1,000$ to $\nagent = 10,000$ in our logistic regression problem (bottom row): this suggest we have reached saturation, and that in this setting, increasing the number of clients does not help beyond this point.
As predicted by our theory, this is not the case in linear regression (top row). Indeed, in this case, the loss function is quadratic (\ie, $\thirdlip = 0$) and the limit on the number of clients stated in \Cref{cor:complexity-with-speed-up} is thus infinite.

\section{Conclusion}
In this paper, we provide a novel analytical framework for the \Scaffold algorithm.
We show that its global iterates and control variates define a Markov chain, that converges to a stationary distribution.
This key property allows us to derive the first rate which shows that \emph{S{\scriptsize CAFFOLD} achieves linear speed-up in the number of clients}.
Our analysis is based on a careful examination of the covariance of \Scaffold's global iterates and covariance, finely tracking the propagation of noise through the algorithm's parameters.

Although our work provide novel insights on the behavior of \Scaffold, many questions remain open.
In particular, it is yet to be understood whether \Scaffold can enjoy "deterministic" accelerated communication complexity as in \citet{mishchenko2022proxskip,hu2023tighter}'s analyses while preserving the desired linear speed-up.
Finally, our analysis highlights that \Scaffold's iterates are still biased: designing novel methods that remove this residual bias is a promising direction for the development of novel stochastic federated learning methods.

\section*{Acknowledgements}
The work of P. Mangold has been supported by Technology Innovation Institute (TII), project Fed2Learn. The work of A. Dieuleveut is supported by Hi!Paris FLAG chair, and this work has benefited from French State aid managed by the Agence Nationale de la Recherche (ANR) under France 2030 program with the reference ANR-23-PEIA-005 (REDEEM project).
The work of E. Moulines has been partly funded by the European Union (ERC-2022-SYG-OCEAN-101071601). Views and opinions expressed are however those of the author(s) only and do not necessarily reflect those of the European Union or the European Research Council Executive Agency. Neither the European Union nor the granting authority can be held responsible for them.

\section*{Impact Statement}
This paper presents work whose goal is to advance the field
of Machine Learning. There are many potential societal
consequences of our work, none which we feel must be
specifically highlighted here.

\bibliography{references}

\begin{thebibliography}{41}
\providecommand{\natexlab}[1]{#1}
\providecommand{\url}[1]{\texttt{#1}}
\expandafter\ifx\csname urlstyle\endcsname\relax
  \providecommand{\doi}[1]{doi: #1}\else
  \providecommand{\doi}{doi: \begingroup \urlstyle{rm}\Url}\fi

\bibitem[Bach \& Moulines(2013)Bach and Moulines]{bach2013non}
Bach, F. and Moulines, E.
\newblock Non-strongly-convex smooth stochastic approximation with convergence
  rate o (1/n).
\newblock \emph{Advances in neural information processing systems}, 26, 2013.

\bibitem[Charles \& Kone{\v{c}}n{\`y}(2021)Charles and
  Kone{\v{c}}n{\`y}]{charles2021convergence}
Charles, Z. and Kone{\v{c}}n{\`y}, J.
\newblock Convergence and accuracy trade-offs in federated learning and
  meta-learning.
\newblock In \emph{International Conference on Artificial Intelligence and
  Statistics}, pp.\  2575--2583. PMLR, 2021.

\bibitem[Chee \& Toulis(2018)Chee and Toulis]{chee2018convergence}
Chee, J. and Toulis, P.
\newblock Convergence diagnostics for stochastic gradient descent with constant
  learning rate.
\newblock In \emph{International Conference on Artificial Intelligence and
  Statistics}, pp.\  1476--1485. PMLR, 2018.

\bibitem[Condat \& Richt{\'a}rik(2022)Condat and
  Richt{\'a}rik]{condat2022randprox}
Condat, L. and Richt{\'a}rik, P.
\newblock Randprox: Primal-dual optimization algorithms with randomized
  proximal updates.
\newblock \emph{arXiv preprint arXiv:2207.12891}, 2022.

\bibitem[Condat et~al.(2022)Condat, Agarsk{\`y}, and
  Richt{\'a}rik]{condat2022provably}
Condat, L., Agarsk{\`y}, I., and Richt{\'a}rik, P.
\newblock Provably doubly accelerated federated learning: The first
  theoretically successful combination of local training and communication
  compression.
\newblock \emph{arXiv preprint arXiv:2210.13277}, 2022.

\bibitem[Diaconis \& Freedman(1999)Diaconis and Freedman]{diaconis1999iterated}
Diaconis, P. and Freedman, D.
\newblock Iterated random functions.
\newblock \emph{SIAM review}, 41\penalty0 (1):\penalty0 45--76, 1999.

\bibitem[Dieuleveut \& Bach(2016)Dieuleveut and
  Bach]{dieuleveut2016nonparametric}
Dieuleveut, A. and Bach, F.
\newblock {Nonparametric stochastic approximation with large step-sizes}.
\newblock \emph{The Annals of Statistics}, 44\penalty0 (4):\penalty0 1363 --
  1399, 2016.
\newblock \doi{10.1214/15-AOS1391}.

\bibitem[Dieuleveut et~al.(2020)Dieuleveut, Durmus, and
  Bach]{dieuleveut2020bridging}
Dieuleveut, A., Durmus, A., and Bach, F.
\newblock {Bridging the gap between constant step size stochastic gradient
  descent and {M}arkov chains}.
\newblock \emph{The Annals of Statistics}, 48\penalty0 (3):\penalty0 1348 --
  1382, 2020.
\newblock \doi{10.1214/19-AOS1850}.

\bibitem[Douc et~al.(2018)Douc, Moulines, Priouret, and
  Soulier]{douc2018markov}
Douc, R., Moulines, E., Priouret, P., and Soulier, P.
\newblock \emph{Markov chains}.
\newblock Springer, 2018.

\bibitem[Fort \& Pages(1999)Fort and Pages]{fort1999asymptotic}
Fort, J.-C. and Pages, G.
\newblock Asymptotic behavior of a markovian stochastic algorithm with constant
  step.
\newblock \emph{SIAM journal on control and optimization}, 37\penalty0
  (5):\penalty0 1456--1482, 1999.

\bibitem[Glasgow et~al.(2022)Glasgow, Yuan, and Ma]{glasgow2022sharp}
Glasgow, M.~R., Yuan, H., and Ma, T.
\newblock Sharp bounds for federated averaging (local sgd) and continuous
  perspective.
\newblock In \emph{International Conference on Artificial Intelligence and
  Statistics}, pp.\  9050--9090. PMLR, 2022.

\bibitem[Gorbunov et~al.(2021)Gorbunov, Hanzely, and
  Richt{\'a}rik]{gorbunov2021local}
Gorbunov, E., Hanzely, F., and Richt{\'a}rik, P.
\newblock Local sgd: Unified theory and new efficient methods.
\newblock In \emph{International Conference on Artificial Intelligence and
  Statistics}, pp.\  3556--3564. PMLR, 2021.

\bibitem[Haddadpour \& Mahdavi(2019)Haddadpour and
  Mahdavi]{haddadpour2019convergence}
Haddadpour, F. and Mahdavi, M.
\newblock On the convergence of local descent methods in federated learning.
\newblock \emph{arXiv preprint arXiv:1910.14425}, 2019.

\bibitem[Hu \& Huang(2023)Hu and Huang]{hu2023tighter}
Hu, Z. and Huang, H.
\newblock Tighter analysis for proxskip.
\newblock In \emph{International Conference on Machine Learning}, pp.\
  13469--13496. PMLR, 2023.

\bibitem[Karimireddy et~al.(2020)Karimireddy, Kale, Mohri, Reddi, Stich, and
  Suresh]{karimireddy2020scaffold}
Karimireddy, S.~P., Kale, S., Mohri, M., Reddi, S., Stich, S., and Suresh,
  A.~T.
\newblock Scaffold: Stochastic controlled averaging for federated learning.
\newblock In \emph{International conference on machine learning}, pp.\
  5132--5143. PMLR, 2020.

\bibitem[Khaled et~al.(2020)Khaled, Mishchenko, and
  Richt{\'a}rik]{khaled2020tighter}
Khaled, A., Mishchenko, K., and Richt{\'a}rik, P.
\newblock Tighter theory for local sgd on identical and heterogeneous data.
\newblock In \emph{International Conference on Artificial Intelligence and
  Statistics}, pp.\  4519--4529. PMLR, 2020.

\bibitem[Li et~al.(2019{\natexlab{a}})Li, Huang, Yang, Wang, and
  Zhang]{li2019convergence}
Li, X., Huang, K., Yang, W., Wang, S., and Zhang, Z.
\newblock On the convergence of fedavg on non-iid data.
\newblock \emph{arXiv preprint arXiv:1907.02189}, 2019{\natexlab{a}}.

\bibitem[Li et~al.(2019{\natexlab{b}})Li, Yang, Wang, and
  Zhang]{li2019communication}
Li, X., Yang, W., Wang, S., and Zhang, Z.
\newblock Communication-efficient local decentralized sgd methods.
\newblock \emph{arXiv preprint arXiv:1910.09126}, 2019{\natexlab{b}}.

\bibitem[Malinovskiy et~al.(2020)Malinovskiy, Kovalev, Gasanov, Condat, and
  Richtarik]{malinovskiy2020local}
Malinovskiy, G., Kovalev, D., Gasanov, E., Condat, L., and Richtarik, P.
\newblock From local sgd to local fixed-point methods for federated learning.
\newblock In \emph{International Conference on Machine Learning}, pp.\
  6692--6701. PMLR, 2020.

\bibitem[Malinovsky et~al.(2022)Malinovsky, Yi, and
  Richt{\'a}rik]{malinovsky2022variance}
Malinovsky, G., Yi, K., and Richt{\'a}rik, P.
\newblock Variance reduced proxskip: Algorithm, theory and application to
  federated learning.
\newblock \emph{Advances in Neural Information Processing Systems},
  35:\penalty0 15176--15189, 2022.

\bibitem[Mangold et~al.(2024{\natexlab{a}})Mangold, Durmus, Dieuleveut,
  Samsonov, and Moulines]{mangold2024refined}
Mangold, P., Durmus, A., Dieuleveut, A., Samsonov, S., and Moulines, E.
\newblock Refined analysis of federated averaging's bias and federated
  richardson-romberg extrapolation.
\newblock \emph{arXiv preprint arXiv:2412.01389}, 2024{\natexlab{a}}.

\bibitem[Mangold et~al.(2024{\natexlab{b}})Mangold, Samsonov, Labbi, Levin,
  Alami, Naumov, and Moulines]{mangold2024scafflsa}
Mangold, P., Samsonov, S., Labbi, S., Levin, I., Alami, R., Naumov, A., and
  Moulines, E.
\newblock Scafflsa: Taming heterogeneity in federated linear stochastic
  approximation and td learning.
\newblock \emph{arXiv preprint arXiv:2402.04114}, 2024{\natexlab{b}}.

\bibitem[McMahan et~al.(2017)McMahan, Moore, Ramage, Hampson, and
  y~Arcas]{mcmahan2017communication}
McMahan, B., Moore, E., Ramage, D., Hampson, S., and y~Arcas, B.~A.
\newblock Communication-efficient learning of deep networks from decentralized
  data.
\newblock In \emph{Artificial intelligence and statistics}, pp.\  1273--1282.
  PMLR, 2017.

\bibitem[Mishchenko et~al.(2022)Mishchenko, Malinovsky, Stich, and
  Richt{\'a}rik]{mishchenko2022proxskip}
Mishchenko, K., Malinovsky, G., Stich, S., and Richt{\'a}rik, P.
\newblock Proxskip: Yes! local gradient steps provably lead to communication
  acceleration! finally!
\newblock In \emph{International Conference on Machine Learning}, pp.\
  15750--15769. PMLR, 2022.

\bibitem[Mitra et~al.(2021)Mitra, Jaafar, Pappas, and Hassani]{mitra2021linear}
Mitra, A., Jaafar, R., Pappas, G.~J., and Hassani, H.
\newblock Linear convergence in federated learning: Tackling client
  heterogeneity and sparse gradients.
\newblock \emph{Advances in Neural Information Processing Systems},
  34:\penalty0 14606--14619, 2021.

\bibitem[Nesterov(2013)]{nesterov2013introductory}
Nesterov, Y.
\newblock \emph{Introductory lectures on convex optimization: A basic course},
  volume~87.
\newblock Springer Science \& Business Media, 2013.

\bibitem[Osekowski(2012)]{oskekowski2012sharp}
Osekowski, A.
\newblock \emph{Sharp martingale and semimartingale inequalities}, volume~72.
\newblock Springer Science \& Business Media, 2012.

\bibitem[Patel \& Dieuleveut(2019)Patel and Dieuleveut]{patel2019communication}
Patel, K.~K. and Dieuleveut, A.
\newblock Communication trade-offs for synchronized distributed sgd with large
  step size.
\newblock \emph{arXiv preprint arXiv:1904.11325}, 2019.

\bibitem[Pathak \& Wainwright(2020)Pathak and Wainwright]{pathak2020fedsplit}
Pathak, R. and Wainwright, M.~J.
\newblock Fedsplit: An algorithmic framework for fast federated optimization.
\newblock \emph{Advances in neural information processing systems},
  33:\penalty0 7057--7066, 2020.

\bibitem[Pedregosa et~al.(2011)Pedregosa, Varoquaux, Gramfort, Michel, Thirion,
  Grisel, Blondel, Prettenhofer, Weiss, Dubourg, et~al.]{pedregosa2011scikit}
Pedregosa, F., Varoquaux, G., Gramfort, A., Michel, V., Thirion, B., Grisel,
  O., Blondel, M., Prettenhofer, P., Weiss, R., Dubourg, V., et~al.
\newblock Scikit-learn: Machine learning in python.
\newblock \emph{the Journal of machine Learning research}, 12:\penalty0
  2825--2830, 2011.

\bibitem[Pflug(1986)]{pflug1986stochastic}
Pflug, G.~C.
\newblock Stochastic minimization with constant step-size: asymptotic laws.
\newblock \emph{SIAM Journal on Control and Optimization}, 24\penalty0
  (4):\penalty0 655--666, 1986.

\bibitem[Sadiev et~al.(2022)Sadiev, Kovalev, and
  Richt{\'a}rik]{sadiev2022communication}
Sadiev, A., Kovalev, D., and Richt{\'a}rik, P.
\newblock Communication acceleration of local gradient methods via an
  accelerated primal-dual algorithm with an inexact prox.
\newblock \emph{Advances in Neural Information Processing Systems},
  35:\penalty0 21777--21791, 2022.

\bibitem[Schmidt et~al.(2017)Schmidt, Le~Roux, and Bach]{schmidt2017minimizing}
Schmidt, M., Le~Roux, N., and Bach, F.
\newblock Minimizing finite sums with the stochastic average gradient.
\newblock \emph{Mathematical Programming}, 162:\penalty0 83--112, 2017.

\bibitem[Stich(2019)]{stich2019local}
Stich, S.~U.
\newblock Local sgd converges fast and communicates little.
\newblock In \emph{International Conference on Learning Representations}, 2019.

\bibitem[Wang \& Joshi(2018)Wang and Joshi]{wang2018cooperative}
Wang, J. and Joshi, G.
\newblock Cooperative sgd: A unified framework for the design and analysis of
  communication-efficient sgd algorithms.
\newblock \emph{arXiv preprint arXiv:1808.07576}, 2018.

\bibitem[Wang et~al.(2024)Wang, Das, Joshi, Kale, Xu, and
  Zhang]{wang2024Unreasonable}
Wang, J., Das, R., Joshi, G., Kale, S., Xu, Z., and Zhang, T.
\newblock On the unreasonable effectiveness of federated averaging with
  heterogeneous data.
\newblock \emph{Trans. Mach. Learn. Res.}, 2024, 2024.

\bibitem[Woodworth et~al.(2020{\natexlab{a}})Woodworth, Patel, Stich, Dai,
  Bullins, Mcmahan, Shamir, and Srebro]{woodworth2020local}
Woodworth, B., Patel, K.~K., Stich, S., Dai, Z., Bullins, B., Mcmahan, B.,
  Shamir, O., and Srebro, N.
\newblock Is local sgd better than minibatch sgd?
\newblock In \emph{International Conference on Machine Learning}, pp.\
  10334--10343. PMLR, 2020{\natexlab{a}}.

\bibitem[Woodworth et~al.(2020{\natexlab{b}})Woodworth, Patel, and
  Srebro]{woodworth2020minibatch}
Woodworth, B.~E., Patel, K.~K., and Srebro, N.
\newblock Minibatch vs local sgd for heterogeneous distributed learning.
\newblock \emph{Advances in Neural Information Processing Systems},
  33:\penalty0 6281--6292, 2020{\natexlab{b}}.

\bibitem[Yang et~al.(2021)Yang, Fang, and Liu]{yang2021achieving}
Yang, H., Fang, M., and Liu, J.
\newblock Achieving linear speedup with partial worker participation in non-iid
  federated learning.
\newblock \emph{arXiv preprint arXiv:2101.11203}, 2021.

\bibitem[Yu et~al.(2019{\natexlab{a}})Yu, Jin, and Yang]{yu2019linear}
Yu, H., Jin, R., and Yang, S.
\newblock On the linear speedup analysis of communication efficient momentum
  sgd for distributed non-convex optimization.
\newblock In \emph{International Conference on Machine Learning}, pp.\
  7184--7193. PMLR, 2019{\natexlab{a}}.

\bibitem[Yu et~al.(2019{\natexlab{b}})Yu, Yang, and Zhu]{yu2019parallel}
Yu, H., Yang, S., and Zhu, S.
\newblock Parallel restarted sgd with faster convergence and less
  communication: Demystifying why model averaging works for deep learning.
\newblock In \emph{Proceedings of the AAAI conference on artificial
  intelligence}, volume~33, pp.\  5693--5700, 2019{\natexlab{b}}.

\end{thebibliography}
\bibliographystyle{icml2025}

\newpage
\appendix
\onecolumn
\resetspaces
\renewcommand{\hiddencst}[1]{#1}

\section{Preliminaries}
\label{sec:op_gen}

\label{sec:scaffold-iterates-def}
\subsection{Strong convexity and Smoothness}
\label{sec:monotonicity-cocoercivity}

We list here the inequalities that are consequences of strong convexity (\Cref{assum:strong-convexity}) and smoothness (\Cref{assum:smoothness}) of the functions that we minimize in \eqref{pb:smooth-fl}.
For $c \in \iint{1}{\nagent}$ and $\randStatew_{(c)} \sim \nu_{(c)}$, \Cref{assum:strong-convexity} and \Cref{assum:smoothness} imply that, for any $\param, \param' \in \rset^d$, 
\begin{align}
\label{eq:cocoercivity}
    \PE\left[\norm{\gnfs{c}{\param}{\randStatew_{(c)}} - \gnfs{c}{\param^{\prime}}{\randStatew_{(c)}} }^2 \right] \leq L\pscal{\nabla \nfw{c}(\theta)-\nabla \nfw{c}(\theta^{\prime})  }{\theta - \theta^{\prime}}
\eqsp.
\end{align}
This inequality is generally referred to as co-coercivity of the gradient of $\nfw{c}$, and is proven in Theorem 2.1.5 of \citet{nesterov2013introductory}.
Assumptions \Cref{assum:strong-convexity} and \Cref{assum:smoothness} also imply that, for any $\param, \param' \in \rset^d$, 
\begin{align}
\label{eq:monotonicity}
- \pscal{\nabla \nfw{c}(\theta)-\nabla \nfw{c}(\theta^{\prime})  }{\theta - \theta^{\prime}}
\le
- \strcvx \norm{ \param - \param' }^2
\eqsp.
\end{align}
This second inequality is generally referred to as monotonicity of the gradient of $\nfw{c}$.
Finally, smoothness of $\nfsw{c}{\randStatez{c}{}}$, for $\randStatez{c}{} \in \msZ$ (\Cref{assum:smoothness}), means that the gradient of $\nfsw{c}{z}$ is Lipschitz, \ie, for any $\param, \param' \in \rset^d$, 
\begin{align}
\label{eq:grad-lipschitz}
\norm{\gnfs{c}{\param}{\randStatez{c}{} } - \gnfs{c}{\param^{\prime}}{\randStatez{c}{} }}
\le
\lip
\norm{\gnfs{c}{\param}{\randStatez{c}{} } - \gnfs{c}{\param^{\prime}}{\randStatez{c}{} }}
\eqsp.
\end{align}

\subsection{Iterate Operators.}
We recall the operators defined in \Cref{sec:convergence-global-iterates}, that generate the local and global updates of \Scaffold.
For $c\in \iint{1}{\nagent}$, $\param \in \rset^d$, $\cvar{c}{} \in\rset^d$ and $\randStatez{c}{} \in \msZ$ define 
\begin{align*}
\locstoscafop{c}{}{\param}{\randStatez{c}{}}{\cvar{c}{}}
& =
\param - \step \left\{ \gnfs{c}{\param}{\randStatez{c}{}} + \cvar{c}{} \right\}
\eqsp,
\end{align*}
Then, set $\locstoscafop{c}{0}{\param}{\cvar{c}{}}{\randStatez{c}{} } = \param$ and define recursively for $\randStatez{c}{1:h+1} = (\randStatez{c}{1},\ldots,\randStatez{c}{h+1}) \in \msZ^{h+1}$,
\begin{align*}
\locstoscafop{c}{h+1}{\param}{\cvar{c}{}}{\randStatez{c}{1:h+1}}
& =
\locstoscafop{c}{}{
\locstoscafop{c}{h}{\param}{\cvar{c}{}}{\randStatez{c}{1:h}}
}{\cvar{c}{}}{\randStatez{c}{h+1}}
\eqsp.
\end{align*}
This allows to define the global update operator, denoting $\cvar{1:\nagent}{} = (\cvar{1}{},\ldots,\cvar{\nagent}{})$ and $\randStatez{1:\nagent}{1:\nlupdates} = (\randStatez{1}{1:\nlupdates}, \ldots, \randStatez{\nagent}{1:\nlupdates})$
\begin{align*}
\globstoscafop{\nlupdates}{\param}{\cvar{1:\nagent}{}}{\randStatez{1:\nagent}{1:\nlupdates}}
=
\frac{1}{\nagent} \sum_{c=1}^\nagent
\locstoscafop{c}{\nlupdates}{\param}{\cvar{c}{}}{\randStatez{c}{1:\nlupdates}}
\eqsp.
\end{align*}
Similarly, we define the operator that updates the control variates, for $c\in \iint{1}{\nagent}$, as
\begin{align*}
\scafopcv{c}{\nlupdates}{\cvar{c}{}}{\param}{\randStatez{1:\nagent}{1:\nlupdates}}
& =
\cvar{c}{}
+ \frac{1}{\step \nlupdates}
\left(
\locstoscafop{c}{\nlupdates}{\param}{\cvar{c}{}}{\randStatez{c}{1:\nlupdates}}
-
\globstoscafop{\nlupdates}{\param}{\cvar{1:\nagent}{}}{\randStatez{1:\nagent}{1:\nlupdates}}
\right)
\eqsp.
\end{align*}
Thus, we can define the update of the \Scaffold algorithm with noise $\randStatez{1:\nagent}{1:\nlupdates}$ as
\begin{align*}
\opscaffold :
\left( \param, \cvar{1}{}, \dots, \cvar{\nagent}{}; \randStatez{1:\nagent}{1:\nlupdates} \right)
\mapsto
\left( \globstoscafop{\nlupdates}{\param}{\cvar{1:\nagent}{}}{\randStatew}, \scafopcv{1}{\nlupdates}{\cvar{1}{}}{\param}{\randStatez{1:\nagent}{1:\nlupdates}}, \dots, \scafopcv{\nagent}{\nlupdates}{\cvar{\nagent}{}}{\param}{\randStatez{1:\nagent}{1:\nlupdates}} \right)
\eqsp.
\end{align*}

\section{Proof of Convergence of \Scaffold}
\label{sec:proof-convergence-scaffold}
\subsection{Convergence of Scaffold's iterates -- Proof of Lemma~\ref{lem:contraction-scaffold-global-update} and Theorem~\ref{thm:convergence-scaffold-stat-dist}}
\label{sec:proof-contract-scaffold-update}

We now analyze the convergence of \Scaffold's iterates. Specifically, we aim to demonstrate that, akin to \FedAvg and \SGD, the iterates of \Scaffold (i.e., its parameters and control variates) converge to a unique stationary distribution. 

To establish this result, we first show that \Scaffold's updates exhibit contractive behavior under certain conditions. For this purpose, we introduce the following norm, which assigns appropriate weights to each parameter and control variate,
\begin{align}
\label{eq:def-lambda-norm}
\norm{ \bigX }[\Lambda]^2
& =
\norm{ \param }^2
+ \frac{\step^2 \nlupdates^2}{\nagent} 
\sum_{c=1}^\nagent
\norm{ \cvar{c}{} }^2
\eqsp,
\end{align}
where $\bigX = ( \param, \cvar{1}{}, \dots, \cvar{\nagent}{} )$. This can be seen as a norm on $\rset^{(\nagent+1) d}$ such that $\norm{ \bigX }[\Lambda]^2 = \pscal{ \bigX }{ \Lambda \bigX}$ for $X \in \rset^{(\nagent+1) d}$ and where 
\begin{align*}
    \Lambda = \diag\left( \Id_d, \frac{\step^2\nlupdates^2}{\nagent} \Id_d, \dots, \frac{\step^2\nlupdates^2}{\nagent} \Id_d\right) \eqsp,
\end{align*}
and $\Id_d$ is the $d \times d$ identity matrix.
We now show that $\opscaffold$ is a contractive operator under the norm $\norm{\cdot}[\Lambda]$.

\contractscaffoldglobalupdate*
\begin{proof}
For readability, we define, for $\param, \param', \cvar{c}{}, \cvar{c}{\prime} \in \rset^d$, notations for the global parameter $\param$ update, the local parameters updates and the control variates $\cvar{c}{}$ updates as,
\begin{align}
\globparam{+}
= \globstoscafop{\nlupdates}{\param}{\cvar{1:\nagent}{}}{\locRandState{1:\nagent}{1:\nlupdates}}
\eqsp,
\qquad
\locparam{c}{h}
= \locstoscafop{c}{h}{\param}{\cvar{c}{}}{\locRandState{c}{1:h}}
\eqsp,
\qquad
\cvar{c}{+}
= \scafopcv{c}{\nlupdates}{\cvar{c}{}}{\param}{\locRandState{c}{1:\nlupdates}}
\eqsp,
\end{align}
and similarly for $\param'$ and $\cvar{c}{\prime}$,
\begin{align}
\globparam{\prime +}
= \globstoscafop{\nlupdates}{\param'}{\tcvar{1:\nagent}{}}{\locRandState{1:\nagent}{1:\nlupdates}}
\eqsp,
\qquad
\locparam{c}{\prime h}
= \locstoscafop{c}{h}{\param'}{\tcvar{c}{}}{\locRandState{c}{1:h}}
\eqsp,
\qquad
\cvar{c}{\prime +}
= \scafopcv{c}{\nlupdates}{\tcvar{c}{}}{\param'}{\locRandState{c}{1:\nlupdates}}
\eqsp.
\end{align}

Recall that $\paramp =\nagent^{-1}\sum_{c=1}^\nagent \locparam{c}{\nlupdates}$ and $\param^{\prime +}=\nagent^{-1}\sum_{c=1}^\nagent \locparam{c}{\prime \nlupdates}$.
We can thus use the fact that $\sum_{c=1}^\nagent \cvar{c}{} = 0$ and $\sum_{c=1}^\nagent \tcvar{c}{} = 0$, as well as \Cref{lem:projection} with $x_c = \locparam{c}{\nlupdates} + \step \nlupdates \cvar{c}{}$ and $y_c = \locparam{c}{\prime \nlupdates} + \step \nlupdates \tcvar{c}{}$ to obtain
\begin{align}
\nonumber
\norm{ \paramp - \param^{\prime +} }^2
& =
\bnorm{ 
\frac{1}{\nagent} \sum_{c=1}^\nagent \left(\locparam{c}{\nlupdates} + \step \nlupdates \cvar{c}{}\right)
- 
\frac{1}{\nagent} \sum_{c=1}^\nagent \left( \locparam{c}{ \prime \nlupdates} + \step \nlupdates \tcvar{c}{}\right)
}^2
\\
\label{eq:convergence-coupled-pythagoras}
& =
\frac{1}{\nagent}
\sum_{c=1}^\nagent
\bnorm{  \locparam{c}{\nlupdates} + \step \nlupdates \cvar{c}{}
-  \locparam{c}{\prime \nlupdates} - \step \nlupdates \tcvar{c}{} }^2
- 
\frac{1}{\nagent}
\sum_{c=1}^\nagent
\bnorm{ \step \nlupdates\!\left( \cvar{c}{+} - \tcvar{c}{+} \right)}^2
\eqsp,
\end{align}
where we used the fact that $\step \nlupdates 
\cvar{c}{+} =
\step \nlupdates \cvar{c}{} + \globparam{+} - \locparam{c}{\nlupdates}$ and $\step \nlupdates 
\tcvar{c}{ +} =
\step \nlupdates \tcvar{c}{} + \tglobparam{+} - \tlocparam{c}{ \nlupdates}$ in the second term.
We now define the shifted parameters, for $c \in \iint{1}{\nagent}$ and $h \in \iint{0}{\nlupdates}$,
\begin{align}
\label{eq:def-adjusted-operator-xi}
\locshiftparam{c}{h}
=
\locparam{c}{h}
+ \step h \cvar{c}{}
\quad,
\qquad
\locshiftparam{c}{\prime h}
=
\locparam{c}{\prime h}
+ \step h \tcvar{c}{}
\eqsp.
\end{align}
The identity~\eqref{eq:convergence-coupled-pythagoras} can be rewritten using the notations introduced in~\eqref{eq:def-adjusted-operator-xi}, which gives
\begin{align}
\label{eq:expression-updated-lyapunov}
\norm{ \opscaffold(\bigX; \randStatew) - \opscaffold(\bigX^{\prime}; \randStatew) }[\Lambda]^2
=
\frac{1}{\nagent}
\sum_{c=1}^\nagent
\norm{ 
\locshiftparam{c}{\nlupdates} - 
\locshiftparam{c}{\prime \nlupdates}}^2
\eqsp.
\end{align}
It remains to derive a bound on each term of this sum. 
We proceed by induction, on $h \in \iint{0}{\nlupdates-1}$ we have
\begin{align*}
\norm{ \locshiftparam{c}{h+1} - 
\locshiftparam{c}{\prime h+1} }^2
& =
\bnorm{ \locshiftparam{c}{h}
- \locshiftparam{c}{\prime h}
- \step 
\Big( 
\gnfs{c}{\locparam{c}{h}}{\locRandState{c}{h+1}}  
- \gnfs{c}{\locparam{c}{\prime h}}{\locRandState{c}{h+1}} 
\Big)
}^2
\eqsp.
\end{align*}
Expanding the square and using \eqref{eq:def-adjusted-operator-xi}, we obtain
\begin{align*}
& \norm{ \locshiftparam{c}{h+1} - 
\locshiftparam{c}{\prime h+1} }^2
\\
& \quad =
\bnorm{ \locshiftparam{c}{h} - 
\locshiftparam{c}{\prime h} }^2
+
\step^2
\bnorm{ 
\gnfs{c}{\locparam{c}{h}}{\locRandState{c}{h+1}}  
- \gnfs{c}{\locparam{c}{\prime h}}{\locRandState{c}{h+1}} 
}^2
- 2 \step
\bpscal{ \locshiftparam{c}{h} - 
\locshiftparam{c}{\prime h} }{ 
\gnfs{c}{\locparam{c}{h}}{\locRandState{c}{h+1}}  
- \gnfs{c}{\locparam{c}{\prime h}}{\locRandState{c}{h+1}}
}
\\
& \quad =
\bnorm{ \locshiftparam{c}{h} - 
\locshiftparam{c}{\prime h} }^2
+
\step^2
\bnorm{ 
\gnfs{c}{\locparam{c}{h}}{\locRandState{c}{h+1}}  
- \gnfs{c}{\locparam{c}{ \prime h}}{\locRandState{c}{h+1}} 
}^2
\\
& \qquad
- 2 \step
\bpscal{ \locparam{c}{h} - 
\locparam{c}{ \prime h}  }{ 
\gnfs{c}{\locparam{c}{h}}{\locRandState{c}{h+1}}  
- \gnfs{c}{\locparam{c}{\prime h}}{\locRandState{c}{h+1}}
}
- 2 \step^2 h
\bpscal{ 
\cvar{c}{} - \tcvar{c}{}
}{  
\gnfs{c}{\locparam{c}{h}}{\locRandState{c}{h+1}}  
- \gnfs{c}{\locparam{c}{\prime h}}{\locRandState{c}{h+1}}
}
\eqsp.
\end{align*}
Now, using Young's inequality to bound $2 \step^2 h a b = 2 (\step^{3/2} L^{1/2} h a) (\step^{1/2} L^{-1/2} b) \le \step^3 h^2 L a^2 + \step L^{-1} b^2$, we get
\begin{align*}
& - 2 \step^2 h
\bpscal{ 
\cvar{c}{} - \tcvar{c}{}
}{ 
\gnfs{c}{\locparam{c}{h}}{\locRandState{c}{h+1}}  
\!-\! \gnfs{c}{\locparam{c}{\prime h}}{\locRandState{c}{h+1}}
}
\le
\step^3 h^2 \lip
\bnorm{ 
\cvar{c}{} - \tcvar{c}{}
}^2
\!\!+
\frac{\step}{\lip} \bnorm{ 
\gnfs{c}{\locparam{c}{h}}{\locRandState{c}{h+1}}  
\!-\! \gnfs{c}{\locparam{c}{\prime h}}{\locRandState{c}{h+1}}
}^2
\eqsp.
\end{align*}
Plugging this in the previous inequality and using the co-coercivity of the gradient \eqref{eq:cocoercivity}, we have
\begin{align*}
\nonumber
\norm{ \locshiftparam{c}{h+1} - 
\locshiftparam{c}{\prime h+1} }^2
& \le
\norm{ \locshiftparam{c}{h} - \locshiftparam{c}{\prime h} }^2
+
\step^3 h^2 \lip
\norm{ 
\cvar{c}{} - \tcvar{c}{}
}^2
- (\step - \step^2 \lip)
\pscal{ \locparam{c}{h} - \locparam{c}{ \prime h} }{ 
\gnfs{c}{\locparam{c}{h}}{\locRandState{c}{h+1}}  
- \gnfs{c}{\locparam{c}{\prime h}}{\locRandState{c}{h+1}}
}
\eqsp.
\end{align*}
Using the fact that $\step \le 1/2\lip$ to bound $- (\step - \step^2 \lip) \le - \step / 2$, taking the conditional expectation and using that $Z^{h+1}_{(c)}$ is independent of $\locRandState{c}{1:h}$, and monotonicity of the gradient \eqref{eq:monotonicity}, we obtain
\begin{align}
\label{eq:bound-tilde-tilde-pscal}
\CPE{ \bnorm{ \locshiftparam{c}{h+1} - 
\locshiftparam{c}{\prime h+1} }^2 }{\locRandState{c}{1:h}}
& \le
\bnorm{ \locshiftparam{c}{h} - 
\locshiftparam{c}{\prime h}  }^2
- \frac{\step \strcvx}{2} \bnorm{ \locparam{c}{h} - 
\locparam{c}{ \prime h}  }^2
+
\step^3 h^2 \lip
\norm{ 
\cvar{c}{} - \tcvar{c}{}
}^2
\eqsp.
\end{align}
Now, we remark that, for $a, b \in \rset^d$, we have $a^2 = (a-b+b)^2 \le 2 (a-b)^2 + 2 b^2$, which implies that $- (a-b)^2 \le - \frac{1}{2} a^2 + b^2$.
Therefore, we have
\begin{align*}
- \frac{\step \strcvx}{2} \bnorm{ \locparam{c}{h+1} - 
\tlocparam{c}{ h+1} }^2
& =
- \frac{\step \strcvx}{2} \bnorm{  \locparam{c}{h} - 
\locparam{c}{ \prime h} 
- \step h (\cvar{c}{} - \tcvar{c}{})
}^2
\le
- \frac{\step \strcvx}{4} \bnorm{  \locparam{c}{h} - 
\locparam{c}{ \prime h}  }^2
+ \frac{\step^3 h^2 \strcvx}{2} \bnorm{ \cvar{c}{} - \tcvar{c}{} }^2
\eqsp.
\end{align*}
Using this inequality in \eqref{eq:bound-tilde-tilde-pscal}, we obtain the following inequality
\begin{equation}
\label{eq:one-step-contraction-abritrary}
\begin{aligned}
& \CPE{ \bnorm{ \locshiftparam{c}{h+1} - 
\locshiftparam{c}{\prime h+1}  }^2 }{\locRandState{c}{1:h}}
\le
\left(1 - \frac{\step \strcvx}{4}\right) \bnorm{ \locshiftparam{c}{h} - 
\locshiftparam{c}{\prime h} }^2
+
\left( 
\step^3 h^2 \strcvx
+ \step^3 h^2 \lip
\right)\bnorm{ 
\cvar{c}{} - \tcvar{c}{}
}^2
\eqsp.
\end{aligned}
\end{equation}
Taking the expectation in the last inequality, a straightforward induction leads to
\begin{equation*}
\begin{aligned}
& \PE\left[\bnorm{ \locshiftparam{c}{\nlupdates} - 
\locshiftparam{c}{\prime \nlupdates}  }^2 \right]
 \le
\left(1 - \frac{\step \strcvx}{4}\right)^\nlupdates \bnorm{ \param - \tilde{\param} }^2 
+
\frac{\step^3 \nlupdates^2 (\nlupdates-1) (\lip + \strcvx)}{2}
\bnorm{ 
\cvar{c}{} - \tcvar{c}{}
}^2 
\eqsp.
\end{aligned}
\end{equation*}
Consequently, whenever $\step \nlupdates (\lip + \strcvx) \le 1$, we can sum this inequality for $c = 1$ to $\nagent$ to obtain
\begin{align}
\nonumber
\PE\left[
\frac{1}{\nagent} \sum_{c=1}^\nagent \bnorm{ \locshiftparam{c}{\nlupdates} - 
\locshiftparam{c}{\prime \nlupdates} }^2
\right]
& \le
\left(1 - \frac{\step \strcvx}{4}\right)^\nlupdates  \bnorm{ \param - \param }^2 
+
\frac{1}{2}
\frac{\step^2\nlupdates^2}{\nagent} \sum_{c=1}^\nagent 
\bnorm{ \cvar{c}{} - \tcvar{c}{}
}^2 
  \le
\left(1 - \frac{\step \strcvx}{4}\right)^\nlupdates \bnorm{ \bigX - \bigX' }[\Lambda]^2 
\eqsp,
\end{align}
where the second inequality comes from $\frac{1}{2} \cdot \step^2 \nlupdates^2 \le
(1 - \frac{\step \strcvx}{4})^\nlupdates \cdot \step^2 \nlupdates^2$.
\end{proof}

\convergencescaffoldstatdist*

\begin{proof}
We use \citet[Theorem~20.3.4]{douc2018markov} with the cost function $c(\bigX,\tilde{\bigX})= \| \bigX - \tilde{\bigX} \|_{\Lambda}^2$, where the norm $\| \cdot \|_{\Lambda}$ is defined in \eqref{eq:def-lambda-norm}. 
\end{proof}
Note that the convergence toward the stationary distribution is geometric.

\subsection{Bound on \Scaffold's Global Iterates in the Stationary Distribution -- Proof of Lemma~\ref{lem:descent-noise} and Theorem~\ref{thm:crude-bound-2-X}}
\label{sec:bound-global-scaffold}
\lemdescentnoise*
\begin{proof}
As in \Cref{sec:contraction-coupling}, we define, for $\varparam, \varcvarw \in \rset^d$, notations for the global parameter update, the local parameter updates and the control variates updates as,
\begin{align}
\globvarparam{+}
= \globstoscafop{\nlupdates}{\varparam}{\cvar{1:\nagent}{}}{\locRandState{1:\nagent}{1:\nlupdates}}
\eqsp,
\qquad
\locvarparam{c}{h}
= \locstoscafop{c}{h}{\varparam}{\cvar{c}{}}{\locRandState{c}{1:h}}
\eqsp,
\qquad
\cvar{c}{+}
= \scafopcv{c}{\nlupdates}{\cvar{c}{}}{\varparam}{\locRandState{c}{1:\nlupdates}}
\eqsp,
\end{align}
for $c \in \iint{1}{\nagent}$ and $h \in \iint{0}{\nlupdates}$.
Recall that $\globvarparam{+} = \nagent^{-1} \sum_{c=1}^\nagent \locvarparam{c}{\nlupdates}$.
We can thus use the fact that $\sum_{c=1}^\nagent \cvar{c}{} = 0$ and $\sum_{c=1}^\nagent \cvarlim{c} = 0$, as well as \Cref{lem:projection} with $x_c = \locvarparam{c}{\nlupdates} + \step \nlupdates \cvar{c}{}$ and $y_c = \paramlim + \step \nlupdates \cvarlim{c}$ to obtain
\begin{align}
\nonumber
\norm{ \globvarparam{+} -  \paramlim }^2
& =
\bnorm{ 
\frac{1}{\nagent} \sum_{c=1}^\nagent \left(
\locvarparam{c}{\nlupdates} + \step \nlupdates \cvar{c}{}\right)
- 
\frac{1}{\nagent} \sum_{c=1}^\nagent \left( \paramlim
+ \step \nlupdates \cvarlim{c}\right)
}^2
\\
\label{eq:lyapunov-after-pythagoras}
& =
\frac{1}{\nagent}
\sum_{c=1}^\nagent
\bnorm{ \locvarparam{c}{\nlupdates} 
+ \step \nlupdates \cvar{c}{} 
- \paramlim
- \step \nlupdates \cvarlim{c}}^2
- 
\frac{1}{\nagent}
\sum_{c=1}^\nagent
\bnorm{ \step \nlupdates\left( 
\cvar{c}{+}
- \cvarlim{c} \right) }^2
\eqsp,
\end{align}
where we used the fact that $\step \nlupdates 
\cvar{c}{+} =
\step \nlupdates \cvar{c}{} + \locvarparam{c}{\nlupdates} - \globvarparam{+}$ in the second term.
Define  for $\varparam \in \rset^d$, $c \in \iint{1}{\nagent}$ and $h \in \iint{0}{\nlupdates}$
\begin{align}
\label{eq:def-adjusted-operator-xi-bound}
\locshiftvarparam{c}{h}
& =
\locscafopabv{c}{h}{\varparam; \cvar{c}{}} + \step h (\cvar{c}{} - \cvarlim{c})
=
\locvarparam{c}{h} + \step h (\cvar{c}{} - \cvarlim{c})
\eqsp.
\end{align}
The identity in \eqref{eq:lyapunov-after-pythagoras} can be rewritten using this expression, as well as the norm $\norm{ \cdot }[\Lambda]$ defined in \eqref{eq:def-lambda-norm},
\begin{align}
\label{eq:expression-updated-lyapunov-bound}
\norm{ \opscaffold(\bigX; \randStatew) - \bigXlim }[\Lambda]^2
=
\frac{1}{\nagent}
\sum_{c=1}^\nagent
\norm{ \locshiftvarparam{c}{\nlupdates} - \paramlim }^2
\eqsp.
\end{align}
It remains to derive a bound on each term of this sum, by induction on $h \in \iint{0}{\nlupdates-1}$.
We have, for $c \in \iint{1}{\nagent}$,
\begin{align*}
\norm{ \locshiftvarparam{c}{h+1} - \paramlim }^2
& =
\bnorm{ \locshiftvarparam{c}{h} - \paramlim
- \step 
\left( 
\gnfs{c}{\locvarparam{c}{h}}{\locRandState{c}{h+1}}  
+ \cvarlim{c}
\right)
}^2
\eqsp.
\end{align*}
Expanding the square and using \eqref{eq:def-adjusted-operator-xi-bound} to write $\locshiftvarparam{c}{h} = \locvarparam{c}{h} + \step h (\cvar{c}{} - \cvarlim{c})$, we obtain
\begin{align*}
\norm{ \locshiftvarparam{c}{h+1} - \paramlim }^2
& =
\bnorm{ \locshiftvarparam{c}{h} - \paramlim }^2
- 2 \step
\bpscal{ \locshiftvarparam{c}{h} - \paramlim }{ 
\gnfs{c}{\locvarparam{c}{h}}{\locRandState{c}{h+1}}  
+ \cvarlim{c} 
}
+
\step^2
\bnorm{ 
\gnfs{c}{ \locvarparam{c}{h} }{\locRandState{c}{h+1}}  
+ \cvarlim{c} 
}^2
\\
& =
\bnorm{ \locshiftvarparam{c}{h} - \paramlim }^2
+
\step^2
\bnorm{ 
\gnfs{c}{\locvarparam{c}{h}}{\locRandState{c}{h+1}}  
+ \cvarlim{c} 
}^2
\\
& \quad
- 2 \step
\bpscal{ \locvarparam{c}{h} - \paramlim }{ 
\gnfs{c}{\locvarparam{c}{h}}{\locRandState{c}{h+1}}  
+ \cvarlim{c} 
}
- 2 \step^2 h
\bpscal{ 
\cvar{c}{} - \cvarlim{c}
}{ 
\gnfs{c}{\locvarparam{c}{h}}{\locRandState{c}{}}  
+ \cvarlim{c}
}
\eqsp.
\end{align*}
Replacing $\cvarlim{c} = - \gnf{c}{\paramlim}$, we have
\begin{align*}
\CPE{ \bnorm{ \locshiftvarparam{c}{h+1} - \paramlim }^2 }{\locRandState{c}{1:h}}
& =
\bnorm{ \locshiftvarparam{c}{h} - \paramlim }^2
+
\step^2
\CPE{
\bnorm{ 
\gnfs{c}{\locvarparam{c}{h}}{\locRandState{c}{h+1}}  
- \gnf{c}{\paramlim}
}^2
}{\locRandState{c}{1:h}}
\\
& ~
- 2 \step
\bpscal{ \locvarparam{c}{h} - \paramlim }{ 
\gnf{c}{\locvarparam{c}{h}} 
- \gnf{c}{\paramlim}
}
- 2 \step^2 h
\bpscal{ 
\cvar{c}{} - \cvarlim{c}
}{ 
\gnf{c}{\locvarparam{c}{h}}
- \gnf{c}{\paramlim}
}
\eqsp.
\end{align*}
Using the inequality $\norm{ u + v }^2 \le 2 \norm{u}^2 + \norm{v}^2$ and bounding the two terms using co-coercivity \eqref{eq:cocoercivity} and \Cref{assum:smooth-var}, we can bound
\begin{align*}
\CPE{
\bnorm{ 
\gnfs{c}{\locvarparam{c}{h}}{\locRandState{c}{h+1}}  
- \gnf{c}{\paramlim}
}^2
}{\locRandState{c}{1:h}}
& = 
\CPE{
\bnorm{ 
\gnfs{c}{\locvarparam{c}{h}}{\locRandState{c}{h+1}}  
\!-\! \gnfs{c}{\paramlim}{\locRandState{c}{h+1}}  
+ \gnfs{c}{\paramlim}{\locRandState{c}{h+1}}  
\!-\! \gnf{c}{\paramlim} 
}^2
}{\locRandState{c}{1:h}}
\\
& \le 
2 \lip \pscal{ \locvarparam{c}{h} - \paramlim }{ 
\gnf{c}{\locvarparam{c}{h}}
- \gnf{c}{\paramlim} 
}
+ 2 \optvar
\eqsp.
\end{align*}
Now, using Young's inequality to bound $2 \step^2 h a b = 2 (\step^{3/2} L^{1/2} h a) (\step^{1/2} L^{-1/2} b) \le \step^3 h^2 L a^2 + \step L^{-1} b^2$ and co-coervicity of the gradient \eqref{eq:cocoercivity}, we get
\begin{align*}
- 2 \step^2 h
\bpscal{ 
\cvar{c}{} - \cvarlim{c}
}{ 
\gnf{c}{\locvarparam{c}{h}}
- \gnf{c}{\paramlim}
}
&
\le
\step^3 h^2 \lip
\norm{ 
\cvar{c}{} - \cvarlim{c}
}^2
+
\frac{\step}{\lip} \norm{ 
\gnf{c}{\locvarparam{c}{h}}
- \gnf{c}{\paramlim}
}^2
\\
&
\le
\step^3 h^2 \lip
\norm{ 
\cvar{c}{} - \cvarlim{c}
}^2
+
\step
\pscal{ \locvarparam{c}{h} - \paramlim }{ \gnf{c}{\locvarparam{c}{h}}
- \gnf{c}{\paramlim} }
\eqsp.
\end{align*}
Plugging the last two equations in the inequality that decompose the update above, we have
\begin{align*}
\CPE{ \bnorm{ \locshiftvarparam{c}{h+1} - \paramlim }^2 }{\locRandState{c}{1:h}}
& \le
\bnorm{ \locshiftvarparam{c}{h} - \paramlim }^2
+
\step^3 h^2 \lip
\norm{ 
\cvar{c}{} - \cvarlim{c}
}^2
\\
\nonumber
& \quad 
- (\step - 2 \step^2 \lip)
\pscal{ \locvarparam{c}{h} - \paramlim }{ 
\gnf{c}{ \locvarparam{c}{h} }
- \gnf{c}{\paramlim} 
}
+ 2 \step^2 \optvar
\eqsp.
\end{align*}
And using the fact that $\step \le 1/4\lip$ to bound $- (\step - 2 \step^2 \lip) \le - \step / 2$, and the monotonocity of the gradient \eqref{eq:monotonicity}, we obtain
\begin{align}
\label{eq:bound-tilde-tilde-pscal-bound}
\CPE{ \bnorm{ \locshiftvarparam{c}{h+1} - \paramlim }^2 }{\locRandState{c}{1:h}}
& \le
\bnorm{ \locshiftvarparam{c}{h} - \paramlim }^2
- \frac{\step \strcvx}{2}
\bnorm{ \locvarparam{c}{h} - \paramlim }^2
+
\step^3 h^2 \lip
\norm{ 
\cvar{c}{} - \cvarlim{c}
}^2
+ 2 \step^2 \optvar
\eqsp.
\end{align}
Now, we remark that, for $a, b \in \rset^d$, we have $\norm{a}^2 = \norm{a-b+b}^2 \le 2 \norm{a-b}^2 + 2 \norm{b}^2$, which implies that $- \norm{a-b}^2 \le - \frac{1}{2} \norm{a}^2 + \norm{b}^2$.
Therefore, we have
\begin{align*}
- \frac{\step \strcvx}{2}
\bnorm{ \locvarparam{c}{h} - \paramlim }^2
& =
- \frac{\step \strcvx}{2}
\bnorm{ \locshiftvarparam{c}{h} - \paramlim - \step h (\cvar{c}{} - \cvarlim{c}) }^2
\le
- \frac{\step \strcvx}{4}
\bnorm{ \locshiftvarparam{c}{h} - \paramlim  }^2
+ \frac{\step^3 h^2 \strcvx}{2}
\bnorm{\cvar{c}{} - \cvarlim{c}}^2
\eqsp.
\end{align*}
Using this inequality in \eqref{eq:bound-tilde-tilde-pscal-bound}, we obtain the following inequality
\begin{equation}
\label{eq:one-step-contraction-abritrary-bound}
\begin{aligned}
\CPE{ \bnorm{ \locshiftvarparam{c}{h+1} - \paramlim }^2 }{\locRandState{c}{1:h}}
& \le
\left(1 - \frac{\step \strcvx}{4} \right) 
\bnorm{ \locshiftvarparam{c}{h} - \paramlim }^2
+ \left(\step^3 h^2 \lip + \step^3 h^2 \strcvx / 2 \right)
\norm{ 
\cvar{c}{} - \cvarlim{c}
}^2
+ 2 \step^2 \optvar
\eqsp.
\end{aligned}
\end{equation}
Applying \eqref{eq:one-step-contraction-abritrary-bound} recursively, we obtain
\begin{equation}
\label{eq:unrol_norm_2_lyap}
\begin{aligned}
& \PE\left[\bnorm{ \locshiftvarparam{c}{\nlupdates} - \paramlim }^2\right]
 \le
\left(1 - \frac{\step \strcvx}{4}\right)^\nlupdates \norm{ \param - \paramlim }^2
+
\frac{\step^3 \nlupdates^2 (\nlupdates-1) (\lip + \strcvx)}{2}
\norm{ 
\cvar{c}{} - \cvarlim{c}
}^2
+ 2 \step^2 \nlupdates \optvar
\eqsp.
\end{aligned}
\end{equation}
Consequently, whenever $\step \nlupdates (\lip + \strcvx) \le 1$, we can sum this inequality for $c = 1$ to $\nagent$ to obtain
\begin{align}
\label{eq:upper-bound-intermediate-sum-theta-tilde}
\frac{1}{\nagent} \sum_{c=1}^\nagent \PE \left[\norm{ \locshiftvarparam{c}{\nlupdates} - \paramlim }^2 \right]
& \le
\left(1 - \frac{\step \strcvx}{4}\right)^\nlupdates \norm{ \param - \paramlim }^2
+
\frac{1}{2}
\frac{\step^2\nlupdates^2}{\nagent} \sum_{c=1}^\nagent 
\norm{ 
\cvar{c}{} - \cvarlim{c}
}^2
+ 2 \step^2 \nlupdates \optvar
\\
\nonumber
& \le
\left(1 - \frac{\step \strcvx}{4}\right)^\nlupdates  \norm{ \bigX - \bigXlim }[\Lambda]^2
+ 2 \step^2 \nlupdates \optvar
\eqsp,
\end{align}
and we get the result of the lemma by taking the expectation of \eqref{eq:expression-updated-lyapunov-bound} and plugging this bound.
\end{proof}

\crudeboundtwox*
\begin{proof}
The proof follows by applying recursively \Cref{lem:descent-noise}  with the natural filtration of the process $(\bigX^t)_{t=0}^\infty$.
\end{proof}

The following corollary is a direct consequence of \Cref{thm:crude-bound-2-X}, and gives a crude bound on the squared error of $\param$ in the stationary distribution.
\begin{corollary}
\label{cor:crude-bounds-global}
Assume \Cref{assum:strong-convexity}, \Cref{assum:smoothness} and \Cref{assum:smooth-var}.
Let $\randStatew = \locRandState{1:\nagent}{1:\nlupdates}$ be i.i.d. random variables satisfying \Cref{assum:smooth-var}.
Let $\step > 0$ be the step size and $\nlupdates > 0$ the number of local updates of S{\scriptsize CAFFOLD}.
Assume that $\step \le 1 / 4 \lip$ and $\step \nlupdates (\lip + \strcvx) \le 1$.
Then, for all $h \in \iint{0}{\nlupdates}$,  it holds that
\begin{gather}
\label{eq:cor-crude-bound-theta}
  \int  \bnorm{ \param - \paramlim }^2 \statdist{\step, \nlupdates}(\rmd \param, \rmd \Cvarw)
  \le \frac{8 \step}{\strcvx} \optvar
    \eqsp,
\\
\label{eq:cor-crude-bound-xi-sum}
  \frac{\step^2 \nlupdates^2}{\nagent} \sum_{c=1}^\nagent \int  \bnorm{ \cvar{c}{} - \cvarlim{c} }^2
  \statdist{\step, \nlupdates}(\rmd \param, \rmd \Cvarw)
  \le \frac{8 \step}{\strcvx} \optvar
    \eqsp,
\\
\label{eq:cor-crude-bound-theta-loc-sum}
  \frac{1}{\nagent} \sum_{c=1}^\nagent \int  \PE\left[ \bnorm{ \locstoscafop{c}{h}{\param}{\cvar{c}{}}{\locRandState{c}{1:h}} - \paramlim }^2 \right] \statdist{\step, \nlupdates}(\rmd \param, \rmd \Cvarw)
  \le \frac{8 \step}{\strcvx} \optvar
    \eqsp,
\end{gather}
where $\Cvarw = (\cvar{1}{}, \dots, \cvar{\nagent}{}) \in \rset^{\nagent \times d}$.
\end{corollary}
\begin{proof}
Inequalities \eqref{eq:cor-crude-bound-theta} and \eqref{eq:cor-crude-bound-xi-sum} follow from \Cref{thm:crude-bound-2-X}.
The third inequality \eqref{eq:cor-crude-bound-theta-loc-sum} is obtained by unrolling \eqref{eq:one-step-contraction-abritrary-bound} until $h$ similarly to \eqref{eq:unrol_norm_2_lyap} and summing over $c=1$ to $\nagent$.
\end{proof}

\subsection{Bounds on \Scaffold's Local Iterates and Control Variates in the Stationary Distribution -- Proof of Lemma~\ref{lem:crude-bound-local-and-cvar}}
\label{sec:bound:stationary}

\begin{lemma}
\label{lem:crude-bound-local-and-cvar-appendix}
Assume \Cref{assum:strong-convexity}, \Cref{assum:smoothness}, \Cref{assum:smooth-var}.
Let $\randStatew = \locRandState{1:\nagent}{1:\nlupdates}$ be i.i.d. random variables satisfying \Cref{assum:smooth-var}.
Assume the step size $\step$ and the number of local updates $\nlupdates$ satisfy $\step \nlupdates (\lip + \strcvx) \leq 1/12$ and $\step \smoothcstvar \le \lip$. Under these conditions, it holds that
\begin{gather}
\label{eq:bound-local-iterate-h-interm}
\int \PE\left[ \norm{  \locstoscafop{c}{h}{\param}{\cvar{c}{}}{\locRandState{c}{1:h}} - \paramlim }^2 \right]
\statdist{\step, \nlupdates}(\rmd \theta, \rmd \Cvarw)
  \le 
\frac{18 \step}{\strcvx} \optvar
+ 3 \step^2\nlupdates^2 
\int 
\norm{ \cvar{c}{} - \cvarlim{c} }^2
\statdist{\step, \nlupdates}(\rmd \theta, \rmd \Cvarw) 
\eqsp,
\\
\label{eq:bound-local-xi-interm}
\int \norm{ \cvar{c}{} \!-\! \cvarlim{c}  }^2 \statdist{\step, \nlupdates}(\rmd \theta, \rmd \Cvarw)
\le 
\frac{8(\lip+\strcvx)}{\strcvx \nlupdates} \optvar
+ \frac{4 \lip^2 + 2\smoothcstvar}{\nlupdates} 
\sum_{h=0}^{\nlupdates\!-\!1}
\int\! \PE\!\left[ \norm{ \locstoscafop{c}{h}{\param}{\cvar{c}{}}{\locRandState{c}{1:h}} \!-\! \paramlim }^2 \right]
\!\statdist{\step, \nlupdates}(\rmd \theta, \rmd \Cvarw)
\eqsp.
\end{gather}
\end{lemma}

\begin{proof}
Let $\param \in \mathbb{R}^d$ and $\{ \cvar{c}{}\}_{c=1}^\nagent \subset \mathbb{R}^d$ that satisfy the constraints
$\sum_{c=1}^\nagent \cvar{c}{} = 0$.
Based on the proof of \Cref{lem:descent-noise}, we define a notation for the local parameters and their counterpart with ideal control variates, 
\begin{align}
\label{eq:def-adjusted-operator-xi-bound-bis}
\locparam{c}{h}
& =
\locscafopabv{c}{h}{\param; \cvar{c}{},\locRandState{c}{1:h}}
\\
\locshiftparam{c}{h}
&
=
\locparam{c}{h} + \step h (\cvar{c}{} - \cvarlim{c})
\eqsp.
\end{align}
\textbf{Bound on the local iterates.}
Then, following the same lines of proof as \Cref{lem:descent-noise} (see \eqref{eq:one-step-contraction-abritrary-bound}) and using the fact that $\step \nlupdates \lip \le 1$, we obtain, for any $h \le \nlupdates$, and $c \in \iint{1}{\nagent}$,
\begin{align*}
\expe{ \norm{ \locshiftparam{c}{h} - \paramlim }^2 }
& \le
\norm{ \param - \paramlim }^2
+ \frac{\step^2\nlupdates^2}{2} 
\norm{ \cvar{c}{} - \cvarlim{c} }^2
+ 2 \step^2 \nlupdates \optvar
\eqsp,
\end{align*}
Since $\locparam{c}{h} = \locshiftparam{c}{h} + \step h (\cvar{c}{} - \cvarlim{c})$, this gives the inequality 
\begin{align*}
\expe{ \norm{ \locparam{c}{h} - \paramlim }^2}
& \le
2 \PE \left[ \norm{ \locshiftparam{c}{h} - \paramlim }^2 \right]
+ 2 \step^2 h^2\norm{ \cvar{c}{} - \cvarlim{c} }^2 
\\
& \le
2 \norm{ \param - \paramlim }^2
+ 3 \step^2\nlupdates^2 
\norm{ \cvar{c}{} - \cvarlim{c} }^2
+ 4 \step^2 \nlupdates \optvar
\eqsp.
\end{align*}
Integrating over the stationary distribution of \Scaffold's iterates and using~\eqref{eq:cor-crude-bound-theta} from \Cref{cor:crude-bounds-global}, we obtain~\eqref{eq:bound-local-iterate-h-interm}.

\textbf{Bound on control variates.}
For ease of notation, we define
\begin{equation}
\label{eq:shorthand-epsilon}
\locnoiseabv{c}{h+1}= \updatefuncnoise{c}{\locRandState{c}{h+1}}{\locparam{c}{h}} \eqsp.
\end{equation}
Let $c \in \iint{1}{\nagent}$, the control variate update can be written as
\begin{align*}
\cvar{c}{+}
& =
\cvar{c}{}
+ \frac{1}{\step \nlupdates} \left(
\param - \step \sum_{h=0}^{\nlupdates-1} \gnf{c}{\locparam{c}{h}} + \cvar{c}{} + \locnoiseabv{c}{h+1}
- \param + \frac{\step}{\nagent} \sum_{i=1}^\nagent\sum_{h=0}^{\nlupdates-1} \gnf{i}{\locparam{i}{h}} + \cvar{i}{} + \locnoiseabv{i}{h+1}
\right)
\\
& =
\cvar{c}{}
- \frac{1}{\step \nlupdates} \left(
\step \sum_{h=0}^{\nlupdates-1} \gnf{c}{\locparam{c}{h}} + \cvar{c}{} + \locnoiseabv{c}{h+1}
- \frac{\step}{\nagent} \sum_{i=1}^\nagent\sum_{h=0}^{\nlupdates-1} \gnf{i}{\locparam{i}{h}} + \cvar{i}{} + \locnoiseabv{i}{h+1}
\right)
\eqsp.
\end{align*}
Using $\sum_{i=1}^\nagent \cvar{i}{} = 0$, $\sum_{i=1}^\nagent \gnf{i}{\paramlim} = 0$, 
$\cvarlim{c}= - \gnf{c}{\paramlim}$, and reorganizing the terms, this gives
\begin{align}
\cvar{c}{+} - \cvarlim{c}
& =
\cvar{c}{} - \cvarlim{c}
- \frac{1}{\nagent \nlupdates} 
\sum_{i=1}^\nagent
\sum_{h=0}^{\nlupdates-1}
\left(
\gnf{c}{\locparam{c}{h}} + \cvar{c}{} 
- \gnf{i}{\locparam{i}{h}} 
+ \locnoiseabv{c}{h+1} - \locnoiseabv{i}{h+1}
\right)
\\
\label{eq:expression-xi-plus-fct-grads}
& =
\frac{1}{\nagent \nlupdates} 
\sum_{i=1}^\nagent
\sum_{h=0}^{\nlupdates-1}
\left(
\left(\gnf{i}{\locparam{i}{h}} - \gnf{i}{\paramlim} \right)
- \left(\gnf{c}{\locparam{c}{h}} - \gnf{c}{\paramlim}\right)
+ \locnoiseabv{i}{h+1} - \locnoiseabv{c}{h+1}
\right)
\eqsp.
\end{align}
Taking the squared norm and expectation of \eqref{eq:expression-xi-plus-fct-grads}, we obtain 
\begin{align*}
\PE\left[ \norm{ \cvar{c}{+} - \cvarlim{c} }^2 \right]
& \le 
2 \PE\left[ \bnorm{ \frac{1}{\nagent \nlupdates} 
\sum_{i=1}^\nagent
\sum_{h=0}^{\nlupdates-1}
\left(\gnf{i}{\locparam{i}{h}} - \gnf{i}{\paramlim} \right)
- \left(\gnf{c}{\locparam{c}{h}} - \gnf{c}{\paramlim}\right)
}^2 \right]
\\
& \quad
+
2 \PE\left[ \bnorm{ \frac{1}{\nagent \nlupdates} 
\sum_{i=1}^\nagent
\sum_{h=0}^{\nlupdates-1}
 \locnoiseabv{i}{h+1} - \locnoiseabv{c}{h+1}
}^2 \right]
\eqsp.
\end{align*} 
Using Jensen's inequality, as well as $\CPE{ \locnoiseabv{c}{h+1}}{\locRandState{1:N}{1:h}}=0$ a.s. and $\CPE{ \locnoiseabv{c}{h+1} \locnoiseabv{i}}{\locRandState{1:N}{1:h}}=0$ for all $c,i \in \iint{1}{\nagent}$, $i\ne c$ and $h \in \iint{0}{\nlupdates-1}$, we have
\begin{align*}
\PE\left[ \norm{ \cvar{c}{+} - \cvarlim{c} }^2 \right]
& \le 
\frac{4}{\nagent \nlupdates} 
\sum_{i=1}^\nagent
\sum_{h=0}^{\nlupdates-1}
\PE\left[ 
\norm{ \gnf{i}{\locparam{i}{h}} - \gnf{i}{\paramlim} }^2 
+ \norm{ \gnf{i}{\locparam{c}{h}} - \gnf{c}{\paramlim} }^2 \right]
\\
& \quad
+
\frac{2}{\nagent \nlupdates^2}  
\sum_{i=1}^\nagent
\sum_{h=0}^{\nlupdates-1}
\PE\left[ \norm{ \locnoiseabv{i}{h+1} }^2 
+ \norm{ \locnoiseabv{c}{h+1} }^2 \right]
\eqsp.
\end{align*}
By Lipschitzness of the gradient \eqref{eq:grad-lipschitz} and smoothness of the error noise (\Cref{assum:smooth-var}),
\begin{align*}
\PE\left[ \norm{ \cvar{c}{+} - \cvarlim{c} }^2 \right]
& \le 
\frac{4 \lip^2}{\nagent \nlupdates} 
\sum_{i=1}^\nagent
\sum_{h=0}^{\nlupdates-1}
\PE\left[ 
\norm{ \locparam{i}{h} - \paramlim }^2 
+ \norm{ \locparam{c}{h} - \paramlim }^2 \right]
\\
& \quad
+
\frac{2}{\nagent \nlupdates^2}  
\sum_{i=1}^\nagent
\sum_{h=0}^{\nlupdates-1}
\left\{ \smoothcstvar \PE\left[ 
\norm{ \locparam{i}{h} - \paramlim }^2 
+ \norm{ \locparam{c}{h} - \paramlim }^2 \right] 
+ 4 \optvar \right\}
\\
& \le 
\frac{8}{\nlupdates} \optvar
+ \frac{4 \lip^2 + 2 \smoothcstvar}{\nagent \nlupdates} 
\sum_{i=1}^\nagent
\sum_{h=0}^{\nlupdates-1}
\PE\left[ 
\norm{ \locparam{i}{h} - \paramlim }^2 
+ \norm{ \locparam{c}{h} - \paramlim }^2 \right]
\eqsp.
\end{align*}
Integrating over the stationary distribution of \Scaffold's iterates and using~\eqref{eq:cor-crude-bound-theta-loc-sum} from \Cref{cor:crude-bounds-global} gives inequality~\eqref{eq:bound-local-xi-interm}.
\end{proof}

\crudeboundlocalandcvar*
\begin{proof}
\textbf{Solving the system of inequations.}
We now aim to find constants $\boundlocparam{c}$ and $\boundcvar{c}$, for $c \in \iint{1}{\nagent}$, such that for all $h \in \iint{0}{\nlupdates}$,
\begin{align*}
\int \PE \left[ \norm{  \locstoscafop{c}{h}{\param}{\cvar{c}{}}{\locRandState{c}{1:h}} - \paramlim }^2 \right] \statdist{\step, \nlupdates}(\rmd \theta, \rmd \Cvarw)
& \le \boundlocparam{c}
\eqsp,
\quad 
\text{ and }
\quad
\int \norm{ \cvar{c}{} - \cvarlim{c}  }^2 \statdist{\step, \nlupdates}(\rmd \theta, \rmd \Cvarw) 
 \le \boundcvar{c}
\eqsp.
\end{align*}
By the first part of the lemma, we have
\begin{align*}
\boundlocparam{c} 
& \le \frac{18 \step}{\strcvx} \optvar + 3 \step^2 \nlupdates^2 \boundcvar{c}
\eqsp,
\quad 
\text{ and }
\quad
\boundcvar{c}
\le \frac{9 \lip}{\strcvx \nlupdates} \optvar + 12 \lip^2 \boundlocparam{c}
\eqsp.
\end{align*}
Since $\step \nlupdates \lip \le 1/12$, this implies that
\begin{align*}
\boundlocparam{c} 
& \le \frac{18 \step}{\strcvx} \optvar 
+ \frac{27 \step^2 \nlupdates \lip}{\strcvx} \optvar
+ 36 \step^2 \nlupdates^2 \lip^2 \boundlocparam{c}
\le \frac{21 \step}{\strcvx} \optvar 
+ \frac{1}{4} \boundlocparam{c}
\eqsp,
\\
\boundcvar{c}
& \le 
\frac{16 \lip}{\strcvx \nlupdates} \optvar 
+ \frac{12 \cdot 28 \step \lip^2}{\strcvx} \optvar
+ 36 \step^2 \nlupdates^2 \lip^2 \boundcvar{c}
\le
\frac{40 \lip}{\strcvx \nlupdates} \optvar
+ \frac{1}{4} \boundcvar{c}
\eqsp,
\end{align*}
and the result follows.
\end{proof}

\subsection{Higher-order bounds}
We now derive bounds on the moments of the error, up to the sixth moment.
\begin{lemma}
\label{lem:descent-noise-powsix}
Assume \Cref{assum:strong-convexity}, \Cref{assum:smoothness} and \Cref{assum:smooth-var}.
Let $\param \in \mathbb{R}^d$ and $\{ \cvar{c}{}\}_{c=1}^\nagent \subset \mathbb{R}^d$ that satisfy the constraint $\sum_{c=1}^\nagent \cvar{c}{} = 0$.
Define the global iterate vector
$\bigX = (\param, \cvar{1}{}, \dots, \cvar{\nagent}{})$, 
and the optimal vector
$ \bigX^\star = (\paramlim, \cvarlim{1}, \dots, \cvarlim{\nagent})$.
Further, let $\randStatew = \locRandState{1:\nagent}{1:\nlupdates}$ be a collection of i.i.d.\ random variables such that for any $c\in\iint{1}{N}$ and $h \in \iint{1}{H}$, $Z_{(c)}^h \sim \nu_{(c)}$.

Assume the step size $\step$ and the number of local updates $\nlupdates$ satisfy $\step \lip \le 1/48$, $\step \nlupdates (\lip + \strcvx) \leq 1/24$.
Then,
\begin{align}
\PE\left[ \norm{ \opscaffold(\bigX; \randStatew) - \bigXlim }[\Lambda]^6 \right]^{1/3}
\le
\left( 1 - \step \strcvx / 6\right)^\nlupdates
\norm{ \bigX - \bigXlim }[\Lambda]^2 
+ 40 \step^2 \nlupdates \optvar
\eqsp.
\end{align}
\end{lemma}
\begin{proof}
We denote, for $\param, \cvar{1}{}, \dots, \cvar{\nagent}{} \in \rset^d$, notations for the global parameter update, the local parameter updates and the control variates updates as,
\begin{align*}
\globparam{+}
= \globstoscafop{\nlupdates}{\param}{\cvar{1:\nagent}{}}{\locRandState{1:\nagent}{1:\nlupdates}}
\eqsp,
\qquad
\locparam{c}{h}
= \locstoscafop{c}{h}{\param}{\cvar{c}{}}{\locRandState{c}{1:h}}
\eqsp,
\qquad
\cvar{c}{+}
= \scafopcv{c}{\nlupdates}{\cvar{c}{}}{\param}{\locRandState{c}{1:\nlupdates}}
\eqsp,
\end{align*}
for $c \in \iint{1}{\nagent}$ and $h \in \iint{0}{\nlupdates}$,
as well as the shifted local parameters
\begin{align*}
\locshiftparam{c}{h}
= \locparam{c}{h} + \step h (\cvar{c}{} - \cvarlim{c})
\eqsp.
\end{align*}
We recall the identity from \eqref{eq:expression-updated-lyapunov-bound},
\begin{align*}
\norm{ \opscaffold(\bigX; \randStatew) - \bigXlim }[\Lambda]^6
& =
\Big(
\norm{ \globparam{+} - \paramlim }^2 
+ \frac{\step^2 \nlupdates^2}{\nagent} \sum_{c=1}^\nagent \norm{ \cvar{c}{} - \cvarlim{c} }^2
\Big)^3
=
\Big( \frac{1}{\nagent} \sum_{c=1}^\nagent 
\norm{ \locshiftparam{c}{\nlupdates} - \paramlim }^2
\Big)^3
\eqsp.
\end{align*}
Thus, using Hölder's inequality, we have
\begin{align}
\label{eq:higher-after-minkowski}
\PE\left[ \norm{ \opscaffold(\bigX; \randStatew) - \bigXlim }[\Lambda]^6 \right]^{1/3}
& \le
\frac{1}{\nagent} \sum_{c=1}^\nagent 
\PE\left[ \norm{ \locshiftparam{c}{\nlupdates} - \paramlim }^6 \right]^{1/3}
\eqsp.
\end{align}
We proceed by induction.
Expanding the second power, for $h \in \iint{0}{\nlupdates-1}$, using $\smash{\cvarlim{c} = - \gnf{c}{\paramlim}}$,
\begin{align*}
\norm{ \locshiftparam{c}{h+1} - \paramlim }^2
& =
\bnorm{ \locshiftparam{c}{h} - \paramlim
- \step \Big(
\gnfs{c}{\locparam{c}{h}}{\locRandState{c}{h+1}} 
- \gnf{c}{\paramlim} 
\Big)
}^2
\\
& =
\bnorm{ \locshiftparam{c}{h} - \paramlim }^2
- 2 \step \bpscal{ \locshiftparam{c}{h} - \paramlim }
{\gnfs{c}{\locparam{c}{h}}{\locRandState{c}{h+1}} - \gnf{c}{\paramlim} }
+
\step^2 \bnorm{\gnfs{c}{\locparam{c}{h}}{\locRandState{c}{h+1}} - \gnf{c}{\paramlim} }^2
\eqsp.
\end{align*}
Now, we compute the third power of this equality.
We write it as $(a^2 - 2 \step b + \step^2 c^2)^3$, with
\begin{align*}
a^2 & = \norm{ \locshiftparam{c}{h} - \paramlim }^2
\eqsp, 
\\
\qquad 
- 2 \step b  
&=
- 2 \step \bpscal{ \locshiftparam{c}{h} - \paramlim }
{\gnfs{c}{\locparam{c}{h}}{\locRandState{c}{h+1}} - \gnf{c}{\paramlim} }
\eqsp,
\\
\step^2 c^2
& =
\step^2 \bnorm{\gnfs{c}{\locparam{c}{h}}{\locRandState{c}{h+1}} - \gnf{c}{\paramlim} }^2
\eqsp.
\end{align*}
We remark that $\abs{ b } \le a c$, which gives
\begin{align}
\nonumber
& \norm{ \locshiftparam{c}{h+1} - \paramlim }^6
=
\left( a^2 - 2 \step b + \step^2 c^2 \right)^{3}
\\
\nonumber
& \quad =
a^6 - 6 \step a^4b + 3\step^2a^4c^2 
+ 12\step^2a^2b^2 - 12\step^3a^2bc^2 + 3\step^4a^2c^4 
- 8\step^3b^3 + 12\step^4b^2c^2 - 6\step^5bc^4 + \step^6c^6 
\\
\nonumber
& \quad \le a^6 - 6 \step a^4b + 3 \step^2 a^4c^2 
+ 12\step^2a^4c^2 + 12 \step^3a^3c^3 + 3 \step^4a^2c^4
+ 8 \step^3a^3c^3 + 12\step^4a^2c^4 + 6\step^5ac^5 + \step^6c^6
\\
\label{eq:full-expansion-powthree-abc}
& \quad = a^6 - 6\step a^4b + 15 \step^2a^4c^2 + 20 \step^3a^3c^3 + 15 \step^4a^2c^4 + 6 \step^5ac^5 + \step^6c^6
\eqsp.
\end{align}
Remark that $a$ is $\sigma(\locRandState{c}{1:h})$-measurable.
Since $\locshiftparam{c}{h}
= \locparam{c}{h} + \step h (\cvar{c}{} - \cvarlim{c})$, we can split the dot product $b$ similarly to \Cref{lem:descent-noise}'s proof, using Young's inequality to bound $\pscal{u}{v} \le 1/6 \norm{u}^2 + 6\norm{v}^2$ for any two vectors $u, v \in \rset^d$,
\begin{align*}
& \CPE{ - 6\step a^4b }{\locRandState{c}{1:h}}
=
- 6 \step a^4
\pscal{ \locshiftparam{c}{h} - \paramlim }
{\gnf{c}{\locparam{c}{h}} - \gnf{c}{\paramlim} }
\\
& =
a^4
\left(
- 6 \step \pscal{ \locparam{c}{h} - \paramlim }
{\gnf{c}{\locparam{c}{h}} - \gnf{c}{\paramlim} }
- 6 \step^2 h
\pscal{ \cvar{c}{} - \cvarlim{c} }
{\gnf{c}{\locparam{c}{h}} - \gnf{c}{\paramlim} }
\right)
\\
& \le
a^4
\left(
- 6 \step 
\pscal{ \locparam{c}{h} - \paramlim }
{\gnf{c}{\locparam{c}{h}} - \gnf{c}{\paramlim} }
+ 36 \step^3 h^2 \lip%
\norm{ \cvar{c}{} - \cvarlim{c} }^2
+ \frac{\step}{\lip} %
\norm{\gnf{c}{\locparam{c}{h}} - \gnf{c}{\paramlim} }^2
\right)
\eqsp.
\end{align*}
Which gives, by co-coercivity of the gradient \eqref{eq:cocoercivity},
\begin{align*}
\CPE{ - 6\step a^4b }{\locRandState{c}{1:h}}
& \le
- 5 \step a^4
\pscal{ \locparam{c}{h} - \paramlim }
{\gnf{c}{\locparam{c}{h}} - \gnf{c}{\paramlim} }
+ 36 \step^3 h^2 \lip a^4 \norm{ \cvar{c}{} - \cvarlim{c} }^2
\eqsp.
\end{align*}
Furthermore, we have, by Lipschitzness of the gradient \eqref{eq:grad-lipschitz}, and smoothness of the error noise (\Cref{assum:smooth-var}), and using the definition $\locshiftparam{c}{h}
= \locparam{c}{h} + \step h (\cvar{c}{} - \cvarlim{c})$, as well as the fact that $(x+y+z)^k \le 3^{k-1}(x^k+y^k+z^k)$ for $2 \le k \le 6$,
\begin{align}
\nonumber
\CPE{ \step^k a^{6-k} c^k }{\locRandState{c}{1:h}}
& =
\step^k a^{6-k}
\CPE{ \bnorm{\gnfs{c}{\locparam{c}{h}}{\locRandState{c}{h+1}} - \gnf{c}{\paramlim} }^k }{\locRandState{c}{1:h}}
\\
\nonumber
& \le
2^{k-1} \step^k
a^{6-k}
\left\{ 
\CPE{ \bnorm{\gnfs{c}{\locparam{c}{h}}{\locRandState{c}{h+1}} - \gnfs{c}{\paramlim}{\locRandState{c}{h+1}} }^k }{\locRandState{c}{1:h}}
+ \sqoptvar^k
\right\}\\
\label{eq:dev-higher-order-bound-a6kck}
& \le
2^{k-1} \step^k
a^{6-k} \{ 2 \lip^{k-1} a^{k-2}
\pscal{ \locparam{c}{h} - \paramlim }
{\gnf{c}{\locparam{c}{h}} - \gnf{c}{\paramlim} }
+ \sqoptvar^k
\}
\eqsp,
\end{align}
where we used \eqref{eq:grad-lipschitz} to bound \Cref{assum:smoothness} and $\norm{\gnfs{c}{\locshiftparam{c}{h}}{\locRandState{c}{h+1}} - \gnfs{c}{\paramlim}{\locRandState{c}{h+1}} } \le  L \norm{\locshiftparam{c}{h} - \paramlim } = L a$ in the last inequality. 
Taking the conditional expectation of \eqref{eq:full-expansion-powthree-abc} and plugging \eqref{eq:dev-higher-order-bound-a6kck}, we have
\begin{align*}
& \CPE{ \norm{ \locshiftparam{c}{h+1} - \paramlim }^6 }{\locRandState{c}{1:h}}
= \CPE{ a^6 - 6\step a^4b + 15 \step^2a^4c^2 + 20 \step^3a^3c^3 + 15 \step^4a^2c^4 + 6 \step^5ac^5 + \step^6c^6 }{\locRandState{c}{1:h}}
\\
& \le
a^6 
- 5 \step a^4
\bpscal{ \locparam{c}{h} - \paramlim }
{\gnf{c}{\locparam{c}{h}} - \gnf{c}{\paramlim} }
+ 36 \step^3 h^2 \lip a^4 \bnorm{ \cvar{c}{} - \cvarlim{c} }^2
\\
& \quad
+ 20 a^4 \sum_{k=2}^6 (2\step)^k \lip^{k-1}
\bpscal{ \locparam{c}{h} - \paramlim }
{\gnf{c}{\locparam{c}{h}} - \gnf{c}{\paramlim} }+  20 \sum_{k=2}^6 (2\step \sqoptvar)^k a^{6-k}
\eqsp.
\end{align*}
Letting $\step \lip \le 1/40$ to bound the second term, we get
\begin{align*}
\CPE{ \norm{ \locshiftparam{c}{h+1} - \paramlim }^6 }{\locRandState{c}{1:h}}
& \le
a^6 
- \step a^4
\bpscal{ \locparam{c}{h} - \paramlim }
{\gnf{c}{\locparam{c}{h}} - \gnf{c}{\paramlim} }
\\
& \quad 
+ 36 \step^3 h^2 \lip a^4 \bnorm{ \cvar{c}{} - \cvarlim{c} }^2
+ 20 \sum_{k=2}^6 (2\step \sqoptvar)^k a^{6-k}
\eqsp.
\end{align*}
As in the second-order bound, we use the monotonicity of the gradient \eqref{eq:monotonicity} to bound $- \step a^4
\bpscal{ \locparam{c}{h} - \paramlim }
{\gnf{c}{\locparam{c}{h}} - \gnf{c}{\paramlim} } \le - \step \strcvx a^4 \norm{ \locparam{c}{h} - \paramlim }^2$, which implies, using the fact that $-\norm{u}^2 \le - \frac{1}{2} \norm{u+v}^2 + \norm{v}^2$ for any pair of vectors $u, v \in \rset^d$,
\begin{align*}
- \step a^4
\bpscal{ \locparam{c}{h} - \paramlim }
{\gnf{c}{\locparam{c}{h}} - \gnf{c}{\paramlim} } 
\le
- \step \strcvx a^4 / 2 \norm{ \locshiftparam{c}{h} - \paramlim }^2
+ 
\step^3 h^2 \strcvx a^4 \norm{ \cvar{c}{} - \cvarlim{c} }^2
\eqsp.
\end{align*}
Finally, we obtain
\begin{align}
\nonumber
\CPE{ \norm{ \locshiftparam{c}{h+1} - \paramlim }^6 }{\locRandState{c}{1:h}}
& \le
(1 - \step \strcvx/2) a^6 
+ 36 \step^3 h^2 (\strcvx + \lip) a^4 \bnorm{ \cvar{c}{} - \cvarlim{c} }^2
+ 20 \sum_{k=2}^6 (2\step \sqoptvar)^k a^{6-k}
\\
\label{eq:expansion-powsix-contract-plus-sumallterms}
& \le
(1 - \step \strcvx/2) a^6 
+ 36 \step^3 h^2 (\strcvx + \lip) a^4 \bnorm{ \cvar{c}{} - \cvarlim{c} }^2
+ 30 \sum_{k=1}^3 (2\step \sqoptvar)^{2 k} a^{6-2 k}
\eqsp,
\end{align}
using for $k$ odd, $ (u v)^k \leq u^{k+1}v^{k-1}/2 + u^{k-1}v^{k+1}/2$.
Using Hölder inequality, we have 
\begin{align}
\nonumber
& \PE[\norm{ \locshiftparam{c}{h+1} - \paramlim }^6 ] \leq  (1-\gamma \mu/2) \PE[a^6]^{1/3} + 36 \step^3 h^2 (\strcvx + \lip) \PE[a^6]^{2/3} \PE[\norm{ \cvar{c}{} - \cvarlim{c} }^6]^{1/3}
+ 30 \sum_{k=1}^3 (2\step \sqoptvar)^{2 k} \PE[a^6]^{1-k/3} \eqsp.
\end{align}
Therefore, we get
\begin{align}
\label{eq:bound-interm-loc-iterates-higher-bound}
& \PE[\norm{ \locshiftparam{c}{h+1} - \paramlim }^6] \leq \left( (1-\gamma \mu/2)^{1/3} \PE[a^6]^{1/3} + 12 \step^3 h^2 (\strcvx + \lip)  \PE[\norm{ \cvar{c}{} - \cvarlim{c} }^6]^{1/3}+  40 \gamma^2\optvar  \right)^3 \eqsp.
\end{align}
Using $(1-\gamma \mu/2)^{1/3} \leq 1-\gamma \mu/6$ and a straightforward induction shows that 
\begin{align}
\nonumber
& \PE[ \norm{ \locshiftparam{c}{H} - \paramlim }^6 ]^{1/3}\leq (1-\gamma \mu/6)\norm{ \param - \paramlim }^{2} + 12 \step^3 H^3 (\strcvx + \lip)  \norm{ \cvar{c}{} - \cvarlim{c} }^2+  40 H \gamma^2\optvar  \eqsp.
\end{align}
Using $(1-\gamma \mu/6) \ge 1/2$ and $\gamma H (\mu+\lip)\leq 1/24 $ completes the proof.
\end{proof}

\begin{corollary}
\label{cor:crude-bounds-global-higher-order}
Assume \Cref{assum:strong-convexity}, \Cref{assum:smoothness} and \Cref{assum:smooth-var}.
Let $\step > 0$ be the step size and $\nlupdates > 0$ the number of local updates of S{\scriptsize CAFFOLD}.
Assume that $\step \lip \le 1 / 48$ and $\step \nlupdates (\lip + \strcvx) \le 1/24$.
Then, for all $h \in \iint{0}{\nlupdates}$, and $p \in \{1, 2, 3\}$, it holds that
\begin{gather}
\label{eq:cor-crude-bound-theta-higher-order}
  \Big( \int  \bnorm{ \param - \paramlim }^{2p} \statdist{\step, \nlupdates}(\rmd \param, \rmd \Cvarw)
  \Big)^{1/p}
  \le \frac{240 \step}{\strcvx} \optvar
    \eqsp,
\\
\label{eq:cor-crude-bound-theta-loc-sum-higher-order}
  \Big(
  \frac{1}{\nagent} \sum_{c=1}^\nagent \int  \PE\left[ \bnorm{ \locstoscafop{c}{h}{\param}{\cvar{c}{}}{\locRandState{c}{1:h}} - \paramlim }^{2p} \right] \statdist{\step, \nlupdates}(\rmd \param, \rmd \Cvarw)
\Big)^{1/p}
  \le \frac{240  \step}{\strcvx} \optvar
    \eqsp,
\\
\label{eq:cor-crude-bound-xi-sum-higher-order}
\Big(
  \frac{\step^2 \nlupdates^2}{\nagent} \sum_{c=1}^\nagent \int  \bnorm{ \cvar{c}{} - \cvarlim{c} }^{2p}
  \statdist{\step, \nlupdates}(\rmd \param, \rmd \Cvarw)
   \Big)^{1/p}
   \le \frac{240 \step}{\strcvx} \optvar
    \eqsp,
\end{gather}
where $\Cvarw = (\cvar{1}{}, \dots, \cvar{\nagent}{}) \in \rset^{\nagent \times d}$.
\end{corollary}

\begin{proof}
By \Cref{lem:descent-noise-powsix}, we can bound the $\Lambda$-norm of the $\nrounds$-th element of the process $(\bigX^t)_{t=0}^\infty$, as
\begin{align*}
\PE\left[ \norm{ \bigX^\nrounds - \bigXlim }[\Lambda]^6 \right]^{1/3}
\le
\left( 1 - \step \strcvx / 6\right)^{\nlupdates \nrounds}
\norm{ \bigX^0 - \bigXlim }[\Lambda]^2 
+ \sum_{t=0}^{\nrounds - 1} \left( 1 - \step \strcvx / 6\right)^{\nlupdates t} \cdot 40 \step^2 \nlupdates \optvar
\eqsp.
\end{align*}
Taking the limit as $\nrounds \rightarrow \infty$, we obtain
\begin{align*}
\lim_{\nrounds \rightarrow \infty} \PE\left[ \norm{ \bigX^\nrounds - \bigXlim }[\Lambda]^6 \right]^{1/3}
\le
\frac{ 240 \step \nlupdates }{\strcvx} \optvar
\eqsp.
\end{align*}
The result follows from derivations similar to the proof of \Cref{cor:crude-bounds-global} to bound the third moment of $\norm{ \bigX - \bigXlim }^2$ (\ie, the case $p=3$).
The result for $p = 1$ and $p = 2$ follows by Hölder's inequality.
\end{proof}

\begin{lemma}
\label{lem:crude-bound-local-and-cvar-higher-order}
Assume \Cref{assum:strong-convexity}, \Cref{assum:smoothness} and \Cref{assum:smooth-var}.
Let $\step > 0$ be the step size and $\nlupdates > 0$ the number of local updates of S{\scriptsize CAFFOLD}.
Assume that $\step \lip \le 1 / 48$, $\step \nlupdates (\lip + \strcvx) \le 1/24$, $\step \nlupdates^{1/2} \smoothcstvar^{1/2} \le 1 / 12$ and $\step \smoothcstvar \le \lip / 12$.
Then, for all $h \in \iint{0}{\nlupdates}$,
\begin{gather}
\label{eq:bound-local-iterate-h-higher-order}
\left( \int \norm{  \locstoscafop{c}{h}{\param}{\cvar{c}{}}{\locRandState{c}{1:h}} - \paramlim }^{6}  \statdist{\step, \nlupdates}(\rmd \theta, \rmd \Cvarw)
\right)^{1/3}
  \le 
\frac{600\step}{\strcvx} \optvar
\eqsp,
\\
\label{eq:bound-local-xi-higher-order}
\left( 
\int \norm{ \cvar{c}{} - \cvarlim{c}  }^{6} \statdist{\step, \nlupdates}(\rmd \theta, \rmd \Cvarw)
\right)^{1/3}
\le
\frac{3000 \lip}{H \strcvx} \optvar
\eqsp.
\end{gather}

\end{lemma}
\begin{proof}
The proof follows the same lines as \Cref{lem:crude-bound-local-and-cvar}. 

\textbf{Bound on local iterates.} Let $\param \in \mathbb{R}^d$ and $\{ \cvar{c}{}\}_{c=1}^\nagent \subset \mathbb{R}^d$ that satisfy the constraints
$\sum_{c=1}^\nagent \cvar{c}{} = 0$.
To bound the local iterates, we proceed as in~\eqref{eq:def-adjusted-operator-xi-bound-bis}, we define $\locparam{c}{h} = \locscafopabv{c}{h}{\param; \cvar{c}{}, \locRandState{c}{1:h}}$ and $\locshiftparam{c}{h} = \locparam{c}{h} + \step h (\cvar{c}{} - \cvarlim{c})$.
Similarly to \Cref{lem:crude-bound-local-and-cvar-appendix}, we use Jensen's inequality to bound
\begin{align}
\label{eq:decomp-powsix-thetaloc-fct-thetashift}
\PE^{1/3} \left[ \norm{ \locparam{c}{h} - \paramlim }^6 \right]
& \le
\PE^{1/3} \left[ \norm{ \locshiftparam{c}{h} - \paramlim }^6 \right]
+ \step^2\nlupdates^2
\norm{ \cvar{c}{} - \cvarlim{c} }^2
\eqsp.
\end{align}
Then, unrolling \eqref{eq:bound-interm-loc-iterates-higher-bound} for $h$ steps and using the fact that $\step \nlupdates (\lip + \strcvx) \le 1/24$, we obtain, for any $h \le \nlupdates$, and $c \in \iint{1}{\nagent}$,
\begin{align}
\label{eq:decomp-powsix-thetashift}
\PE^{1/3} \left[ \norm{ \locshiftparam{c}{h} - \paramlim }^6 \right]
& \le
\norm{ \param - \paramlim }^6
+ \frac{\step^2\nlupdates^2}{2} 
\norm{ \cvar{c}{} - \cvarlim{c} }^2
+ 40 \step^2 \nlupdates \optvar
\eqsp.
\end{align}
Plugging \eqref{eq:decomp-powsix-thetashift} in \eqref{eq:decomp-powsix-thetaloc-fct-thetashift}, we obtain
\begin{align}
\label{eq:upper-bound-intermediate-sum-theta-tilde-intermediate-h}
\PE^{1/3} \left[ \norm{ \locparam{c}{h} - \paramlim }^6 \right]
& \le
\norm{ \param - \paramlim }^2
+ \frac{3\step^2\nlupdates^2}{2} 
\norm{ \cvar{c}{} - \cvarlim{c} }^2
+ 40 \step^2 \nlupdates \optvar
\eqsp.
\end{align}
Taking the third power of this inequality, integrating it over the stationary distribution of \Scaffold's iterates and using \Cref{cor:crude-bounds-global-higher-order}, and using Jensen's inequality, we obtain
\begin{align}
\nonumber
\int 
\expe{ \norm{ \locparam{c}{h} \!-\! \paramlim }^6}
\statdist{\step, \nlupdates}(\rmd \param, \rmd \Cvarw)
& \le
\int
\Big( 
3^2 \expe{ \norm{ \param \!-\! \paramlim }^6}
\!+\! 8 \step^6 \nlupdates^6 \norm{ \cvar{c}{} \!-\! \cvarlim{c} }^6
+ 3^2 \cdot 40^3 \cdot \step^6 \nlupdates^3 \sqoptvar^6 \Big) \statdist{\step, \nlupdates}(\rmd \param, \rmd \Cvarw)
\\
\label{eq:ineq-one-higher-order-loc}
& \le
8\step^6 \nlupdates^6 \int
\norm{ \cvar{c}{} - \cvarlim{c} }^6
\statdist{\step, \nlupdates}(\rmd \param, \rmd \Cvarw)
+ \frac{3^2 \cdot (240^3 + 1) \cdot \step^3}{\strcvx^3} \sqoptvar^6
\eqsp,
\end{align}
where we used $\step \lip \le 1/48$ and $1/\lip \le 1/\strcvx$ to bound $40^3 \step^3 \le 1 / \strcvx^3$.

\textbf{Bound on control variates.}
To derive the second inequality, we start from~\eqref{eq:expression-xi-plus-fct-grads},
\begin{align*}
\cvar{c}{+} - \cvarlim{c}
& =
\frac{1}{\nagent \nlupdates} 
\sum_{i=1}^\nagent
\sum_{h=0}^{\nlupdates-1}
\left(
\left(\gnf{i}{\locparam{i}{h}} - \gnf{i}{\paramlim} \right)
- \left(\gnf{c}{\locparam{c}{h}} - \gnf{c}{\paramlim}\right)
+ \locnoiseabv{i}{h+1} - \locnoiseabv{c}{h+1}
\right)
\eqsp.
\end{align*}
Using Jensen's inequality, we obtain
\begin{align*}
\PE^{1/3}\left[ \norm{ \cvar{c}{+} - \cvarlim{c} }^6 \right]
& \le
2 \PE^{1/3}\left[ \bnorm{
\frac{1}{\nagent \nlupdates} 
\sum_{i=1}^\nagent
\sum_{h=0}^{\nlupdates-1}
\left(\gnf{i}{\locparam{i}{h}} - \gnf{i}{\paramlim} \right)
- \left(\gnf{c}{\locparam{c}{h}} - \gnf{c}{\paramlim}\right)
}^6 \right]
\\
& \quad
+
4 \PE^{1/3}\left[ \bnorm{
\frac{1}{\nagent \nlupdates} 
\sum_{i=1}^\nagent
\sum_{h=0}^{\nlupdates-1}
 \locnoiseabv{i}{h+1} 
}^6 \right]
+
4 \PE^{1/3}\left[ \bnorm{
\frac{1}{\nlupdates} 
\sum_{h=0}^{\nlupdates-1}
 \locnoiseabv{c}{h+1}
}^6 \right]
\eqsp.
\end{align*}
To control the last two terms, we note that they are reverse martingale differences w.r.t. the filtration $\mcF^h = \sigma( \locRandState{1:\nagent}{1:h} )$.
By Burkholder's inequality (see, \eg, \citet{oskekowski2012sharp}, Theorem 8.6) which holds due to \Cref{assum:smooth-var}, we have
\begin{align*}
\PE^{1/3}\left[ \bnorm{
\frac{1}{\nagent \nlupdates} 
\sum_{i=1}^\nagent
\sum_{h=0}^{\nlupdates-1}
 \locnoiseabv{i}{h+1} 
}^6 \right]
& \le
\frac{3^2}{\nagent^2 \nlupdates^2}
\PE^{1/3}\left[ 
\Big(
\sum_{i=1}^\nagent
\sum_{h=0}^{\nlupdates-1}
\norm{
 \locnoiseabv{i}{h+1} 
}^2 \Big)^{3} \right]
\le
\frac{3^2}{\nagent^2 \nlupdates^2}
\sum_{i=1}^\nagent
\sum_{h=0}^{\nlupdates-1}
\PE^{1/3}\left[ 
\norm{
 \locnoiseabv{i}{h+1} 
}^6  \right]
\eqsp.
\end{align*}
Using the smoothness of the error noise's moments (\Cref{assum:smooth-var}), we thus obtain
\begin{align*}
\PE^{1/3}\left[ \bnorm{
\frac{1}{\nagent \nlupdates} 
\sum_{i=1}^\nagent
\sum_{h=0}^{\nlupdates-1}
 \locnoiseabv{i}{h+1} 
}^6 \right]
& \le
\frac{3^2}{\nagent^2 \nlupdates^2}
\sum_{i=1}^\nagent
\sum_{h=0}^{\nlupdates-1}
 \smoothcstvar
\PE\left[ \norm{ \locparam{i}{h} - \paramlim }^6 \right]^{1/3}
+ \optvar 
\eqsp.
\end{align*}
Using Jensen's inequality again, and proceeding as in \Cref{lem:crude-bound-local-and-cvar}'s proof using Lipschitzness of the gradient \eqref{eq:grad-lipschitz}, we have
\begin{align}
\nonumber
\PE^{1/3}\left[ \norm{ \cvar{c}{+} - \cvarlim{c} }^6 \right]
& \le 
\frac{4 \lip^2}{\nagent \nlupdates} 
\sum_{i=1}^\nagent
\sum_{h=0}^{\nlupdates-1}
\PE^{1/3}\left[ 
\norm{ \locparam{i}{h} - \paramlim }^6 \right]
+ 
\PE^{1/3}\left[ \norm{ \locparam{c}{h} - \paramlim }^6 \right]
\\
\nonumber
& \quad
+
\frac{4 \cdot 3^2 }{\nagent^2 \nlupdates^2}  
\sum_{i=1}^\nagent
\sum_{h=0}^{\nlupdates-1}
\left\{ \smoothcstvar \PE^{1/3}\left[ 
\norm{ \locparam{i}{h} - \paramlim }^6 \right]
+ \optvar
\right\}
+
\frac{4 \cdot 3^2 }{\nlupdates^2}  
\sum_{h=0}^{\nlupdates-1}
\left\{ \smoothcstvar \PE^{1/3}\left[ 
\norm{ \locparam{c}{h} - \paramlim }^6 \right]
+ \optvar
\right\}
\\
\label{eq:expansion-xi-powsix-fct-loc-iteates}
& \le 
\frac{8 \cdot 3^2 }{\nlupdates} \optvar
+ \frac{4 \lip^2 + 4  \cdot 3^2 \smoothcstvar}{\nagent \nlupdates^2} 
\sum_{i=1}^\nagent
\sum_{h=0}^{\nlupdates-1}
\left\{ 
\PE^{1/3}\left[ 
\norm{ \locparam{i}{h} - \paramlim }^6 \right]
+ 
\PE^{1/3}\left[ \norm{ \locparam{c}{h} - \paramlim }^6 \right]
\right\}
\eqsp.
\end{align}
Plugging \eqref{eq:upper-bound-intermediate-sum-theta-tilde-intermediate-h} in~\eqref{eq:expansion-xi-powsix-fct-loc-iteates}, we obtain
\begin{align}
\nonumber
& \PE^{1/3}\left[ \norm{ \cvar{c}{+} - \cvarlim{c} }^6 \right]
\le 
\frac{72}{\nlupdates} \optvar
+ \frac{8 \lip^2 + 72 \smoothcstvar}{\nlupdates} 
\Big( \norm{ \param - \paramlim }^2
+ \frac{3\step^2\nlupdates^2}{2} 
\norm{ \cvar{c}{} - \cvarlim{c} }^2
+ 40 \step^2 \nlupdates \optvar
\Big)
\\
\nonumber
& \quad =
\frac{72}{\nlupdates} \optvar
+ (40 \cdot 8 \lip^2 + 40 \cdot 72 \smoothcstvar) \step^2 \optvar
+ \frac{8 \lip^2 + 72 \smoothcstvar}{\nlupdates} \norm{ \param - \paramlim }^2
+ (3 \cdot 4 \lip^2 + 3 \cdot 36 \smoothcstvar) 
\step^2\nlupdates
\norm{ \cvar{c}{} - \cvarlim{c} }^2
\\
\label{eq:bound-xi-powsix-interm-fct-theta-xi}
& \quad \le
\frac{72}{\nlupdates} \optvar
+ \frac{1 + 5}{\nlupdates} \optvar
+ \frac{8 \lip^2 + 72 \smoothcstvar}{\nlupdates} \norm{ \param - \paramlim }^2
+ \Big( \frac{1}{96} + \frac{3}{16} \Big)
\norm{ \cvar{c}{} - \cvarlim{c} }^2
\eqsp,
\end{align}
where we used $\step \lip \le 1/48$, $\step (\lip + \strcvx) \nlupdates \le 1/24$ and $\step \nlupdates^{1/2} \smoothcstvar^{1/2} \le 1/12$.
Remark that $1/96 + 3/16 \le 1/5$.
Taking the third power of \eqref{eq:bound-xi-powsix-interm-fct-theta-xi} and using Jensen's inequality, we obtain
\begin{align}
\nonumber
& \PE\left[ \norm{ \cvar{c}{+} - \cvarlim{c} }^6 \right]
\le
\frac{3^2 \cdot 78^3}{\nlupdates^3} \sqoptvar^6
+ 3^2  \cdot \left( \frac{8 \lip^2 + 72 \smoothcstvar}{\nlupdates} \right)^3 \norm{ \param - \paramlim }^6
+ \frac{3^2}{5^3}
\norm{ \cvar{c}{} - \cvarlim{c} }^6
\eqsp.
\end{align}
Integrating over $\statdist{\step,\nlupdates}$, remarking that $\int \PE\left[ \norm{ \cvar{c}{+} - \cvarlim{c} }^6 \right] \statdist{\step,\nlupdates}(\rmd \param, \rmd \Cvarw) = \int \norm{ \cvar{c}{} - \cvarlim{c} }^6 \statdist{\step,\nlupdates}(\rmd \param, \rmd \Cvarw)$, using the fact that $3^2 / 5^3 \le 1/10$,  and multiplying the resulting inequality by $10/9$, we obtain
\begin{align}
\nonumber
& \int \norm{ \cvar{c}{} - \cvarlim{c} }^6  \statdist{\step, \nlupdates}(\rmd \param, \rmd \Cvarw)
\le
\frac{10\cdot 78^3}{\nlupdates^3} \sqoptvar^6
+ 10 \Big( \frac{8 \lip^2 + 72 \smoothcstvar}{\nlupdates} \Big)^3
\int \norm{ \param - \paramlim }^6
\statdist{\step, \nlupdates}(\rmd \param, \rmd \Cvarw)
\\
\label{eq:ineq-one-higher-order-loc-xi}
& \quad \le
\frac{10 \cdot 78^3}{\nlupdates^3} \sqoptvar^6
+ 10  \frac{(8 \lip^2 + 72 \smoothcstvar)^3}{\nlupdates^3} 
\cdot
\frac{240^3 \step^3}{\strcvx^3} \sqoptvar^6
\le
\frac{10 \cdot 78^3}{\nlupdates^3} \sqoptvar^6
+ 90  \frac{8^3 \lip^6 + 72^3 \smoothcstvar^3 }{\nlupdates^3} 
\cdot
\frac{240^3 \step^3}{\strcvx^3} \sqoptvar^6
\le
\frac{3000^3 \lip^3}{\strcvx^3 \nlupdates^3} \sqoptvar^6
\eqsp,
\end{align}
where the last inequality follows from $\step \lip \le 1/48$, $\step \smoothcstvar^{1/2} \nlupdates^{1/2} \le 1/12$ and $\step \smoothcstvar \le \lip / 12$.

\textbf{Final bound on the local itrerates.}
From \eqref{eq:ineq-one-higher-order-loc} and \eqref{eq:ineq-one-higher-order-loc-xi}, we have
\begin{align*}
\int 
\expe{ \norm{ \locparam{c}{h} \!-\! \paramlim }^6}
\statdist{\step, \nlupdates}(\rmd \param, \rmd \Cvarw)
& \le
8\step^6 \nlupdates^6 \int
\norm{ \cvar{c}{} - \cvarlim{c} }^6
\statdist{\step, \nlupdates}(\rmd \param, \rmd \Cvarw)
+ \frac{3^2 \cdot (240^3 + 1) \cdot \step^3}{\strcvx^3} \sqoptvar^6
\\
& \le
\frac{8 \cdot 3000^3 \step^6 \nlupdates^3 \lip^3}{\strcvx^3} \sqoptvar^6
+ \frac{3^2 \cdot (240^3 + 1) \cdot \step^3}{\strcvx^3} \sqoptvar^6
\eqsp,
\end{align*}
and the result follows from $\step \nlupdates \lip \le 1/24$, which ensures that $8 \cdot 3000^3 \step^3 \nlupdates^3 \lip^3 + 3^2 \cdot (240^3 + 1) \le 600^3$.
\end{proof}

\section{Bounding the Variance of \Scaffold}
\label{sec:variance-scaffold}
We now study the bias of the \Scaffold algorithm.
Let $X = (\param, \cvar{1}{}, \dots, \cvar{\nagent}{})$, where the global parameter and control variates are $\param, \cvar{1}{}, \dots, \cvar{\nagent}{}$ is a vector in $\rset^{(\nagent+1) d}$ drawn from the stationary distribution $\statdist{\step,\nlupdates}$.
To study its expected value, we use the fact that, by definition, the  $\opscaffold(X; Z)$ has the same distribution as $X$.

\paragraph{Notations.}
For $\param \in \rset^d$ and $\Cvarw = (\cvar{1}{}, \dots, \cvar{\nagent}{}) \in \rset^{\nagent \times d}$, we define the variances and covariances of parameters and control variates in the stationary distribution $\statdist{\step,\nlupdates}$ as
\begin{align*}
\covparam
& \eqdef{}
\int \left( \param - \paramlim \right)^{\otimes 2}  \statdist{\step, \nlupdates}( \rmd \param, \rmd \Cvarw )
\eqsp,
\\
\covcvar{c,c'}
& \eqdef{}
\int \left( \cvar{c}{} - \cvarlim{c} \right)
\left( \cvar{c'}{} - \cvarlim{c'} \right)^\top \statdist{\step, \nlupdates}( \rmd \param, \rmd \Cvarw )
\eqsp,
\\
\covparamcvar{c}
& \eqdef{}
\int 
\Big( \param - \paramlim \Big)
\left( \cvar{c}{} - \cvarlim{c} \right)^\top 
\statdist{\step, \nlupdates}( \rmd \param, \rmd \Cvarw )
\eqsp,
\\
\covcvarparam{c}
& \eqdef{}
\int 
\left( \cvar{c}{} - \cvarlim{c} \right)
\Big( \param - \paramlim \Big)^\top
\statdist{\step, \nlupdates}( \rmd \param, \rmd \Cvarw )
\eqsp.
\end{align*}
In the following, we use the following matrices and tensor, that appear in the integral remainders of our expansions
\begin{align}
\label{eq:def-mat-d2}
\avghnf{c}{h}(\theta) 
& =
\int_0^1 \hnf{c}{\paramlim + t\left( \locstoscafop{c}{h}{\param}{\cvar{c}{}}{\locRandState{c}{1:h}} - \paramlim\right)} 
\rmd t
\eqsp,
\\
\label{eq:def-mat-d3}
\avghhnf{c}{h}(\theta)
& =
\int_0^1 {(1-t)} \hhnf{c}{\paramlim + t\left( \locstoscafop{c}{h}{\param}{\cvar{c}{}}{\locRandState{c}{1:h}} - \paramlim\right)} 
\rmd t
\eqsp.
\end{align}
{For conciseness, we will often use the abbreviated notations} 
\begin{equation}
\label{eq:abbreviated-avghnf}
\avghnf{c}{h}:= \avghnf{c}{h}(\locparam{c}{h}) \quad \text{and} \quad \avghhnf{c}{h}:= \avghhnf{c}{h}(\locparam{c}{h}) \eqsp.
\end{equation}

Following an update step of the \textsc{Scaffold} algorithm, we obtain their updated counterparts, which reflect the adjustments made during this iteration.
\begin{align}
\label{eq:def-param-general}
\globparam{+}
= \globstoscafop{\nlupdates}{\param}{\cvar{1:\nagent}{}}{\locRandState{1:\nagent}{1:\nlupdates}}
\eqsp,
\qquad
\locparam{c}{h}
= \locstoscafop{c}{h}{\param}{\cvar{c}{}}{\locRandState{c}{1:h}}
\eqsp,
\qquad
\cvar{c}{+}
= \scafopcv{c}{\nlupdates}{\cvar{c}{}}{\param}{\locRandState{c}{1:\nlupdates}}
\eqsp,
\end{align}
for $h \in \iint{0}{\nlupdates}$ and $c \in \iint{1}{\nagent}$. We define the noise accumulated in one round,
with $\locnoiseabv{c}{h}$ as defined in \eqref{eq:shorthand-epsilon}.
\begin{align*}
{\locnoiseabv{c}{1:\nlupdates}}
= 
\sum_{h=1}^{\nlupdates} \loccontractw{c}^{\nlupdates - h} \locnoiseabv{c}{h}
\eqsp.
\end{align*}

\paragraph{Matrix notations.}
We define the contraction matrix $\loccontractw{c} = \Id - \step \hnf{c}{\paramlim}$,
as well as its powers, for $h \in \iint{0}{\nlupdates}$, average, and scaled difference between the local matrices and their average,
\begin{align}
\loccontractw{c}^h = \left( \Id - \step \hnf{c}{\paramlim} \right)^{h}
\eqsp,
\quad
\globcontractw = \frac{1}{\nagent} \sum_{c=1}^\nagent \loccontractw{c}^\nlupdates
\eqsp,
\quad
\diffcontractc{c}
=
\frac{1}{\step \nlupdates} \Big( \loccontractw{c} - \globcontractw \Big)
\eqsp.
\end{align}
Finally, we define
\begin{align}
\label{eq:def-expansion-matrices}
\locmat{c}{1:\nlupdates} 
& = - \frac{1}{\nlupdates} \sum_{h=0}^{\nlupdates-1} \loccontractw{c}^{\nlupdates - h - 1}
\eqsp,
\quad%
\shiftedlocmat{c}{1:\nlupdates} 
= \Id - \frac{1}{\nlupdates} \sum_{h=0}^{\nlupdates-1} \loccontractw{c}^{\nlupdates - h - 1}
\eqsp,
\quad
\locreste{c}{1:\nlupdates}
 =
\sum_{h=0}^{\nlupdates-1} \loccontractw{c}^{\nlupdates - h - 1} \avghhnf{c}{h} \left( \locparam{c}{h} - \paramlim \right)^{\otimes 2}
\eqsp.
\end{align}

\subsection{Expansions of local updates and control variates}

First, we give explicit expansions of the local and global parameter updates.
\begin{lemma}
\label{lem:expansion-loc-glob-param}
Let $\param \in \rset^d$ and $\Cvarw = (\cvar{1}{}, \dots, \cvar{\nagent}{}) \in \rset^{\nagent \times d}$.
After one global update of S{\scriptsize CAFFOLD}, we obtain a global parameter $\globparam{+}$, $\nagent$ control variates $\cvar{c}{+}$ and $\nagent \cdot \nlupdates$ local iterates $\locparam{c}{h}$ as defined in \eqref{eq:def-param-general}.
These updates parameters can be expressed as
\begin{align}
\label{eq:expansion-loc-update}
\locparam{c}{\nlupdates} - \paramlim
& =\loccontractw{c}^{\nlupdates} \left( \param - \paramlim \right)
+ \step \nlupdates \locmat{c}{1:\nlupdates} \left( \cvar{c}{} - \cvarlim{c} \right)
- \step \locreste{c}{1:\nlupdates}
- \step \locnoiseabv{c}{1:\nlupdates}
\eqsp,
\\
\label{eq:expansion-glob-update}
\globparam{+} - \paramlim
& =
\globcontractw \left( \param - \paramlim \right)
+ \frac{\step \nlupdates}{\nagent} \sum_{c=1}^\nagent \shiftedlocmat{c}{1:\nlupdates} \left( \cvar{c}{} - \cvarlim{c} \right)
- \frac{\step}{\nagent} \sum_{c=1}^\nagent \locreste{c}{1:\nlupdates}
- \frac{\step}{\nagent} \sum_{c=1}^\nagent \locnoiseabv{c}{1:\nlupdates}
\eqsp.
\end{align}
\end{lemma}
\begin{proof}
Let $c \in \iint{1}{\nagent}$ and $h \in \iint{0}{\nlupdates-1}$.
Expanding the gradient at step $h$ gives
\begin{align}
\nonumber
\locparam{c}{h+1}
& =
\locparam{c}{h}
- \step \left( 
\gnf{c}{\locparam{c}{h}} + \cvar{c}{} + \locnoiseabv{c}{h+1}
\right)
\\
\label{eq:expansion-grad-hc-second}
& =
\locparam{c}{h}
- \step \left( 
\gnf{c}{\paramlim}
+ \hnf{c}{\paramlim} \left( \locparam{c}{h} - \paramlim \right)
+ \avghhnf{c}{h+1} \left( \locparam{c}{h} - \paramlim \right)^{\otimes 2}
+ \cvar{c}{} + \locnoiseabv{c}{h+1}
\right)
\eqsp.
\end{align}
Since $\cvarlim{c} = - \gnf{c}{\paramlim}$, we obtain
\begin{align*}
\locparam{c}{h+1} - \paramlim
& =
\locparam{c}{h} - \paramlim
- \step \hnf{c}{\paramlim} \left( \locparam{c}{h} - \paramlim \right)
- \step \left( \cvar{c}{} - \cvarlim{c} \right)
- \step \avghhnf{c}{h+1} \left( \locparam{c}{h} - \paramlim \right)^{\otimes 2}
- \step \locnoiseabv{c}{h+1}
\\
& =
\underbrace{\Big( \Id - \step \hnf{c}{\paramlim} \Big)}_{\loccontractw{c}} \left( \locparam{c}{h} - \paramlim \right)
- \step \left( \cvar{c}{} - \cvarlim{c} \right)
- \step \avghhnf{c}{h} \left( \locparam{c}{h} - \paramlim \right)^{\otimes 2}
- \step \locnoiseabv{c}{h+1}
\eqsp.
\end{align*}
We obtain the following expression for the local updates
\begin{align*}
\locparam{c}{\nlupdates} - \paramlim
& =
\loccontractw{c}^{\nlupdates} \left( \param - \paramlim \right)
- \step \sum_{h=0}^{\nlupdates-1} \loccontractw{c}^{\nlupdates - h - 1} \left( \cvar{c}{} - \cvarlim{c} \right)
- \step \sum_{h=0}^{\nlupdates-1} \loccontractw{c}^{\nlupdates - h - 1} \avghhnf{c}{h} \left( \locparam{c}{h} - \paramlim \right)^{\otimes 2}
- \step \sum_{h=0}^{\nlupdates-1} \loccontractw{c}^{\nlupdates - h - 1} \locnoiseabv{c}{h+1}
\\
& =
\loccontractw{c}^{\nlupdates} \left( \param - \paramlim \right)
+ \step \nlupdates \locmat{c}{1:\nlupdates} \left( \cvar{c}{} - \cvarlim{c} \right)
- \step \locreste{c}{1:\nlupdates}
- \step \locnoiseabv{c}{1:\nlupdates}
\eqsp,
\end{align*}
which gives the first identity~\eqref{eq:expansion-loc-update}.
The second identity~\eqref{eq:expansion-glob-update} follows from averaging the first one over all clients and using the fact that 
\begin{equation}
\frac{1}{\nagent}\sum_{c=1}^\nagent \locmat{c}{1:\nlupdates} \left( \cvar{c}{} - \cvarlim{c} \right) = 
\frac{1}{\nagent} \sum_{c=1}^\nagent \shiftedlocmat{c}{1:\nlupdates} \left( \cvar{c}{} - \cvarlim{c} \right)
\eqsp,
\end{equation}
which follows from $\sum_{c=1}^\nagent \cvar{c}{} - \cvarlim{c} = 0$.
\end{proof}

Based on \Cref{lem:expansion-loc-glob-param}, we can give an expression for the control variate updates.
\begin{lemma}
\label{lem:expansion-control-var}
Let $\param \in \rset^d$ and $\Cvarw = (\cvar{1}{}, \dots, \cvar{\nagent}{}) \in \rset^{\nagent \times d}$.
After one global update of Scaffold, we obtain a global parameter $\globparam{+}$, $\nagent$ control variates $\cvar{c}{+}$ and $\nagent \cdot \nlupdates$ local iterates $\locparam{c}{h}$ as defined in \eqref{eq:def-param-general}.
The updated control variates can be expressed as
\begin{equation}
\label{eq:expansion-control-var}
\begin{aligned}
\cvar{c}{+}
- 
\cvarlim{c}{} 
& =
\diffcontractc{c} \left( \param - \paramlim \right)
+ {\shiftedlocmat{c}{1:\nlupdates}} \left( \cvar{c}{} - \cvarlim{c} \right)
- \frac{1}{\nagent} \sum_{i=1}^\nagent
\shiftedlocmat{i}{1:\nlupdates} \left( \cvar{i}{} - \cvarlim{i} \right)
\\
& \quad 
- \frac{1}{\nlupdates} \locreste{c}{1:\nlupdates}
+ \frac{1}{\nagent \nlupdates} \sum_{i=1}^\nagent \locreste{i}{1:\nlupdates}
- \frac{1}{\nlupdates} \locnoiseabv{c}{1:\nlupdates}
+ \frac{1}{\nagent \nlupdates} \sum_{i=1}^\nagent \locnoiseabv{i}{1:\nlupdates}
\eqsp.
\end{aligned}
\end{equation}
where $\locmat{c}{1:\nlupdates}$, $\shiftedlocmat{c}{1:\nlupdates}$, $\locreste{c}{1:\nlupdates}$, and $\locnoiseabv{c}{1:\nlupdates}$ are defined in \eqref{eq:def-expansion-matrices}.
\end{lemma}
\begin{proof}
Let $c \in \iint{1}{\nagent}$, $\cvar{c}{}$ is updated as $\cvar{c}{+} = \cvar{c}{} + \frac{1}{\step\nlupdates} \left( \locparam{c}{\nlupdates} - \globparam{+} \right)$, which gives
\begin{align*}
\cvar{c}{+}
& =
\cvar{c}{} 
+ \frac{1}{\step\nlupdates} \left( \loccontractw{c}^{\nlupdates} - \globcontractw \right) \left( \param - \paramlim \right)
+ \locmat{c}{1:\nlupdates} \left( \cvar{c}{} - \cvarlim{c} \right)
+ \frac{1}{\nagent} \sum_{i=1}^\nagent \locmat{i}{1:\nlupdates} \left( \cvar{i}{} - \cvarlim{i} \right)
\\
& \quad 
- \frac{1}{\nlupdates} \locreste{c}{1:\nlupdates}
+ \frac{1}{\nagent \nlupdates} \sum_{i=1}^\nagent \locreste{i}{1:\nlupdates}
- \frac{1}{\nlupdates} \locnoiseabv{c}{1:\nlupdates}
+ \frac{1}{\nagent \nlupdates} \sum_{i=1}^\nagent \locnoiseabv{i}{1:\nlupdates}
\\
& =
\cvarlim{c}
+ \cvar{c}{} - \cvarlim{c} 
+ \diffcontractc{c} \left( \param - \paramlim \right)
+ \locmat{c}{1:\nlupdates} \left( \cvar{c}{} - \cvarlim{c} \right)
+ \frac{1}{\nagent} \sum_{i=1}^\nagent \locmat{i}{1:\nlupdates} \left( \cvar{i}{} - \cvarlim{i} \right)
\\
& \quad 
- \frac{1}{\nlupdates} \locreste{c}{1:\nlupdates}
+ \frac{1}{\nagent \nlupdates} \sum_{i=1}^\nagent \locreste{i}{1:\nlupdates}
- \frac{1}{\nlupdates} \locnoiseabv{c}{1:\nlupdates}
+ \frac{1}{\nagent \nlupdates} \sum_{i=1}^\nagent \locnoiseabv{i}{1:\nlupdates}
\eqsp.
\end{align*}
Then, remark that $\cvar{c}{} - \cvarlim{c} + \locmat{c}{1:\nlupdates} (\cvar{c}{} - \cvarlim{c}) = \shiftedlocmat{c}{1:\nlupdates} (\cvar{c}{} - \cvarlim{c})$ since $\shiftedlocmat{c}{1:\nlupdates}  = \Id + \locmat{c}{1:\nlupdates} $.

\end{proof}

\subsection{Covariance of the Parameters and Control Variates}

\subsubsection{Recursion on covariance matrices}
    
\begin{lemma}
\label{lem:expansion-squared-theta-plus}
Assume \Cref{assum:strong-convexity}, \Cref{assum:smoothness} and \Cref{assum:smooth-var}.
Assume the step size $\step$ and the number of local updates $\nlupdates$ satisfy $\step \nlupdates (\lip + \strcvx) \leq 1$. 
Then, it holds that
\begin{align*}
\covparam
=
\globcontractw \covparam \globcontractw
+ \frac{\step\nlupdates}{\nagent} \sum_{c=1}^\nagent
\left( 
\globcontractw\covparamcvar{c}  \shiftedlocmat{c}{1:\nlupdates} 
+  \shiftedlocmat{c}{1:\nlupdates} \covcvarparam{c}\globcontractw \right)
+ \frac{\step^2\nlupdates^2}{\nagent^2}
\sum_{c=1}^\nagent \sum_{c'=1}^\nagent
 \shiftedlocmat{c}{1:\nlupdates} \covcvar{c,c'}  \shiftedlocmat{c'}{1:\nlupdates} 
+ \frac{\step^2}{\nagent} \covonestep
+ \mathrm{R}^{\param}
\eqsp,
\end{align*}
where $\covonestep = \frac{1}{\nagent} 
\sum_{c=1}^\nagent \PE\left[ (\locnoiseabv{c}{1:\nlupdates}
)^{\otimes 2} \right]$,
and 
    $\mathrm{R}^{\param} = \mathrm{R}_1^{\param} + \mathrm{R}_1^{\param}{}^\top + \mathrm{R}_2^{\param} + \mathrm{R}_2^{\param}{}^\top + \mathrm{R}_3^{\param}$, %
with
\begin{align*}
\mathrm{R}_1^{\param} 
& =
\frac{\step^2}{\nagent^2} \sum_{c=1}^\nagent 
\int 
\PE \left[ \Big( \locnoiseabv{c}{1:\nlupdates} \Big)
\left( \locreste{c}{1:\nlupdates} \right)^\top \right]
 \statdist{\step, \nlupdates}( \rmd \param, \rmd \Cvarw )
\eqsp,
\\
\nonumber
\mathrm{R}_2^{\param} & =
- \frac{\step}{\nagent} \sum_{c=1}^\nagent
\int {\PE}\left[\locreste{c}{1:\nlupdates} \right] \left( \param - \paramlim \right)^\top \globcontractw  \statdist{\step, \nlupdates}( \rmd \param, \rmd \Cvarw )
\\
\nonumber
& \quad 
- \frac{\step^2 \nlupdates}{\nagent^2} 
\sum_{c=1}^\nagent
\sum_{c'=1}^\nagent
\int 
{\PE}\left[\locreste{c}{1:\nlupdates} \right] \left( \cvar{c'}{} - \cvarlim{c'} \right)^\top \shiftedlocmat{c'}{1:\nlupdates} 
 \statdist{\step, \nlupdates}( \rmd \param, \rmd \Cvarw )
\eqsp,
\\
\mathrm{R}_3^{\param} & =
\frac{\step^2}{\nagent^2} \sum_{c=1}^\nagent \sum_{c'=1}^\nagent 
\int
{\PE}\left[\Big(  \locreste{c}{1:\nlupdates} \Big)
\Big(  \locreste{c'}{1:\nlupdates} \Big)^\top \right]
 \statdist{\step, \nlupdates}( \rmd \param, \rmd \Cvarw )
\eqsp.
\end{align*}
\end{lemma}
\begin{proof}
Using the results from~\Cref{lem:expansion-loc-glob-param}, we have
\begin{align*}
\left( \globparam{+} - \paramlim \right)^{\otimes 2}
& =
\left(
\globcontractw \left( \param - \paramlim \right)
+ \frac{\step\nlupdates}{\nagent} \sum_{c=1}^\nagent \shiftedlocmat{c}{1:\nlupdates} \left( \cvar{c}{} - \cvarlim{c} \right)
- \frac{\step}{\nagent} \sum_{c=1}^\nagent \locreste{c}{1:\nlupdates}
\right)^{\otimes 2}
+
\frac{\step^2}{\nagent^2} 
\left(\sum_{c=1}^\nagent \locnoiseabv{c}{1:\nlupdates}
\right)^{\otimes 2}
\\
& \quad -
\frac{\step}{\nagent} \sum_{c=1}^\nagent 
 \locnoiseabv{c}{1:\nlupdates}
\left(
\globcontractw \left( \param - \paramlim \right)
+ \frac{\step\nlupdates}{\nagent} \sum_{c'=1}^\nagent \shiftedlocmat{c'}{1:\nlupdates} \left( \cvar{c'}{} - \cvarlim{c'} \right)
- \frac{\step}{\nagent} \sum_{c'=1}^\nagent \locreste{c'}{1:\nlupdates}
\right)^\top
\\
& \quad - 
\frac{\step}{\nagent} \sum_{c=1}^\nagent 
\left(
\globcontractw \left( \param - \paramlim \right)
+ \frac{\step \nlupdates}{\nagent} \sum_{c'=1}^\nagent \shiftedlocmat{c'}{1:\nlupdates} \left( \cvar{c'}{} - \cvarlim{c'} \right)
- \frac{\step}{\nagent} \sum_{c'=1}^\nagent \locreste{c'}{1:\nlupdates}
\right)^{\otimes 2}
\left( \locnoiseabv{c}{1:\nlupdates} \right)^\top
\eqsp.
\end{align*}
Taking the expectation, and using the fact that the $\locRandStatew{c}$ are independent from one client to another, we obtain
\begin{align*}
\PE\left[ \left( \globparam{+} - \paramlim \right)^{\otimes 2} \right]
& =
\PE\left[
\left(
\globcontractw \left( \param - \paramlim \right)
+ \frac{\step \nlupdates}{\nagent} \sum_{c=1}^\nagent \shiftedlocmat{c}{1:\nlupdates} \left( \cvar{c}{} - \cvarlim{c} \right)
- \frac{\step}{\nagent} \sum_{c=1}^\nagent \locreste{c}{1:\nlupdates}
\right)^{\otimes 2}
\right]
\\
& \quad
+
\frac{\step^2}{\nagent^2} 
\sum_{c=1}^\nagent 
\PE\left[ \left( \locnoiseabv{c}{1:\nlupdates} \right)^{\otimes 2} \right]
+
\frac{\step^2}{\nagent^2} \sum_{c=1}^\nagent 
\PE\left[ 
\locnoiseabv{c}{1:\nlupdates}
\left( \locreste{c}{1:\nlupdates} \right)^\top
+
\locreste{c}{1:\nlupdates}
\left( \locnoiseabv{c}{1:\nlupdates} \right)^\top
\right]
\eqsp.
\end{align*}
The first term can be expressed using the identity
\begin{align*}
&
\PE\left[ \left(
\globcontractw \left( \param - \paramlim \right)
+ \frac{\step \nlupdates}{\nagent} \sum_{c=1}^\nagent \shiftedlocmat{c}{1:\nlupdates} \left( \cvar{c}{} - \cvarlim{c} \right)
- \frac{\step}{\nagent} \sum_{c=1}^\nagent \locreste{c}{1:\nlupdates}
\right)^{\otimes 2} \right]
\\
& =
\left(
\globcontractw \left( \param - \paramlim \right)
+ \frac{\step \nlupdates}{\nagent} \sum_{c=1}^\nagent \shiftedlocmat{c}{1:\nlupdates} \left( \cvar{c}{} - \cvarlim{c} \right)
\right)^{\otimes 2}
+ \frac{\step^2}{\nagent^2} \PE\left[ \Big( \sum_{c=1}^\nagent \locreste{c}{1:\nlupdates} \Big)^{\otimes 2} \right]
\\
& \quad - \frac{\step}{\nagent} \sum_{c=1}^\nagent \PE\left[ \locreste{c}{1:\nlupdates} \right] \left(
\globcontractw \left( \param - \paramlim \right)
+ \frac{\step\nlupdates}{\nagent} \sum_{c'=1}^\nagent \shiftedlocmat{c'}{1:\nlupdates} \left( \cvar{c'}{} - \cvarlim{c'} \right)
\right)^\top
\\
& \quad - \frac{\step}{\nagent} \sum_{c=1}^\nagent
\left(
\globcontractw \left( \param - \paramlim \right)
+ \frac{\step\nlupdates}{\nagent} \sum_{c'=1}^\nagent \shiftedlocmat{c'}{1:\nlupdates} \left( \cvar{c'}{} - \cvarlim{c'} \right)
\right) \PE\left[ \left( \locreste{c'}{1:\nlupdates} \right)^\top \right]
\eqsp.
\end{align*}
The first term can be expanded as
\begin{align*}
& \left(
\globcontractw \left( \param - \paramlim \right)
+ \frac{\step\nlupdates}{\nagent} \sum_{c=1}^\nagent \shiftedlocmat{c}{1:\nlupdates} \left( \cvar{c}{} - \cvarlim{c} \right)
\right)^{\otimes 2}
\\
& \quad =
\globcontractw \left( \param - \paramlim \right)^{\otimes 2} \globcontractw
+ \frac{\step^2 \nlupdates^2}{\nagent^2 }
\sum_{c=1}^\nagent \sum_{c'=1}^\nagent
\shiftedlocmat{c}{1:\nlupdates} 
\left( \cvar{c}{} - \cvarlim{c} \right)
\left( \cvar{c'}{} - \cvarlim{c'} \right)^\top
\shiftedlocmat{c'}{1:\nlupdates} 
\\
& \qquad 
+ \frac{\step\nlupdates}{\nagent} \sum_{c=1}^\nagent
\left\{
\globcontractw \left( \param - \paramlim \right)
( \cvar{c}{} - \cvarlim{c} )^\top
\shiftedlocmat{c}{1:\nlupdates} 
+
\shiftedlocmat{c}{1:\nlupdates} 
( \cvar{c}{} - \cvarlim{c} )
\left( \param - \paramlim \right) \globcontractw
\right\}
\eqsp,
\end{align*}
and the lemma follows by integrating over the stationary distribution of \Scaffold.
\end{proof}

\begin{lemma}
\label{lem:expansion-squared-theta-cvar-plus}
Assume \Cref{assum:strong-convexity}, \Cref{assum:smoothness} and \Cref{assum:smooth-var}.
Assume the step size $\step$ and the number of local updates $\nlupdates$ satisfy $\step \nlupdates (\lip + \strcvx) \leq 1$. 
Then, it holds that
\begin{align*}
\covparamcvar{c}
& =
 \globcontractw 
\covparam
\diffcontractc{c}
+
\globcontractw
\covparamcvar{c}
 \shiftedlocmat{c}{1:\nlupdates}
- \frac{1}{\nagent} \sum_{i'=1}^\nagent
\globcontractw \covparamcvar{i'} \shiftedlocmat{i'}{1:\nlupdates}
+ \frac{\step \nlupdates}{\nagent} \sum_{i=1}^\nagent
\shiftedlocmat{i}{1:\nlupdates} \covcvarparam{i}
\diffcontractc{c}
\\
& \quad
+ \frac{\step\nlupdates}{\nagent} \sum_{i=1}^\nagent \shiftedlocmat{i}{1:\nlupdates} \covcvar{i,c} \shiftedlocmat{c}{1:\nlupdates}
- \frac{\step \nlupdates}{\nagent^2} \sum_{i=1}^\nagent \sum_{i'=1}^\nagent
\shiftedlocmat{i}{1:\nlupdates} \covcvar{i,i'} \shiftedlocmat{i'}{1:\nlupdates}
+ \frac{\step}{\nagent \nlupdates}
\left( \loccovonestep{c} -  \covonestep  \right)
+ \mathrm{R}^{\param, \cvarw}_{(c)}
\eqsp,
\end{align*}
where $\covonestep = \frac{1}{\nagent} 
\sum_{c=1}^\nagent \PE\left[ \left(\locnoiseabv{c}{1:\nlupdates}
\right)^{\otimes 2} \right]$,
and $\mathrm{R}_{(c)}^{\param, \cvarw} = \mathrm{R}_{(c),1}^{\param, \cvarw} 
+\mathrm{R}_{(c),2}^{\param, \cvarw}
+\mathrm{R}_{(c),3}^{\param, \cvarw}
+ \mathrm{R}_{(c),4}^{\param, \cvarw} 
+\mathrm{R}_{(c),5}^{\param, \cvarw}$,
with
\begin{align*}
\mathrm{R}_{(c),1}^{\param, \cvarw}
& =
\frac{\step}{\nagent} \int \sum_{i=1}^\nagent 
\PE \left[ \locnoiseabv{i}{1:\nlupdates} \Big( 
\frac{1}{\nlupdates} \locreste{c}{1:\nlupdates}
- \frac{1}{\nagent \nlupdates} \sum_{i'=1}^\nagent \locreste{i'}{1:\nlupdates}
\Big)^\top
+ \locreste{i}{1:\nlupdates} \Big( \frac{1}{\nlupdates} \locnoiseabv{c}{1:\nlupdates}
- \frac{1}{\nagent \nlupdates} \sum_{i'=1}^\nagent \locnoiseabv{i'}{1:\nlupdates}
\Big)^\top \right]
\statdist{\step, \nlupdates}( \rmd \param, \rmd \Cvarw )
\eqsp,
\\
\mathrm{R}_{(c),2}^{\param, \cvarw}
& =
- \frac{\step}{\nagent} \sum_{i=1}^\nagent
\int \PE\left[ \locreste{i}{1:\nlupdates} \right]
\left(  \diffcontractc{c} \left( \param - \paramlim \right)
\right)^\top
\statdist{\step, \nlupdates}( \rmd \param, \rmd \Cvarw )
\eqsp,
\\
\mathrm{R}_{(c),3}^{\param, \cvarw}
& =
- \frac{\step}{\nagent} \sum_{i=1}^\nagent
\int \PE\left[ \locreste{i}{1:\nlupdates} \right]
\left(  
\shiftedlocmat{c}{1:\nlupdates} \left( \cvar{c}{} - \cvarlim{c} \right)
-
\frac{1}{\nagent}
 \sum_{i'=1}^\nagent
\shiftedlocmat{i'}{1:\nlupdates} \left( \cvar{i'}{} - \cvarlim{i'} \right)
\right)^\top
\statdist{\step, \nlupdates}( \rmd \param, \rmd \Cvarw )
\eqsp,
\\
\mathrm{R}_{(c),4}^{\param, \cvarw}
& =
\int 
\PE\left[ \Big( 
\globcontractw \left( \param - \paramlim \right)
+ \frac{\step \nlupdates}{\nagent} \sum_{i=1}^\nagent \shiftedlocmat{i}{1:\nlupdates} \left( \cvar{i}{} - \cvarlim{i} \right)
\Big)
\Big(
- \frac{1}{\nlupdates} \locreste{c}{1:\nlupdates}
+ \frac{1}{\nagent \nlupdates} \sum_{i'=1}^\nagent \locreste{i'}{1:\nlupdates}
\Big)^\top \right]
\statdist{\step, \nlupdates}( \rmd \param, \rmd \Cvarw )
\eqsp,
\\
\mathrm{R}_{(c),5}^{\param, \cvarw} 
& =
\frac{\step}{\nagent}
\sum_{i=1}^\nagent
\int
\PE\left[  \locreste{i}{1:\nlupdates}
\Big(
 \frac{1}{\nlupdates} \locreste{c}{1:\nlupdates}
- \frac{1}{\nagent \nlupdates} \sum_{i'=1}^\nagent \locreste{i'}{1:\nlupdates}
\Big)^\top \right]
\statdist{\step, \nlupdates}( \rmd \param, \rmd \Cvarw )
\eqsp.
\end{align*}
\end{lemma}
\begin{proof}
Using \Cref{lem:expansion-loc-glob-param} and \Cref{lem:expansion-control-var}, we have
\begin{align*}
& \Big( \globparam{+} - \paramlim \Big)
\Big( \cvar{c}{+} - \cvarlim{c} \Big)^\top
=
\left( 
\globcontractw \left( \param - \paramlim \right)
+ \frac{\step\nlupdates}{\nagent} \sum_{i=1}^\nagent \shiftedlocmat{i}{1:\nlupdates} \left( \cvar{i}{} - \cvarlim{i} \right)
- \frac{\step}{\nagent} \sum_{i=1}^\nagent \locreste{i}{1:\nlupdates}
- \frac{\step}{\nagent} \sum_{i=1}^\nagent \locnoiseabv{i}{1:\nlupdates}
\right) \\
& \times \left( 
\diffcontractc{c} \left( \param - \paramlim \right)
+ \shiftedlocmat{c}{1:\nlupdates} \left( \cvar{c}{} - \cvarlim{c} \right)
- \frac{1}{\nagent} \sum_{i'=1}^\nagent
\shiftedlocmat{i'}{1:\nlupdates} \left( \cvar{i'}{} - \cvarlim{i'} \right)
\right.
\\
& \qquad \left.
- \frac{1}{\nlupdates} \locreste{c}{1:\nlupdates}
+ \frac{1}{\nagent \nlupdates} \sum_{i'=1}^\nagent \locreste{i'}{1:\nlupdates}
- \frac{1}{\nlupdates} \locnoiseabv{c}{1:\nlupdates}
+ \frac{1}{\nagent \nlupdates} \sum_{i=1}^\nagent \locnoiseabv{i'}{1:\nlupdates}
\right)
\eqsp.
\end{align*}
Taking the expectation, we have
\begin{align*}
& \PE\left[ \Big( \globparam{+} - \paramlim \Big)
\Big( \cvar{c}{+} - \cvarlim{c} \Big)^\top \right]
\\
& =
\left( \!
\globcontractw \left( \param - \paramlim \right)
+ \frac{\step\nlupdates}{\nagent} \sum_{i=1}^\nagent \shiftedlocmat{i}{1:\nlupdates} \left( \cvar{i}{} - \cvarlim{i} \right)\!
\right) \!\!\left( \!
\diffcontractc{c} \left( \param - \paramlim \right)
+ \shiftedlocmat{c}{1:\nlupdates} \left( \cvar{c}{} - \cvarlim{c} \right)
- \frac{1}{\nagent} \sum_{i'=1}^\nagent
\shiftedlocmat{i'}{1:\nlupdates} \left( \cvar{i'}{} - \cvarlim{i'} \right)
\!\right)^\top
\\
& \quad
- \frac{\step}{\nagent} \sum_{i=1}^\nagent \PE\left[ \locnoiseabv{i}{1:\nlupdates}\times \Big( 
- \frac{1}{\nlupdates} \locnoiseabv{c}{1:\nlupdates}
+ \frac{1}{\nagent \nlupdates} \sum_{i'=1}^\nagent \locnoiseabv{i'}{1:\nlupdates}
\Big)^\top  \right]
\\
& \quad
- \frac{\step}{\nagent} \sum_{i=1}^\nagent \PE\left[ \locnoiseabv{i}{1:\nlupdates}\times \Big( 
- \frac{1}{\nlupdates} \locreste{c}{1:\nlupdates}
+ \frac{1}{\nagent \nlupdates} \sum_{i'=1}^\nagent \locreste{i'}{1:\nlupdates}
\Big)^\top  \right]
- \frac{\step}{\nagent} \sum_{i=1}^\nagent \PE\left[ \locreste{i}{1:\nlupdates} \Big( - \frac{1}{\nlupdates} \locnoiseabv{c}{1:\nlupdates}
+ \frac{1}{\nagent \nlupdates} \sum_{i'=1}^\nagent \locnoiseabv{i'}{1:\nlupdates}
\Big)^\top  \right]
\\
& \quad
+
\PE\left[ \Big( 
\globcontractw \left( \param - \paramlim \right)
+ \frac{\step\nlupdates}{\nagent} \sum_{i=1}^\nagent \shiftedlocmat{i}{1:\nlupdates} \left( \cvar{i}{} - \cvarlim{i} \right)
- \frac{\step}{\nagent} \sum_{i=1}^\nagent \locreste{i}{1:\nlupdates}
\Big)
\Big(
- \frac{1}{\nlupdates} \locreste{c}{1:\nlupdates}
+ \frac{1}{\nagent \nlupdates} \sum_{i'=1}^\nagent \locreste{i'}{1:\nlupdates}
\Big)^\top  \right]
\\
& \quad
- \frac{\step}{\nagent} \sum_{i=1}^\nagent \PE\left[ \locreste{i}{1:\nlupdates} \right]
\left( 
\diffcontractc{c} \left( \param - \paramlim \right)
+ \shiftedlocmat{c}{1:\nlupdates} \left( \cvar{c}{} - \cvarlim{c} \right)
- \frac{1}{\nagent} \sum_{i'=1}^\nagent
\shiftedlocmat{i'}{1:\nlupdates} \left( \cvar{i'}{} - \cvarlim{i'} \right)
\right)^\top
\eqsp.
\end{align*}
The result follows by expanding the first term of the right hand side and integrating the resulting identity over \Scaffold's stationary distribution.

\end{proof}

\begin{lemma}
\label{lem:expansion-squared-cvar-cvar-plus}
Assume \Cref{assum:strong-convexity}, \Cref{assum:smoothness} and \Cref{assum:smooth-var}.
Assume the step size $\step$ and the number of local updates $\nlupdates$ satisfy $\step \nlupdates (\lip + \strcvx) \leq 1$. 
Then, for $c, c' \in \iint{1}{\nagent}$ such that $c \neq c'$, it holds that
\begin{align*}
& \covcvar{c,c} =
\diffcontractc{c}
\covparam
\diffcontractc{c'}
+ \frac{1}{\nlupdates^2} \loccovonestep{c}
- \frac{2}{\nagent\nlupdates^2} \loccovonestep{c}
+ \frac{1}{\nagent\nlupdates^2} \covonestep
\\
& \quad  
+ \diffcontractc{c} \covparamcvar{c} \shiftedlocmat{c}{1:\nlupdates}
- \frac{1}{\nagent}
\sum_{i'=1}^\nagent
\diffcontractc{c} \covparamcvar{i'} \shiftedlocmat{i'}{1:\nlupdates} 
+ \quad \shiftedlocmat{c}{1:\nlupdates} \covcvarparam{c} \diffcontractc{c} 
 - \frac{1}{\nagent} \sum_{i=1}^\nagent
\shiftedlocmat{i}{1:\nlupdates} 
\covcvarparam{i} \diffcontractc{c}
\\
& \quad
+ \shiftedlocmat{c}{1:\nlupdates} 
\covcvar{c,c}
\shiftedlocmat{c}{1:\nlupdates} 
- \frac{1}{\nagent} \sum_{i'=1}^\nagent \shiftedlocmat{c}{1:\nlupdates} 
\covcvar{c,i'}
\shiftedlocmat{i'}{1:\nlupdates}
- \frac{1}{\nagent} \sum_{i=1}^\nagent
\shiftedlocmat{i}{1:\nlupdates} 
\covcvar{i,c}
\shiftedlocmat{c}{1:\nlupdates}
+ \frac{1}{\nagent^2} 
\sum_{i=1}^\nagent \sum_{i'=1}^\nagent
\shiftedlocmat{i}{1:\nlupdates} 
\covcvar{i,i'}
\shiftedlocmat{i'}{1:\nlupdates} 
+ \mathrm{R}_{(c,c)}^{\cvarw}
\eqsp,
\\
& \covcvar{c,c'} =
\diffcontractc{c}
\covparam
\diffcontractc{c'}
- \frac{1}{\nagent\nlupdates^2} \loccovonestep{c}
- \frac{1}{\nagent\nlupdates^2} \loccovonestep{c'}
+ \frac{1}{\nagent\nlupdates^2} \covonestep
\\
&  
\quad + \diffcontractc{c}
\covparamcvar{c'}
\shiftedlocmat{c'}{1:\nlupdates}
- \frac{1}{\nagent}
\sum_{i'=1}^\nagent
\diffcontractc{c}
\covparamcvar{i'}
\shiftedlocmat{i'}{1:\nlupdates} 
+\shiftedlocmat{c}{1:\nlupdates} \covcvarparam{c}
\diffcontractc{c'}
 - \frac{1}{\nagent} \sum_{i=1}^\nagent
\shiftedlocmat{i}{1:\nlupdates} 
\covcvarparam{i}
\diffcontractc{c'}
\\
& \quad
+ \shiftedlocmat{c}{1:\nlupdates} 
\covcvar{c,c'}
\shiftedlocmat{c'}{1:\nlupdates} 
- \frac{1}{\nagent} \sum_{i'=1}^\nagent \shiftedlocmat{c}{1:\nlupdates} 
\covcvar{c,i'}
\shiftedlocmat{i'}{1:\nlupdates}
- \frac{1}{\nagent} \sum_{i=1}^\nagent
\shiftedlocmat{i}{1:\nlupdates} 
\covcvar{i,c'}
\shiftedlocmat{c'}{1:\nlupdates}
\!+\! \frac{1}{\nagent^2} 
\sum_{i=1}^\nagent \sum_{i'=1}^\nagent
\shiftedlocmat{i}{1:\nlupdates} 
\covcvar{i,i'}
\shiftedlocmat{i'}{1:\nlupdates} 
+ \mathrm{R}_{(c,c')}^{\cvarw}
\eqsp,
\end{align*}
where $\mathrm{R}^{\cvarw}_{(c,c')} = \mathrm{R}^{\cvarw}_{(c,c'),1} + \mathrm{R}_{(c',c),1}^{\cvarw}{}^\top 
+ \mathrm{R}_{(c,c'),2}^{\cvarw} + \mathrm{R}_{(c',c),2}^{\cvarw}{}^\top
+ \mathrm{R}_{(\star,c'),3}^{\cvarw} + \mathrm{R}_{(\star,c)}^{\cvarw}{}^\top 
+ \mathrm{R}_{(c,c'),4}^{\cvarw} + \mathrm{R}_{(c',c),4}^{\cvarw}{}^\top
+ \mathrm{R}_{(c,c'),5}^{\cvarw}$, with
\begin{align*}
\mathrm{R}_{(c,c'),1}^{\cvarw} 
& =
- \frac{1}{\nlupdates} \int 
\diffcontractc{c} \left( \param - \paramlim \right)
\PE\left[ \locreste{c'}{1:\nlupdates}
- \frac{1}{\nagent} \sum_{i'=1}^\nagent \locreste{i'}{1:\nlupdates} \right]^\top
\statdist{\step, \nlupdates}( \rmd \param, \rmd \Cvarw )
\eqsp,
\\
\mathrm{R}_{(c,c'),2}^{\cvarw} & = 
- \frac{1}{\nlupdates}
\int
\shiftedlocmat{c}{1:\nlupdates} \left( \cvar{c}{} - \cvarlim{c} \right)
\PE\left[ \locreste{c'}{1:\nlupdates}
- \frac{1}{\nagent} \sum_{i'=1}^\nagent \locreste{i'}{1:\nlupdates} \right]^\top
\statdist{\step, \nlupdates}( \rmd \param, \rmd \Cvarw )
\eqsp,
\\
\mathrm{R}_{(\star,c'),3}^{\cvarw}  & =
\frac{1}{\nagent\nlupdates}
\int 
\sum_{i=1}^\nagent
\shiftedlocmat{i}{1:\nlupdates} \left( \cvar{i}{} - \cvarlim{i} \right)
\PE\left[ \frac{1}{\nlupdates} \locreste{c'}{1:\nlupdates}
- \frac{1}{\nagent \nlupdates} \sum_{i'=1}^\nagent \locreste{i'}{1:\nlupdates} \right]^\top
\statdist{\step, \nlupdates}( \rmd \param, \rmd \Cvarw )
\eqsp,
\\
\mathrm{R}_{(c,c'),4}^{\cvarw}
& =
\frac{1}{\nlupdates^2} 
\int 
\PE\left[ \Big( \locreste{c}{1:\nlupdates}
- \frac{1}{\nagent} \sum_{i=1}^\nagent \locreste{i}{1:\nlupdates} \Big)
\Big(
\locnoiseabv{c'}{1:\nlupdates}
- \frac{1}{\nagent} \sum_{i=1}^\nagent \locnoiseabv{i'}{1:\nlupdates}
\Big)^\top
\right]
\statdist{\step, \nlupdates}( \rmd \param, \rmd \Cvarw )
\eqsp,
\\
\mathrm{R}_{(c,c'),5}^{\cvarw} & =
\frac{1}{\nlupdates^2} \int \PE\left[ \left( \locreste{c}{1:\nlupdates}
+ \frac{1}{\nagent} \sum_{i=1}^\nagent \locreste{i}{1:\nlupdates} \right)
\left( \locreste{c'}{1:\nlupdates}
- \frac{1}{\nagent } \sum_{i'=1}^\nagent \locreste{i'}{1:\nlupdates}
\right)^\top
\right]
\statdist{\step, \nlupdates}( \rmd \param, \rmd \Cvarw )
\eqsp.
\end{align*}

\begin{proof}
Recall the expression of $\cvar{c}{+}$ from \Cref{lem:expansion-control-var}, we have
\begin{align*}
\cvar{c}{+}
- 
\cvarlim{c}{} 
& =
\diffcontractc{c} \left( \param - \paramlim \right)
+ {\shiftedlocmat{c}{1:\nlupdates}} \left( \cvar{c}{} - \cvarlim{c} \right)
- \frac{1}{\nagent} \sum_{i=1}^\nagent
\shiftedlocmat{i}{1:\nlupdates} \left( \cvar{i}{} - \cvarlim{i} \right)
\\
& \quad 
- \frac{1}{\nlupdates} \locreste{c}{1:\nlupdates}
+ \frac{1}{\nagent \nlupdates} \sum_{i=1}^\nagent \locreste{i}{1:\nlupdates}
- \frac{1}{\nlupdates} \locnoiseabv{c}{1:\nlupdates}
+ \frac{1}{\nagent \nlupdates} \sum_{i=1}^\nagent \locnoiseabv{i}{1:\nlupdates}
\eqsp.
\end{align*}
Taking the expectation and expanding the product, we obtain, for any $c, c' \in \iint{1}{\nagent}$,
\begin{align*}
& \PE\left[ \Big( \cvar{c}{+} - \cvarlim{c} \Big)
\Big( \cvar{c'}{+} - \cvarlim{c'} \Big)^\top \right]
\\
& =
\diffcontractc{c}
\left( \param - \paramlim \right)^{\otimes 2}
\diffcontractc{c'}
+ 
\diffcontractc{c} \left( \param - \paramlim \right)
\left( \cvar{c'}{} - \cvarlim{c'} \right)^\top
\shiftedlocmat{c'}{1:\nlupdates}
\\
& \quad - \frac{1}{\nagent}
\sum_{i'=1}^\nagent
\diffcontractc{c} \left( \param - \paramlim \right) 
\left( \cvar{i'}{} - \cvarlim{i'} \right)^{\top}
\shiftedlocmat{i'}{1:\nlupdates} 
- \frac{1}{\nlupdates} \diffcontractc{c} \left( \param - \paramlim \right)
\PE\left[ \locreste{c'}{1:\nlupdates}
- \frac{1}{\nagent} \sum_{i'=1}^\nagent \locreste{i'}{1:\nlupdates}\right]^\top
\\
& \quad
+ \shiftedlocmat{c}{1:\nlupdates} \left( \cvar{c}{} - \cvarlim{c} \right) \left( \param - \paramlim \right)^\top
 \diffcontractc{c'}
+ \shiftedlocmat{c}{1:\nlupdates} \left( \cvar{c}{} - \cvarlim{c} \right) \left( \cvar{c'}{} - \cvarlim{c'} \right)^\top
\shiftedlocmat{c'}{1:\nlupdates} 
\\
& \quad
- \frac{1}{\nagent} \sum_{i'=1}^\nagent \shiftedlocmat{c}{1:\nlupdates} \left( \cvar{c}{} - \cvarlim{c} \right)
 \left( \cvar{i'}{} - \cvarlim{i'} \right)^\top
\shiftedlocmat{i'}{1:\nlupdates}
- \frac{1}{\nlupdates} \shiftedlocmat{c}{1:\nlupdates} \left( \cvar{c}{} - \cvarlim{c} \right)
\PE\left[ \locreste{c'}{1:\nlupdates}
- \frac{1}{\nagent} \sum_{i'=1}^\nagent \locreste{i'}{1:\nlupdates} \right]^\top
\\
& \quad
- \frac{1}{\nagent} \sum_{i=1}^\nagent
\shiftedlocmat{i}{1:\nlupdates} \left( \cvar{i}{} - \cvarlim{i} \right)
\left( \param - \paramlim \right)^\top \diffcontractc{c'}
- \frac{1}{\nagent} \sum_{i=1}^\nagent
\shiftedlocmat{i}{1:\nlupdates} \left( \cvar{i}{} - \cvarlim{i} \right)
\left( \cvar{c'}{} - \cvarlim{c'} \right)^\top
\shiftedlocmat{c'}{1:\nlupdates} 
\\
& \quad
+ \frac{1}{\nagent^2} 
\sum_{i=1}^\nagent \sum_{i'=1}^\nagent
\shiftedlocmat{i}{1:\nlupdates} \!
\left( \cvar{i}{} - \cvarlim{i} \right)\!
\left( \cvar{i'}{} - \cvarlim{i'} \right)^\top\!
\shiftedlocmat{i'}{1:\nlupdates} 
+ \frac{1}{\nagent\nlupdates} \sum_{i=1}^\nagent\!
\shiftedlocmat{i}{1:\nlupdates} \left( \cvar{i}{} - \cvarlim{i} \right)\!
\PE\left[ \locreste{c'}{1:\nlupdates}
- \frac{1}{\nagent} \sum_{i'=1}^\nagent \locreste{i'}{1:\nlupdates} \right]^\top
\\
& \quad
- \frac{1}{\nlupdates} \PE\left[ \locreste{c}{1:\nlupdates}
- \frac{1}{\nagent} \sum_{i=1}^\nagent \locreste{i}{1:\nlupdates} \right]
\Big( 
\left( \param - \paramlim \right)^\top
\diffcontractc{c'}
+ \left( \cvar{c'}{} - \cvarlim{c'} \right)^\top 
\shiftedlocmat{c'}{1:\nlupdates} 
- \frac{1}{\nagent} \sum_{i'=1}^\nagent
\left( \cvar{i'}{} - \cvarlim{i'} \right)^\top
\shiftedlocmat{i'}{1:\nlupdates} 
\Big)
\\
& \quad
+ \frac{1}{\nlupdates^2} \PE\left[ \Big( \locreste{c}{1:\nlupdates}
- \frac{1}{\nagent} \sum_{i=1}^\nagent \locreste{i}{1:\nlupdates} \Big)
\Big( \locreste{c'}{1:\nlupdates}
- \frac{1}{\nagent } \sum_{i'=1}^\nagent \locreste{i'}{1:\nlupdates}
+ \locnoiseabv{c'}{1:\nlupdates}
- \frac{1}{\nagent} \sum_{i=1}^\nagent \locnoiseabv{i'}{1:\nlupdates}
\Big)^\top
\right]
\\
& \quad
+ \frac{1}{\nlupdates^2} \PE\left[ \Big( \locnoiseabv{c}{1:\nlupdates}
- \frac{1}{\nagent} \sum_{i=1}^\nagent \locnoiseabv{i}{1:\nlupdates} \Big)
\Big(
\locreste{c'}{1:\nlupdates}
- \frac{1}{\nagent} \sum_{i'=1}^\nagent \locreste{i'}{1:\nlupdates}
+ \locnoiseabv{c'}{1:\nlupdates}
- \frac{1}{\nagent} \sum_{i=1}^\nagent \locnoiseabv{i'}{1:\nlupdates}
\Big)
\right]
\eqsp.
\end{align*}
Integrating over the stationary distribution, this yields
\begin{align*}
& \covcvar{c,c'}  =
\diffcontractc{c}
\covparam
\diffcontractc{c'}
\\
&  
+ 
\diffcontractc{c}
\covparamcvar{c'}
\shiftedlocmat{c'}{1:\nlupdates}
- \frac{1}{\nagent}
\sum_{i'=1}^\nagent
\diffcontractc{c}
\covparamcvar{i'}
\shiftedlocmat{i'}{1:\nlupdates} 
+\shiftedlocmat{c}{1:\nlupdates} \covcvarparam{c}
\diffcontractc{c'}
 - \frac{1}{\nagent} \sum_{i=1}^\nagent
\shiftedlocmat{i}{1:\nlupdates} 
\covcvarparam{i}
\diffcontractc{c'}
\\
& 
+ \shiftedlocmat{c}{1:\nlupdates} 
\covcvar{c,c'}
\shiftedlocmat{c'}{1:\nlupdates} 
- \frac{1}{\nagent} \sum_{i'=1}^\nagent \shiftedlocmat{c}{1:\nlupdates} 
\covcvar{c,i'}
\shiftedlocmat{i'}{1:\nlupdates}
- \frac{1}{\nagent} \sum_{i=1}^\nagent
\shiftedlocmat{i}{1:\nlupdates} 
\covcvar{i,c'}
\shiftedlocmat{c'}{1:\nlupdates}
+ \frac{1}{\nagent^2} 
\sum_{i=1}^\nagent \sum_{i'=1}^\nagent
\shiftedlocmat{i}{1:\nlupdates} 
\covcvar{i,i'}
\shiftedlocmat{i'}{1:\nlupdates} 
\\
& 
+ \frac{1}{\nlupdates^2} \int \PE\left[ \Big( \locnoiseabv{c}{1:\nlupdates}
- \frac{1}{\nagent} \sum_{i=1}^\nagent \locnoiseabv{i}{1:\nlupdates} \Big)
\Big(
\locnoiseabv{c'}{1:\nlupdates}
- \frac{1}{\nagent} \sum_{i'=1}^\nagent \locnoiseabv{i'}{1:\nlupdates}
\Big)^\top\right]
\statdist{\step, \nlupdates}( \rmd \param, \rmd \Cvarw )
\\
& 
- \frac{1}{\nlupdates} \int 
\PE \left[ \diffcontractc{c} \left( \param - \paramlim \right)\Big( 
 \locreste{c'}{1:\nlupdates}{}
- \frac{1}{\nagent} \sum_{i'=1}^\nagent \locreste{i'}{1:\nlupdates}{}\Big)^\top
+ \Big( \locreste{c}{1:\nlupdates}
- \frac{1}{\nagent} \sum_{i=1}^\nagent \locreste{i}{1:\nlupdates} \Big)
\left( \param - \paramlim \right)^\top
\diffcontractc{c'}
\right]
\statdist{\step, \nlupdates}( \rmd \param, \rmd \Cvarw )
\\
& 
- \frac{1}{\nlupdates}
\int \PE \left[ 
\shiftedlocmat{c}{1:\nlupdates} \left( \cvar{c}{} \!\!-\! \cvarlim{c} \right)\!
\Big( \locreste{c'}{1:\nlupdates}
- \frac{1}{\nagent} \sum_{i'=1}^\nagent \locreste{i'}{1:\nlupdates} \Big)\!^\top\!\!
+\!\! \Big( \locreste{c}{1:\nlupdates}
- \frac{1}{\nagent} \sum_{i=1}^\nagent \locreste{i}{1:\nlupdates} \Big)\!
\left( \cvar{c'}{} \!\!-\! \cvarlim{c'} \right)^\top \!\!
\shiftedlocmat{c'}{1:\nlupdates} 
\right]
\statdist{\step, \nlupdates}( \rmd \param, \rmd \Cvarw )
\\
& 
+ \frac{1}{\nagent\nlupdates}
\int 
\PE
\Bigg[ \sum_{i=1}^\nagent
\shiftedlocmat{i}{1:\nlupdates} \left( \cvar{i}{} - \cvarlim{i} \right)
\Big( \locreste{c'}{1:\nlupdates}
- \frac{1}{\nagent} \sum_{i'=1}^\nagent \locreste{i'}{1:\nlupdates} \Big)^\top
\Bigg]
\statdist{\step, \nlupdates}( \rmd \param, \rmd \Cvarw )
\\
& 
+ \frac{1}{\nagent\nlupdates}
\int 
\PE
\Bigg[
\Big( \locreste{c}{1:\nlupdates}
- \frac{1}{\nagent} \sum_{i=1}^\nagent \locreste{i}{1:\nlupdates} \Big)
\sum_{i'=1}^\nagent
\left( \cvar{i'}{} - \cvarlim{i'} \right)^\top
\shiftedlocmat{i'}{1:\nlupdates} 
\Bigg]
\statdist{\step, \nlupdates}( \rmd \param, \rmd \Cvarw )
\\
& 
+ \frac{1}{\nlupdates^2} 
\int 
\PE\left[ \Big( \locreste{c}{1:\nlupdates}
- \frac{1}{\nagent} \sum_{i=1}^\nagent \locreste{i}{1:\nlupdates} \Big)
\Big(
\locnoiseabv{c'}{1:\nlupdates}
- \frac{1}{\nagent} \sum_{i=1}^\nagent \locnoiseabv{i'}{1:\nlupdates}
\Big)^\top
\right]
\statdist{\step, \nlupdates}( \rmd \param, \rmd \Cvarw )
\\
& 
+ \frac{1}{\nlupdates^2} 
\int 
\PE\left[ 
\Big( \locnoiseabv{c}{1:\nlupdates}
- \frac{1}{\nagent} \sum_{i=1}^\nagent \locnoiseabv{i}{1:\nlupdates} \Big)
\Big(
\locreste{c'}{1:\nlupdates}
- \frac{1}{\nagent} \sum_{i'=1}^\nagent \locreste{i'}{1:\nlupdates}
\Big)^\top \right]
\statdist{\step, \nlupdates}( \rmd \param, \rmd \Cvarw )
\\
& 
+ \frac{1}{\nlupdates^2} \int \PE\left[ \Big( \locreste{c}{1:\nlupdates}
- \frac{1}{\nagent} \sum_{i=1}^\nagent \locreste{i}{1:\nlupdates} \Big)
\Big( \locreste{c'}{1:\nlupdates}
- \frac{1}{\nagent } \sum_{i'=1}^\nagent \locreste{i'}{1:\nlupdates}
\Big)^\top
\right]
\statdist{\step, \nlupdates}( \rmd \param, \rmd \Cvarw )
\eqsp.
\end{align*}
To study the noise term, we expand
\begin{align*}
& \Big( \locnoiseabv{c}{1:\nlupdates}
- \frac{1}{\nagent} \sum_{i=1}^\nagent \locnoiseabv{i}{1:\nlupdates} \Big)
\Big(
\locnoiseabv{c'}{1:\nlupdates}
- \frac{1}{\nagent} \sum_{i=1}^\nagent \locnoiseabv{i'}{1:\nlupdates}
\Big)^\top
\\
& \quad =
\locnoiseabv{c}{1:\nlupdates}
\locnoiseabv{c'}{1:\nlupdates}{}^\top
- \frac{1}{\nagent} \sum_{i=1}^\nagent 
\locnoiseabv{i}{1:\nlupdates}
\locnoiseabv{c'}{1:\nlupdates}{}^\top
- \frac{1}{\nagent} \sum_{i'=1}^\nagent 
\locnoiseabv{c}{1:\nlupdates}
\locnoiseabv{i'}{1:\nlupdates}{}^\top
+ \frac{1}{\nagent^2} 
\sum_{i=1}^\nagent \sum_{i'=1}^\nagent 
\locnoiseabv{i}{1:\nlupdates}
\locnoiseabv{i'}{1:\nlupdates}{}^\top
\eqsp.
\end{align*}
Now we distinguish two cases.
First, if $c \neq c'$, we have
\begin{align*}
\frac{1}{\nlupdates^2} \!\int\! \PE\left[ \Big( \locnoiseabv{c}{1:\nlupdates}
- \frac{1}{\nagent} \sum_{i=1}^\nagent \locnoiseabv{i}{1:\nlupdates} \Big)
\Big(
\locnoiseabv{c'}{1:\nlupdates}
- \frac{1}{\nagent} \sum_{i=1}^\nagent \locnoiseabv{i'}{1:\nlupdates}
\Big)^\top \right]
\statdist{\step, \nlupdates}( \rmd \param, \rmd \Cvarw )
& 
= 
- \frac{1}{\nagent\nlupdates^2} \loccovonestep{c}
- \frac{1}{\nagent\nlupdates^2} \loccovonestep{c'}
+ \frac{1}{\nagent\nlupdates^2} \covonestep
\eqsp.
\end{align*}
Otherwise, we have $c = c'$ and
\begin{align*}
\frac{1}{\nlupdates^2} \int \PE\left[ \Big( \locnoiseabv{c}{1:\nlupdates}
- \frac{1}{\nagent} \sum_{i=1}^\nagent \locnoiseabv{i}{1:\nlupdates} \Big)
\Big(
\locnoiseabv{c}{1:\nlupdates}
- \frac{1}{\nagent} \sum_{i=1}^\nagent \locnoiseabv{i'}{1:\nlupdates}
\Big)^\top \right]
\statdist{\step, \nlupdates}( \rmd \param, \rmd \Cvarw )
& 
= 
\frac{1}{\nlupdates^2} \loccovonestep{c}
- \frac{2}{\nagent\nlupdates^2} \loccovonestep{c}
+ \frac{1}{\nagent\nlupdates^2} \covonestep
\eqsp,
\end{align*}
and plugging these identities in the above equality gives the lemma.
\end{proof}
\end{lemma}

\subsubsection{Bound on remainder terms}
\begin{lemma}
\label{lem:remainder-bound-term-by-term}
Assume \Cref{assum:strong-convexity}, \Cref{assum:smoothness} and \Cref{assum:smooth-var}.
Assume the step size $\step$ and the number of local updates $\nlupdates$ satisfy $\step \nlupdates (\lip + \strcvx) \leq 1/48$, $\step \smoothcstvar^{1/2} \nlupdates^{1/2} \le 1 / 12$, and $\step \smoothcstvar \le \lip / 12$.
Then, it holds that
\begin{align*}
\abs{ \tr \mathrm{R}^{\param} } & \le
\frac{1080 \step^{5/2} \nlupdates \thirdlip}{\strcvx^{3/2}} \sqoptvar^3
+ \frac{2 \cdot 600^2 \step^4 \nlupdates^2 \thirdlip^2}{\nagent\strcvx^2} \sqoptvar^4
\eqsp,
\\
\abs{ \tr \mathrm{R}_{(c)}^{\param, \cvarw} } & \le
\frac{6000 \step^{3/2} \thirdlip }{\strcvx^{3/2}}
\sqoptvar^3
+
\frac{2 \cdot 600^2 \step^3 \nlupdates \thirdlip^2}{\strcvx^2} \sqoptvar^4
\eqsp,
\\
\abs{ \tr \mathrm{R}^{\cvarw}_{(c,c')} } & \le
\frac{8000 \step^{1/2} \thirdlip }{ \nlupdates \strcvx^{3/2} } \sqoptvar^3
+
\frac{4 \cdot 600^2 \step^2 \thirdlip^2}{\strcvx^2} \sqoptvar^4
\eqsp,
\end{align*}
where $\rmR^{\param}$, $\rmR^{\param,\xi}$ and $\rmR^{\xi}$ are defined in \Cref{lem:expansion-squared-theta-plus}, \Cref{lem:expansion-squared-theta-cvar-plus} and \Cref{lem:expansion-squared-cvar-cvar-plus} respectively.
\end{lemma}

\begin{proof}
\textbf{Bound on ${\mathrm{R}^{\param}}$.}
We bound each of the terms from 
$\abs{ \tr \mathrm{R}^{\param} } = \abs{2 \tr \mathrm{R}_1^{\param} + 2 \tr \mathrm{R}_2^{\param} + 2 \tr \mathrm{R}_3^{\param}} \le \abs{2 \tr \mathrm{R}_1^{\param}} + \abs{2 \tr \mathrm{R}_2^{\param}} + \abs{2 \tr \mathrm{R}_3^{\param}}$.
We have, using Cauchy-Schwarz and Hölder inequalities,
\begin{align*}
\abs{ \tr \mathrm{R}_1^{\param} }
& \le
\frac{\step^2}{\nagent^2} \sum_{c=1}^\nagent 
\babs{ \int 
\PE \left[  \tr \Big( \locnoiseabv{c}{1:\nlupdates} \Big)
\left( \locreste{c}{1:\nlupdates} \right)^\top  \right]
 \statdist{\step, \nlupdates}( \rmd \param, \rmd \Cvarw ) }
\\
& \le
\frac{\step^2}{\nagent^2} \sum_{c=1}^\nagent 
\left( \int 
\PE \left[ \norm{ \locnoiseabv{c}{1:\nlupdates} }^2 \right]
 \statdist{\step, \nlupdates}( \rmd \param, \rmd \Cvarw )
\right)^{1/2}
\left( 
\int 
\PE \left[ \norm{ \locreste{c}{1:\nlupdates} }^2 \right]
 \statdist{\step, \nlupdates}( \rmd \param, \rmd \Cvarw )
 \right)^{1/2}
\eqsp.
\end{align*}
By \Cref{lem:bound-noise-loc} and \Cref{lem:bound-reste-loc},
\begin{align*}
\abs{ \tr \mathrm{R}_1^\param } 
& \le
\frac{\step^2}{\nagent}
\left( 
\nlupdates^{1/2} \sqoptvar + \frac{6 \step^{1/2} \smoothcstvar^{1/2} \nlupdates^{1/2}}{\strcvx^{1/2}} \sqoptvar
\right)
\frac{600 \step \nlupdates \thirdlip}{\strcvx} \optvar
=
\frac{600 \step^3 \nlupdates^{3/2} \thirdlip}{\nagent\strcvx} \sqoptvar^3
+ 
\frac{6 \cdot 600 \step^{7/2} \smoothcstvar^{1/2} \nlupdates^{3/2} \thirdlip}{\nagent \strcvx^{3/2}} \sqoptvar^3
\eqsp.
\end{align*}
Then, by \Cref{cor:crude-bounds-global}, \Cref{lem:bound-reste-loc}, and \Cref{lem:crude-bound-local-and-cvar}
\begin{align*}
\abs{ \tr \mathrm{R}_2^{\param} } & \le
\frac{\step}{\nagent} \sum_{c=1}^\nagent
\babs{ \int \tr {\PE}\left[\locreste{c}{1:\nlupdates} \right] \left( \param - \paramlim \right)^\top \globcontractw  \statdist{\step, \nlupdates}( \rmd \param, \rmd \Cvarw ) }
\\
& \quad
+ \frac{\step^2 \nlupdates}{\nagent^2} 
\sum_{c=1}^\nagent
\sum_{c'=1}^\nagent
\babs{\int 
 \tr {\PE}\left[\locreste{c}{1:\nlupdates} \right] \left( \cvar{c'}{} - \cvarlim{c'} \right)^\top \shiftedlocmat{c'}{1:\nlupdates} 
 \statdist{\step, \nlupdates}( \rmd \param, \rmd \Cvarw ) }
 \\
& \le
\step 
\cdot \frac{28 \step \nlupdates \thirdlip}{\strcvx} \optvar
\cdot \frac{3 \step^{1/2}}{\strcvx^{1/2}} \sqoptvar
+
{\step^2 \nlupdates}
\cdot \frac{28 \step \nlupdates \thirdlip}{\strcvx} \optvar
\cdot \frac{8 \lip^{1/2}}{\strcvx^{1/2} \nlupdates^{1/2}} \sqoptvar
\eqsp,
\end{align*}
which gives, using $\step \nlupdates \lip \le 1/48$ in the second inequality,
\begin{align*}
\abs{ \tr \mathrm{R}_2^{\param} } & \le
\frac{84 \thirdlip \step^{5/2} \nlupdates }{\strcvx^{3/2}} \sqoptvar^3
+
\frac{224 \thirdlip \step^{3} \nlupdates^{3/2} \lip^{1/2} }{\strcvx^{3/2}} \sqoptvar^3
 \le
\frac{90 \thirdlip \step^{5/2} \nlupdates }{\strcvx^{3/2}} \sqoptvar^3
\eqsp.
\end{align*}
Finally, by \Cref{lem:bound-reste-loc}, we obtain
\begin{align*}
\abs{ \tr \mathrm{R}_3^{\param} }
& \le
\frac{\step^2}{\nagent^2} \sum_{c=1}^\nagent \sum_{c'=1}^\nagent 
\int \babs{ 
{\PE}\left[ \tr \Big(  \locreste{c}{1:\nlupdates} \Big)
\Big(  \locreste{c'}{1:\nlupdates} \Big)^\top \right]
 \statdist{\step, \nlupdates}( \rmd \param, \rmd \Cvarw )
 }
\le
\step^2
\frac{600^2 \step^2 \nlupdates^2 \thirdlip^2}{\strcvx^2} \sqoptvar^4
\eqsp.
\end{align*}
Summing these inequalities, we obtain
\begin{align*}
\abs{ \tr \mathrm{R}^{\param} } & \le
\frac{1200 \step^3 (\strcvx^{1/2} + 6 \step^{1/2} \smoothcstvar^{1/2} ) \nlupdates^{3/2} \thirdlip}{\nagent\strcvx^{3/2}} \sqoptvar^3
+ \frac{180 \thirdlip \step^{5/2} \nlupdates }{\strcvx^{3/2}} \sqoptvar^3
+ \frac{2 \cdot 600^2 \step^4 \nlupdates^2 \thirdlip^2}{\strcvx^2} \sqoptvar^4
\eqsp,
\end{align*}
and the result follows from $\step \smoothcstvar^{1/2} \nlupdates^{1/2} \le 1/12$ and $\step^{1/2} \nlupdates^{1/2} \strcvx^{1/2} \le 1/6$.

\textbf{Bound on ${\mathrm{R}_{(c)}^{\param,\cvarw}}$.}
We bound each term of 
$\abs{ \tr \mathrm{R}_{(c)}^{\param, \cvarw} }
= \abs{ \mathrm{R}_{(c),1}^{\param, \cvarw} 
+ \tr \mathrm{R}_{(c),2}^{\param, \cvarw} 
+ \tr \mathrm{R}_{(c),3}^{\param, \cvarw}
+ \tr \mathrm{R}_{(c),4}^{\param, \cvarw} 
+ \tr \mathrm{R}_{(c),5}^{\param, \cvarw} }
\le \abs{ \mathrm{R}_{(c),1}^{\param, \cvarw} }
+ \abs{\tr \mathrm{R}_{(c),2}^{\param, \cvarw} }
+ \abs{\tr \mathrm{R}_{(c),3}^{\param, \cvarw} }
+ \abs{\tr \mathrm{R}_{(c),4}^{\param, \cvarw} }
+ \abs{\tr \mathrm{R}_{(c),5}^{\param, \cvarw} }$.
By \Cref{lem:bound-noise-loc}, and \Cref{lem:bound-reste-loc},
\begin{align*}
\abs{ \tr \mathrm{R}_{(c),1}^{\param, \cvarw} }
& \!\le\!
\frac{\step}{\nagent} \sum_{i=1}^\nagent \babs{ \int 
\PE\! \left[ \tr \locnoiseabv{i}{1:\nlupdates} \Big( 
\frac{1}{\nlupdates} \locreste{c}{1:\nlupdates}
\!-\! \frac{1}{\nagent \nlupdates} \sum_{i'=1}^\nagent \locreste{i'}{1:\nlupdates}
\Big)^\top
\!\!\!+ \tr \locreste{i}{1:\nlupdates} \Big( \frac{1}{\nlupdates} \locnoiseabv{c}{1:\nlupdates}
\!-\! \frac{1}{\nagent \nlupdates} \sum_{i'=1}^\nagent \locnoiseabv{i'}{1:\nlupdates}
\Big)^\top \right] 
\statdist{\step, \nlupdates}( \rmd \param, \rmd \Cvarw )
}
\\
& \le\!
\step \left(\! 
\left( \nlupdates^{1/2} \sqoptvar \!+\! \frac{6 \step^{1/2} \smoothcstvar^{1/2} \nlupdates^{1/2}}{\strcvx^{1/2}} \sqoptvar\right)
\!\cdot\!
\frac{2 \cdot 600 \step \thirdlip}{\strcvx} \sqoptvar^2
\!\right)
+
2 \step \!\left( 
\left( \nlupdates^{1/2} \sqoptvar \!+\! \frac{6 \step^{1/2} \smoothcstvar^{1/2} \nlupdates^{1/2}}{\strcvx^{1/2}} \sqoptvar\right)
\!\cdot\!
\frac{600 \step \thirdlip}{\strcvx} \sqoptvar^2
\right)
\eqsp,
\end{align*}
which implies $\abs{ \tr \mathrm{R}_{(c),1}^{\param, \cvarw} }
\le \frac{2400 \step^{2} \thirdlip (\strcvx^{1/2}+6\step^{1/2} \smoothcstvar^{1/2}) \nlupdates^{1/2}}{\strcvx^{3/2}}
\sqoptvar^3 $.
Then, using \Cref{cor:crude-bounds-global}, \Cref{lem:bound-diff-loccontract}, and \Cref{lem:bound-reste-loc},
\begin{align*}
\abs{ \tr \mathrm{R}_{(c),2}^{\param, \cvarw} }
& \le
\frac{\step}{\nagent} \sum_{i=1}^\nagent
\babs{ \int \tr \PE\left[ \locreste{i}{1:\nlupdates} \right]
\left(  \diffcontractc{c} \left( \param - \paramlim \right)
\right)^\top
\statdist{\step, \nlupdates}( \rmd \param, \rmd \Cvarw ) }
\le
\step 
\cdot
\frac{600 \step \thirdlip}{\strcvx} \sqoptvar^2
\cdot 
\heterboundhess
\cdot
\frac{3 \step^{1/2} }{\strcvx^{1/2}} \sqoptvar
\eqsp,
\end{align*}
which gives $\abs{ \tr \mathrm{R}_{(c),2}^{\param, \cvarw} } \le \frac{1800 \step^{5/2} \thirdlip \heterboundhess }{\strcvx^{3/2}} \sqoptvar^{3}$.
Furthermore, we have, from \Cref{lem:crude-bound-local-and-cvar}, \Cref{lem:bound-shiftedlocamat}, and \Cref{lem:bound-reste-loc},
\begin{align*}
\abs{ \tr \mathrm{R}_{(c),3}^{\param, \cvarw} }
& \le
\frac{\step}{\nagent} \sum_{i=1}^\nagent
\babs{ 
\int \tr \PE\left[ \locreste{i}{1:\nlupdates} \right]
\left(  
\shiftedlocmat{c}{1:\nlupdates} \left( \cvar{c}{} - \cvarlim{c} \right)
-
\frac{1}{\nagent}
 \sum_{i'=1}^\nagent
\shiftedlocmat{i'}{1:\nlupdates} \left( \cvar{i'}{} - \cvarlim{i'} \right)
\right)^\top
\statdist{\step, \nlupdates}( \rmd \param, \rmd \Cvarw )
}
\\
& \le
2 \step 
\cdot 
\frac{600 \step \thirdlip}{\strcvx} \optvar
\cdot
\frac{\step (\nlupdates-1) \lip}{2}
\cdot
\frac{8 \lip^{1/2}}{\strcvx^{1/2} \nlupdates^{1/2}} \sqoptvar
\eqsp,
\end{align*}
therefore, we have $\abs{ \tr \mathrm{R}_{(c),3}^{\param, \cvarw} } \le \frac{4800 \step^3 \lip^{3/2} \nlupdates^{1/2} \thirdlip }{\strcvx^{3/2}} \sqoptvar^3$.
We also bound, using \Cref{lem:crude-bound-local-and-cvar}, \Cref{lem:bound-loccontract}, \Cref{lem:bound-shiftedlocamat}, and \Cref{lem:bound-reste-loc},
\begin{align*}
\abs{ \mathrm{R}_{(c),4}^{\param, \cvarw} }
& =
\babs{ \int \tr
\PE\left[ \Big( 
\globcontractw \left( \param - \paramlim \right)
+ \frac{\step \nlupdates}{\nagent} \sum_{i=1}^\nagent \shiftedlocmat{i}{1:\nlupdates} \left( \cvar{i}{} - \cvarlim{i} \right)
\Big)
\Big(
- \frac{1}{\nlupdates} \locreste{c}{1:\nlupdates}
+ \frac{1}{\nagent \nlupdates} \sum_{i'=1}^\nagent \locreste{i'}{1:\nlupdates}
\Big)^\top \right]
\statdist{\step, \nlupdates}( \rmd \param, \rmd \Cvarw )
}
\\
& \le
\left(
\frac{3 \step^{1/2}}{\strcvx^{1/2}} \sqoptvar
+ 
\step \nlupdates \cdot 
\frac{\step \nlupdates \lip}{2} \cdot \frac{8 \lip^{1/2}}{\strcvx^{1/2} \nlupdates^{1/2}} \sqoptvar
\right)
\cdot
\frac{1200 \step \thirdlip}{\strcvx} \optvar
\eqsp,
\end{align*}
and we obtain $\abs{ \mathrm{R}_{(c),4}^{\param, \cvarw} } \le \left( \frac{3 \step^{1/2}}{\strcvx^{1/2}} \sqoptvar + \frac{4 \step^2 \nlupdates^2 \lip^{3/2} \sqoptvar }{ \strcvx^{1/2} \nlupdates^{3/2} }\right) \frac{1200 \step \thirdlip}{\strcvx} \optvar
= \frac{3600 \thirdlip  \step^{3/2} + 9600 \thirdlip  \step^3 \nlupdates^{3/2} \lip^{3/2}}{\strcvx^{3/2}} \sqoptvar^3 $.
Finally, we have, by \Cref{lem:bound-reste-loc},
\begin{align*}
\abs{ \mathrm{R}_{(c),5}^{\param, \cvarw}  }
& \le
\frac{\step}{\nagent}
\sum_{i=1}^\nagent
\babs{ 
\int
\tr
\PE\left[  \locreste{i}{1:\nlupdates}
\Big(
 \frac{1}{\nlupdates} \locreste{c}{1:\nlupdates}
- \frac{1}{\nagent \nlupdates} \sum_{i'=1}^\nagent \locreste{i'}{1:\nlupdates}
\Big)^\top \right]
\statdist{\step, \nlupdates}( \rmd \param, \rmd \Cvarw )
}
\le
\step \cdot \frac{2 \cdot 600^2 \step^2 \nlupdates \thirdlip^2}{\strcvx^2} \sqoptvar^4
\eqsp,
\end{align*}
summing these four inequalities gives
\begin{align*}
\abs{ \tr \mathrm{R}_{(c)}^{\param, \cvarw} } & \le
\frac{2400 \step^{2} \thirdlip (\strcvx^{1/2} + 6 \step^{1/2} \smoothcstvar^{1/2}) \nlupdates^{1/2}}{\strcvx^{3/2}}
\sqoptvar^3 
+
\frac{1800 \step^{5/2} \thirdlip \heterboundhess }{\strcvx^{3/2}} \sqoptvar^{3}
+
\frac{4800 \step^3 \lip^{3/2} \nlupdates^{1/2} \thirdlip }{\strcvx^{3/2}} \sqoptvar^3
\\
& \quad +
\frac{3600 \thirdlip \step^{3/2} + 9600 \thirdlip  \step^3 \nlupdates^{3/2} \lip^{3/2}}{\strcvx^{3/2}}
\sqoptvar^3
+
\frac{2 \cdot 600^2 \step^3 \nlupdates \thirdlip^2}{\strcvx^2} \sqoptvar^4
\eqsp,
\end{align*}
and the result follows from $\step \smoothcstvar^{1/2} \nlupdates^{1/2} \le 1/12$ and $\step \nlupdates (\lip + \strcvx) \le 1/48$.

\textbf{Bound on ${\mathrm{R}^{\cvarw}_{(c,c')}}$.}
We bound each term of  $\abs{ \tr \mathrm{R}^{\cvarw}_{(c,c')} } 
=\abs{ \tr \mathrm{R}^{\cvarw}_{(c,c'),1} + \tr  \mathrm{R}_{(c',c),1}^{\cvarw}{}^\top 
+ \tr  \mathrm{R}_{(c,c'),2}^{\cvarw} + \tr  \mathrm{R}_{(c',c),2}^{\cvarw}{}^\top
+ \tr \mathrm{R}_{(\star,c'),3}^{\cvarw} + \tr \mathrm{R}_{(\star,c),3}^{\cvarw}{}^\top 
+ \tr \mathrm{R}_{(c,c'),4}^{\cvarw} + \mathrm{R}_{(c',c),4}^{\cvarw}{}^\top
+ \tr \mathrm{R}_{(c,c'),5}^{\cvarw} }
\le
\abs{ \tr \mathrm{R}^{\cvarw}_{(c,c'),1}}
+ \abs{\tr  \mathrm{R}_{(c',c),1}^{\cvarw}{}^\top} 
+ \abs{\tr  \mathrm{R}_{(c,c'),2}^{\cvarw} }
+ \abs{\tr  \mathrm{R}_{(c',c),2}^{\cvarw}{}}
+ \abs{\tr \mathrm{R}_{(\star,c'),3}^{\cvarw}}
+ \abs{\tr \mathrm{R}_{(\star,c),3}^{\cvarw}{}^\top }
+ \abs{\tr \mathrm{R}_{(c,c'),4}^{\cvarw} }
+ \abs{\tr \mathrm{R}_{(c',c),4}^{\cvarw}{}}
+ \abs{\tr \mathrm{R}_{(c,c'),5}^{\cvarw} }$.
First, by \Cref{lem:bound-diff-loccontract}, and \Cref{lem:bound-reste-loc},
\begin{align*}
\abs{ \tr \mathrm{R}_{(c,c'),1}^{\cvarw} } 
& \le
\frac{1}{\nlupdates} 
\babs{ \int 
\tr
\diffcontractc{c} \left( \param - \paramlim \right)
\PE\left[ \locreste{c'}{1:\nlupdates}
- \frac{1}{\nagent} \sum_{i'=1}^\nagent \locreste{i'}{1:\nlupdates}\right]^\top
\statdist{\step, \nlupdates}( \rmd \param, \rmd \Cvarw )
}
 \le
\frac{1}{\nlupdates}
\!\cdot\!
\heterboundhess
\!\cdot\!
\frac{3 \step^{1/2}}{\strcvx^{1/2}} \sqoptvar
\!\cdot\!
\frac{1200 \step \nlupdates \thirdlip}{\strcvx} \sqoptvar^2
\eqsp,
\end{align*}
which gives $\abs{ \tr \mathrm{R}_{(c,c'),1}^{\cvarw} } \le \frac{3600 \step^{3/2} \thirdlip }{ \strcvx^{3/2} } \sqoptvar^3$.
Then, using \Cref{lem:crude-bound-local-and-cvar}, \Cref{lem:bound-shiftedlocamat}, and \Cref{lem:bound-reste-loc}, we have that
\begin{align*}
\abs{ \tr \mathrm{R}_{(c,c'),2}^{\cvarw} }
& \le
\frac{1}{\nlupdates}
\babs{
\int
\tr
\shiftedlocmat{c}{1:\nlupdates} \left( \cvar{c}{} - \cvarlim{c} \right)
\PE\left[ \locreste{c'}{1:\nlupdates}
- \frac{1}{\nagent} \sum_{i'=1}^\nagent \locreste{i'}{1:\nlupdates} \right]^\top
\statdist{\step, \nlupdates}( \rmd \param, \rmd \Cvarw )
}
\\
& \le
\frac{1}{\nlupdates}
\cdot
\frac{\step (\nlupdates-1) \lip}{2}
\cdot
\frac{8 \lip^{1/2}}{\strcvx^{1/2} \nlupdates^{1/2}} \sqoptvar
\cdot
\frac{1200 \step \nlupdates \thirdlip}{\strcvx} \sqoptvar^2
\eqsp,
\end{align*}
and thus $\abs{ \tr \mathrm{R}_{(c,c'),2}^{\cvarw} } \le \frac{9600 \step^2 \nlupdates^{1/2} \lip^{3/2} \thirdlip }{ \strcvx^{3/2} } \sqoptvar^3$.
The next term can be bounded using \Cref{lem:crude-bound-local-and-cvar}, \Cref{lem:bound-shiftedlocamat}, and \Cref{lem:bound-reste-loc},
\begin{align*}
\abs{ \tr \mathrm{R}_{(\star,c'),3}^{\cvarw} }  & =
\frac{1}{\nagent\nlupdates}
\sum_{i=1}^\nagent
\babs{
\int 
\tr
\shiftedlocmat{i}{1:\nlupdates} \left( \cvar{i}{} - \cvarlim{i} \right)
\PE\left[ \frac{1}{\nlupdates} \locreste{c'}{1:\nlupdates}
- \frac{1}{\nagent \nlupdates} \sum_{i'=1}^\nagent \locreste{i'}{1:\nlupdates} \right]^\top
\statdist{\step, \nlupdates}( \rmd \param, \rmd \Cvarw )
}
\\
& \le
\frac{1}{\nlupdates}
\cdot
\frac{\step (\nlupdates-1) \lip}{2}
\cdot
\frac{8 \lip^{1/2}}{\strcvx^{1/2} \nlupdates^{1/2}} \sqoptvar
\cdot
\frac{1200 \step \nlupdates \thirdlip}{\strcvx} \sqoptvar^2
\eqsp,
\end{align*}
which implies $\abs{ \tr \mathrm{R}_{(\star,c'),3}^{\cvarw} } \le \frac{9600 \step^2 \nlupdates^{1/2} \lip^{3/2} \thirdlip }{ \strcvx^{3/2} } \sqoptvar^3$.
Moreover, we have, by \Cref{lem:bound-noise-loc} and \Cref{lem:bound-reste-loc},
\begin{align*}
\abs{ \tr \mathrm{R}_{(c,c'),4}^{\cvarw} }
& =
\frac{1}{\nlupdates^2} 
\babs{
\int 
\tr
\PE\left[ \left( \locreste{c}{1:\nlupdates}
- \frac{1}{\nagent} \sum_{i=1}^\nagent \locreste{i}{1:\nlupdates} \right)
\left(
\locnoiseabv{c'}{1:\nlupdates}
- \frac{1}{\nagent} \sum_{i=1}^\nagent \locnoiseabv{i'}{1:\nlupdates}
\right)^\top
\right]
\statdist{\step, \nlupdates}( \rmd \param, \rmd \Cvarw )
}
\\
& \le
\frac{1}{\nlupdates^2}
\cdot
\frac{1200 \step \nlupdates \thirdlip}{\strcvx} \sqoptvar^2
\cdot
\left(
2 \nlupdates^{1/2} \sqoptvar + \frac{12 \step^{1/2} \smoothcstvar^{1/2} \nlupdates^{1/2}}{\strcvx^{1/2}} \sqoptvar
\right)
\eqsp,
\end{align*}
and thus $\abs{ \tr \mathrm{R}_{(c,c'),4}^{\cvarw} } \le \frac{4800 \step \thirdlip \left( \strcvx^{1/2} + 6 \step^{1/2} \smoothcstvar^{1/2} \right)}{\strcvx^{3/2} \nlupdates^{1/2}} \sqoptvar^3$.
Finally, \Cref{lem:bound-reste-loc} gives
\begin{align*}
\abs{ \tr \mathrm{R}_{(c,c'),5}^{\cvarw} }
& =
\frac{1}{\nlupdates^2} \babs{ \int \tr \PE\left[ \left( \locreste{c}{1:\nlupdates}
- \frac{1}{\nagent} \sum_{i=1}^\nagent \locreste{i}{1:\nlupdates} \right)
\left( \locreste{c'}{1:\nlupdates}
- \frac{1}{\nagent } \sum_{i'=1}^\nagent \locreste{i'}{1:\nlupdates}
\right)^\top
\right]
\statdist{\step, \nlupdates}( \rmd \param, \rmd \Cvarw )
}
\\
& \le
\frac{1}{\nlupdates^2}
\frac{4 \cdot 600^2 \step^2 \nlupdates^2 \thirdlip^2}{\strcvx^2} \sqoptvar^4
\eqsp.
\end{align*}
Combining these bounds, we obtain
\begin{align*}
\tr \mathrm{R}^{\cvarw}_{(c,c')} & \le
\frac{7200 \step^{3/2} \heterboundhess \thirdlip }{ \strcvx^{3/2} } \sqoptvar^3
+
\frac{19200 \step^2 \nlupdates^{1/2} \lip^{3/2} \thirdlip }{ \strcvx^{3/2} } \sqoptvar^3
+
\frac{19200 \step^2 \nlupdates^{1/2} \lip^{3/2} \thirdlip }{ \strcvx^{3/2} } \sqoptvar^3
\\
& \quad
+
\frac{9600 \step \thirdlip \left( \strcvx^{1/2} + 6 \step^{1/2} \smoothcstvar^{1/2} \right)}{\strcvx^{3/2} \nlupdates^{1/2}} \sqoptvar^3
+
\frac{4 \cdot 600^2 \step^2 \thirdlip^2}{\strcvx^2} \sqoptvar^4
\eqsp,
\end{align*}
and we conclude using $\step \smoothcstvar^{1/2} \nlupdates^{1/2} \le 1/12$ and $\step \nlupdates (\lip + \strcvx) \le 1/48$.
\end{proof}

\begin{corollary}
\label{cor:bound-remainder}
Assume \Cref{assum:strong-convexity}, \Cref{assum:smoothness} and \Cref{assum:smooth-var}.
Assume the step size $\step$ and the number of local updates $\nlupdates$ satisfy $\step \nlupdates (\lip + \strcvx) \leq 1/12$. 
Then, it holds that
\begin{align*}
\norm{ \mathrm{R}^{\param} }
+ \frac{\step (\nlupdates-1)}{\nagent} \sum_{c=1}^\nagent  \norm{ \mathrm{R}^{\param, \cvarw}_{(c)} }
+ \frac{\step^2(\nlupdates-1)^2}{\nagent^2} \sum_{c,c'=1}^\nagent \norm{ \mathrm{R}_{(c,c')}^{\cvarw} }
\le
\frac{ 15080 \step^{5/2} \nlupdates \thirdlip}{\strcvx^{3/2}}
\sqoptvar^3
+
\frac{8 \cdot 600^2 \step^4 \nlupdates^2 \thirdlip^2}{\strcvx^2} \sqoptvar^4
\eqsp.
\end{align*}
\end{corollary}
\begin{proof}
We have, using the results from \Cref{lem:remainder-bound-term-by-term},
\begin{align*}
& \norm{ \mathrm{R}^{\param} }
+ \frac{\step (\nlupdates-1)}{\nagent} \sum_{c=1}^\nagent  \norm{ \mathrm{R}^{\param, \cvarw}_{(c)} }
+ \frac{\step^2(\nlupdates-1)^2}{\nagent^2} \sum_{c,c'=1}^\nagent \norm{ \mathrm{R}_{(c,c')}^{\cvarw} }
\\
& \le
\tr { \mathrm{R}^{\param} }
+ \frac{\step (\nlupdates-1)}{\nagent} \sum_{c=1}^\nagent  \tr { \mathrm{R}^{\param, \cvarw}_{(c)} }
+ \frac{\step^2(\nlupdates-1)^2}{\nagent^2} \sum_{c,c'=1}^\nagent \tr { \mathrm{R}_{(c,c')}^{\cvarw} }
\\
& \le
\Big( 1080 \step^{5/2} \nlupdates + 6000 \step^{3/2} \cdot \step \nlupdates + \frac{8000 \step^{1/2}}{\nlupdates} \cdot \step^{2} \nlupdates^2 \Big)
\frac{ \thirdlip}{\strcvx^{3/2}}
\sqoptvar^3
+
\Big( 2 \step^4 \nlupdates^2 + 2 \step^3 \nlupdates \cdot \step \nlupdates + 4 \step^2 \cdot \step^2 \nlupdates^2 \Big)
\frac{600^2 \thirdlip^2}{\strcvx^2} \sqoptvar^4
\eqsp,
\end{align*}
and the result follows.
\end{proof}

\subsection{Upper bound on covariance matrices -- Proof of Lemma~\ref{lem:ineq-bound-covariances} Theorem~\ref{thm:bound-cov-stationary} }
\label{sec:app:upper-bound-cov-mat}

In this section, we derive an upper bound on \Scaffold's global iterates' error covariance $\norm{ \covparam }$.
To this end, we define 
\begin{align*}
\bcovparam & = \norm{ \covparam } \eqsp,
\quad
\bcovparamcvar = \frac{1}{\nagent} \sum_{c=1}^{\nagent} \norm{ \covparamcvar{c} } \eqsp,
\quad
\bcovcvareq = \frac{1}{\nagent} \sum_{c=1}^\nagent \norm{ \covcvar{c,c} } \eqsp,
\quad
\bcovcvarneq = \frac{1}{\nagent(\nagent-1)} \sum_{c\neq c'}^\nagent\norm{ \covcvar{c,c'} }
\eqsp.
\end{align*}
We also define the following quantity, relating the average norm of the noise injected at each step
\begin{align*}
\globboundnoise
= \frac{1}{\nagent} \sum_{c=1}^\nagent \norm{ \loccovonestep{c} }
\eqsp.
\end{align*}
We now derive a system of inequations that relate all the quantities we just defined.
This Lemma is a complete version of \Cref{lem:ineq-bound-covariances}.
\begin{lemma}
\label{lem:ineq-bound-covariances-appendix}
Assume \Cref{assum:strong-convexity}, \Cref{assum:smoothness} and \Cref{assum:smooth-var}.
Assume the step size $\step$ and the number of local updates $\nlupdates$ satisfy $\step \nlupdates (\lip + \strcvx) \leq 1/12$, then
\begin{align}
\label{eq:ineq-bound-covariances-param}
& \bcovparam
\le
(1 - \step \strcvx)^{\nlupdates} \bcovparam 
+ {\step^2\nlupdates(\nlupdates-1)\lip}\bcovparamcvar
+ \frac{\step^4\nlupdates^2 (\nlupdates-1)^2 \lip^2}{4} \left( \frac{1}{\nagent} 
\bcovcvareq
+ \left( 1 - \frac{1}{\nagent} \right) \bcovcvarneq \right)
+ \frac{\step^2}{\nagent} \globboundnoise
+ \norm{ \mathrm{R}^{\param} }
\eqsp,
\\
\label{eq:ineq-bound-covariances-param-cvar}
& \bcovparamcvar
\le
2 \heterboundhess \bcovparam
+ 4 {\step^3\nlupdates} (\nlupdates-1)^2 \lip^2
\Big( \frac{1}{\nagent} \bcovcvareq + \Big( 1 - \frac{1}{\nagent}\Big) \bcovcvarneq  \Big)
+ \frac{4 \step}{\nagent \nlupdates} \globboundnoise
+ \frac{2}{\nagent} \sum_{c=1}^\nagent  \norm{ \mathrm{R}^{\param, \cvarw}_{(c)} }
\eqsp,
\\
\label{eq:ineq-bound-covariances-cvar}
& \frac{1}{\nagent}
\bcovcvareq
+ \left( 1 - \frac{1}{\nagent} \right)
\bcovcvarneq
\le
2 \heterboundhess^2
\bcovparam
+ \frac{9}{\nagent \nlupdates^2} \globboundnoise
+
4 \heterboundhess
\step (\nlupdates - 1) \lip
\bcovparamcvar
+ \frac{2}{\nagent^2} \sum_{c,c'=1}^\nagent \norm{ \mathrm{R}_{(c,c')}^{\cvarw} }
\eqsp.
\end{align}
\end{lemma}
\begin{proof}
\textbf{Parameter Covariance.} 
Taking the operator norm of \Cref{lem:expansion-squared-theta-plus} and using triangle inequality and sub-multiplicativity of the matrix operator norm, we have
\begin{align*}
\norm{ \covparam }
& \le
\norm{ \globcontractw \covparam \globcontractw }
+ \frac{\step\nlupdates}{\nagent} \sum_{c=1}^\nagent
\norm{ \globcontractw\covparamcvar{c}  \shiftedlocmat{c}{1:\nlupdates} } 
+ \norm{ \shiftedlocmat{c}{1:\nlupdates} \covcvarparam{c}\globcontractw }
\\
& \quad
+ \frac{\step^2\nlupdates^2}{\nagent^2}
\sum_{c=1}^\nagent \sum_{c'=1}^\nagent
\norm{  \shiftedlocmat{c}{1:\nlupdates} \covcvar{c,c'}  \shiftedlocmat{c'}{1:\nlupdates} }
+ \frac{\step^2}{\nagent} \norm{ \covonestep }
+ \norm{ \mathrm{R}^{\param} }
\\
& \le
\norm{ \globcontractw } \norm{ \covparam } \norm{ \globcontractw }
+ \frac{\step\nlupdates}{\nagent} \sum_{c=1}^\nagent
\norm{ \globcontractw }
\norm{ \covparamcvar{c} } 
\norm{ \shiftedlocmat{c}{1:\nlupdates} }
+ 
\norm{ \shiftedlocmat{c}{1:\nlupdates} }
\norm{ \covcvarparam{c} }
\norm{ \globcontractw }
\\
& \quad 
+ \frac{\step^2\nlupdates^2}{\nagent^2}
\sum_{c=1}^\nagent \sum_{c'=1}^\nagent
\norm{ \shiftedlocmat{c}{1:\nlupdates} }
\norm{ \covcvar{c,c'} }
\norm{ \shiftedlocmat{c'}{1:\nlupdates} }
+ \frac{\step^2}{\nagent} \norm{ \covonestep }
+ \norm{ \mathrm{R}^{\param} }
\eqsp.
\end{align*}
This gives, using \Cref{lem:bound-loccontract}, \Cref{lem:bound-shiftedlocamat},
\begin{align*}
\bcovparam
& \le
(1 - \step \strcvx)^{\nlupdates} \bcovparam 
+ \frac{\step^2}{\nagent} \norm{ \covonestep }
+ \norm{ \mathrm{R}^{\param} }
+ \frac{\step\nlupdates}{\nagent} \sum_{c=1}^\nagent \left\{ 
\norm{ \covparamcvar{c} }
\cdot
 \frac{\step (\nlupdates-1) \lip}{2}
+ 
 \frac{\step (\nlupdates-1) \lip}{2}
\cdot
\bcovparamcvar
\right\}
\\
& \quad 
+ \frac{\step^2\nlupdates^2}{\nagent^2}
\sum_{c=1}^\nagent 
 \frac{\step (\nlupdates-1) \lip}{2}
\cdot
\bcovcvareq
\cdot
 \frac{\step (\nlupdates-1) \lip}{2}
+ \frac{\step^2\nlupdates^2}{\nagent^2}
\sum_{c=1}^\nagent \sum_{c'=1}^\nagent
 \frac{\step (\nlupdates-1) \lip}{2}
\cdot
\bcovcvarneq
\cdot
 \frac{\step (\nlupdates-1) \lip}{2}
\\
& \le
(1 - \step \strcvx)^{\nlupdates} \bcovparam 
+ \frac{\step^2}{\nagent} \norm{ \covonestep }
+ \norm{ \mathrm{R}^{\param} }
\\
& \quad
+ {\step^2\nlupdates(\nlupdates-1) \lip}\bcovparamcvar
+ \frac{\step^4\nlupdates^2 (\nlupdates-1)^2}{4 \nagent}
\bcovcvareq
+ \frac{\step^4\nlupdates^2 (\nlupdates-1)^2 \lip^2}{4} \left( 1 - \frac{1}{\nagent} \right)
\bcovcvarneq
\eqsp.
\end{align*}

\textbf{Parameter-Control Variate Covariance.}
By \Cref{lem:expansion-squared-theta-cvar-plus}, we have
\begin{align*}
\norm{ \covparamcvar{c} }
& \le
\norm{ \globcontractw } 
\norm{ \covparam }
\norm{ \diffcontractc{c} }
+
\norm{ \globcontractw }
\norm{ \covparamcvar{c} }
\norm{ \shiftedlocmat{c}{1:\nlupdates} }
\\
& \quad
+ \frac{1}{\nagent} \sum_{i'=1}^\nagent
\norm{ \globcontractw } 
\norm{ \covparamcvar{i'} }
\norm{ \shiftedlocmat{i'}{1:\nlupdates} }
+ \frac{\step \nlupdates}{\nagent} \sum_{i=1}^\nagent
\norm{ \shiftedlocmat{i}{1:\nlupdates} }
\norm{ \covcvarparam{i} }
\norm{ \diffcontractc{c} }
\\
& \quad
+ \frac{\step\nlupdates}{\nagent} \sum_{i=1}^\nagent 
\norm{ \shiftedlocmat{i}{1:\nlupdates} }
\norm{ \covcvar{i,c} } 
\norm{ \shiftedlocmat{c}{1:\nlupdates} }
+ \frac{\step \nlupdates}{\nagent^2} 
\sum_{i=1}^\nagent 
\sum_{i'=1}^\nagent
\norm{ \shiftedlocmat{i}{1:\nlupdates} }
\norm{ \covcvar{i,i'} } 
\norm{ \shiftedlocmat{i'}{1:\nlupdates} }
+ \frac{\step}{\nagent \nlupdates}
\norm{ \loccovonestep{c} -  \covonestep }
+ \norm{ \mathrm{R}^{\param, \cvarw}_{(c)} }
\eqsp,
\end{align*}
Averaging this inequality for $c=1$ to $\nagent$ and using \Cref{lem:bound-loccontract}, \Cref{lem:bound-shiftedlocamat}, and \Cref{lem:bound-diff-loccontract} gives
\begin{align*}
\bcovparamcvar
& \le
(1 - \step \strcvx)^{\nlupdates} \cdot \heterboundhess \cdot
\norm{ \covparam }
+
(1 - \step \strcvx)^{\nlupdates} \cdot \step (\nlupdates-1) \lip \cdot\bcovparamcvar
+ \step \nlupdates \cdot 
\step (\nlupdates-1) \lip \cdot
\heterboundhess \cdot
\bcovparamcvar
\\
& \quad
+ {\step\nlupdates}
\cdot
\step (\nlupdates-1) \lip
\cdot
\Big( \frac{1}{\nagent} \bcovcvareq + \Big( 1 - \frac{1}{\nagent}\Big) \bcovcvarneq  \Big)
\cdot
\step (\nlupdates-1) \lip
\\
& \quad
+ \frac{\step \nlupdates}{\nagent^2} 
\sum_{i=1}^\nagent 
\sum_{i'=1}^\nagent
\step (\nlupdates-1) \lip
\cdot
\Big( \frac{1}{\nagent} \bcovcvareq + \Big( 1 - \frac{1}{\nagent}\Big) \bcovcvarneq  \Big)
\cdot
\step (\nlupdates-1) \lip
+ \frac{\step}{\nagent \nlupdates}
\norm{ \loccovonestep{c} -  \covonestep }
+ \norm{ \mathrm{R}^{\param, \cvarw}_{(c)} }
\eqsp,
\end{align*}
which gives
\begin{align*}
\bcovparamcvar
& \le
\heterboundhess \cdot
\bcovparam
+ \step (\nlupdates-1) \lip \bcovparamcvar
+ \step^2 \nlupdates(\nlupdates-1) \lip
\heterboundhess \cdot
\bcovparamcvar
+ 2 {\step^3\nlupdates} (\nlupdates-1)^2 \lip^2
\Big( \frac{1}{\nagent} \bcovcvareq + \Big( 1 - \frac{1}{\nagent}\Big) \bcovcvarneq  \Big)
\\
& \quad + \frac{\step}{\nagent \nlupdates}
\norm{ \loccovonestep{c} -  \covonestep }
+ \frac{1}{\nagent} \sum_{c=1}^\nagent \norm{ \mathrm{R}^{\param, \cvarw}_{(c)} }
\eqsp,
\end{align*}
and the second inequality follows from $\step \nlupdates \lip + \step^2 \nlupdates(\nlupdates-1) \lip
\heterboundhess \le 1/2$.

\textbf{Control variate covariance.}
By \Cref{lem:expansion-squared-cvar-cvar-plus}, we have
\begin{align*}
& \norm{ \covcvar{c,c'} } 
\le
\norm{ \diffcontractc{c} }
\norm{ \covparam }
\norm{ \diffcontractc{c'} }
+ \frac{1}{\nagent\nlupdates^2} \norm{ \loccovonestep{c} }
+ \frac{1}{\nagent\nlupdates^2} \norm{ \loccovonestep{c'} }
+ \frac{1}{\nagent\nlupdates^2} \norm{ \covonestep }
\\
&  
+ \norm{ \diffcontractc{c} }
\norm{ \covparamcvar{c'} }
\norm{ \shiftedlocmat{c'}{1:\nlupdates} }
+ \frac{1}{\nagent}
\sum_{i'=1}^\nagent
\norm{ \diffcontractc{c} }
\norm{ \covparamcvar{i'} }
\norm{ \shiftedlocmat{i'}{1:\nlupdates} }
+
\norm{ \shiftedlocmat{c}{1:\nlupdates} }
\norm{ \covcvarparam{c} }
\norm{ \diffcontractc{c'} }
+ \frac{1}{\nagent} \sum_{i=1}^\nagent
\norm{ \shiftedlocmat{i}{1:\nlupdates} }
\norm{ \covcvarparam{i} }
\norm{ \diffcontractc{c'} }
\\
& 
+ 
\norm{ \shiftedlocmat{c}{1:\nlupdates} }
\norm{ \covcvar{c,c'} }
\norm{ \shiftedlocmat{c'}{1:\nlupdates} }
+ \frac{1}{\nagent} \sum_{i'=1}^\nagent 
\norm{ \shiftedlocmat{c}{1:\nlupdates} }
\norm{ \covcvar{c,i'} }
\norm{ \shiftedlocmat{i'}{1:\nlupdates} }
+ \frac{1}{\nagent} \sum_{i=1}^\nagent
\norm{ \shiftedlocmat{i}{1:\nlupdates} }
\norm{ \covcvar{i,c'} }
\norm{ \shiftedlocmat{c'}{1:\nlupdates} }
\\
& + \frac{1}{\nagent^2} 
\sum_{i=1}^\nagent \sum_{i'=1}^\nagent
\norm{ \shiftedlocmat{i}{1:\nlupdates} }
\norm{ \covcvar{i,i'} }
\norm{ \shiftedlocmat{i'}{1:\nlupdates} }
+ \norm{ \mathrm{R}_{(c,c')}^{\cvarw} }
\eqsp.
\end{align*}
Averaging over all pairs $c,c' \in \iint{1}{\nagent}$ with $c \neq c'$, we have
\begin{align*}
\bcovcvarneq
& \le
\heterboundhess \cdot
\bcovparam \cdot
\heterboundhess
+ \frac{1}{\nagent\nlupdates^2} \globboundnoise
+ \frac{1}{\nagent\nlupdates^2} \globboundnoise
+ \frac{1}{\nagent\nlupdates^2} \globboundnoise
+
\heterboundhess \cdot
\bcovparamcvar \cdot
\frac{\step (\nlupdates - 1) \lip}{2}
+ 
2 \heterboundhess \cdot
\bcovparamcvar \cdot
{\step (\nlupdates - 1) \lip}
\\
& \quad
+ 
\frac{\step^2 (\nlupdates - 1)^2 \lip^2}{4}
\bcovcvarneq 
+ 
2 \cdot \frac{\step^2 (\nlupdates - 1)^2 \lip^2}{4}
\Big( \frac{1}{\nagent} \bcovcvareq + \Big( 1 - \frac{1}{\nagent}\Big) \bcovcvarneq  \Big)
\\
& \quad
+
\frac{\step^2 (\nlupdates - 1)^2 \lip^2}{4}
\Big( \frac{1}{\nagent} \bcovcvareq + \Big( 1 - \frac{1}{\nagent}\Big) \bcovcvarneq  \Big)
+ \frac{1}{\nagent(\nagent-1)} \sum_{c\neq c'}  \norm{ \mathrm{R}_{(c,c')}^{\cvarw} }
\\
& \le
\heterboundhess^2
\bcovparam
+ \frac{3}{\nagent\nlupdates^2} \globboundnoise
+
3 \heterboundhess
\step (\nlupdates - 1) \lip
\bcovparamcvar
\\
& \quad
+ 
\frac{\step^2 (\nlupdates - 1)^2 \lip^2}{4}
\bcovcvarneq 
+ 
\frac{3 \step^2 (\nlupdates - 1)^2 \lip^2}{4}
\Big( \frac{1}{\nagent} \bcovcvareq + \Big( 1 - \frac{1}{\nagent}\Big) \bcovcvarneq  \Big)
+ \frac{1}{\nagent(\nagent-1)} \sum_{c\neq c'}  \norm{ \mathrm{R}_{(c,c')}^{\cvarw} }
\eqsp.
\end{align*}
Bounding $\frac{\step^2 (\nlupdates - 1)^2 \lip^2}{4}
\bcovcvarneq + 2 \cdot \frac{\step^2 (\nlupdates - 1)^2 \lip^2}{4}
 \bcovcvarneq + \frac{\step^2 (\nlupdates - 1)^2 \lip^2}{4}
 \bcovcvarneq \le 1/2 \bcovcvarneq $, we obtain the third inequality of the lemma.
 With similar derivations, we bound the control variates' covariances
 \begin{align*}
\bcovcvareq
& \le
\heterboundhess \cdot
\bcovparam \cdot
\heterboundhess
+ \frac{1}{\nlupdates^2} \globboundnoise
+ \frac{2}{\nagent\nlupdates^2} \globboundnoise
+ \frac{1}{\nagent\nlupdates^2} \globboundnoise
+
\heterboundhess \cdot
\bcovparamcvar \cdot
\frac{\step (\nlupdates - 1) \lip}{2}
+ 
2 \heterboundhess \cdot
\bcovparamcvar \cdot
{\step (\nlupdates - 1) \lip}
\\
& \quad
+ 
\frac{\step^2 (\nlupdates - 1)^2 \lip^2}{4}
\bcovcvareq 
+ 
2 \cdot \frac{\step^2 (\nlupdates - 1)^2 \lip^2}{4}
\Big( \frac{1}{\nagent} \bcovcvareq + \Big( 1 - \frac{1}{\nagent}\Big) \bcovcvarneq  \Big)
\\
& \quad
+
\frac{\step^2 (\nlupdates - 1)^2 \lip^2}{4}
\Big( \frac{1}{\nagent} \bcovcvareq + \Big( 1 - \frac{1}{\nagent}\Big) \bcovcvarneq  \Big)
+ \frac{1}{\nagent} \sum_{c=1}^\nagent \norm{ \mathrm{R}_{(c,c)}^{\cvarw} }
\\
& \le
\heterboundhess^2
\bcovparam
+ \frac{4}{\nlupdates^2} \globboundnoise
+
3 \heterboundhess
\step (\nlupdates - 1) \lip
\bcovparamcvar
\\
& \quad
+ 
\frac{\step^2 (\nlupdates - 1)^2 \lip^2}{4}
\bcovcvareq 
+ 
\frac{3 \step^2 (\nlupdates - 1)^2 \lip^2}{4}
\Big( \frac{1}{\nagent} \bcovcvareq + \Big( 1 - \frac{1}{\nagent}\Big) \bcovcvarneq  \Big)
+ \frac{1}{\nagent} \sum_{c=1}^\nagent \norm{ \mathrm{R}_{(c,c)}^{\cvarw} }
\eqsp.
\end{align*}
Summing these two inequalities, we obtain
\begin{align*}
\frac{1}{\nagent}
\bcovcvareq
+ \left( 1 - \frac{1}{\nagent} \right)
\bcovcvarneq
& \le
\heterboundhess^2
\bcovparam
+ \frac{8}{\nagent \nlupdates^2} \globboundnoise
+
3 \heterboundhess
\step (\nlupdates - 1) \lip
\bcovparamcvar
\\
& \quad
+
{\step^2 (\nlupdates - 1)^2 \lip^2}
\Big( \frac{1}{\nagent} \bcovcvareq + \Big( 1 - \frac{1}{\nagent}\Big) \bcovcvarneq  \Big)
+ \frac{1}{\nagent^2} \sum_{c,c'=1}^\nagent \norm{ \mathrm{R}_{(c,c')}^{\cvarw} }
\eqsp.
\end{align*}
Since $\step^2 (\nlupdates-1)^2 \lip^2 \le 1/12^2 $, we obtain
\begin{align*}
\frac{1}{\nagent}
\bcovcvareq
+ \left( 1 - \frac{1}{\nagent} \right)
\bcovcvarneq
& \le
2 \heterboundhess^2
\bcovparam
+ \frac{9}{\nagent \nlupdates^2} \globboundnoise
+
4 \heterboundhess
\step (\nlupdates - 1) \lip
\bcovparamcvar
+ \frac{2}{\nagent^2} \sum_{c,c'=1}^\nagent \norm{ \mathrm{R}_{(c,c')}^{\cvarw} }
\eqsp,
\end{align*}
which is the third inequality of the lemma.
\end{proof}

\begin{lemma}
\label{thm:bound-cov-speed-up}
Assume \Cref{assum:strong-convexity}, \Cref{assum:smoothness}, \Cref{assum:third-derivative}, \Cref{assum:heterogeneity}, and \Cref{assum:smooth-var}.
Furthermore, assume that $5 \step (\nlupdates - 1) \lip \heterboundhess \le \strcvx / 2$ and $\step \nlupdates (\lip+\strcvx) \le 1/12$ and $\frac{\step \smoothcstvar}{\strcvx} \le 1/19$. %
Then, it holds that
\begin{align*}
\bcovparam
\le
\frac{10 \step }{\nagent \strcvx} \optvar 
+ \frac{2}{\step \strcvx \nlupdates} \left( \norm{ \mathrm{R}^{\param} }
+ \frac{\step (\nlupdates-1)}{\nagent} \sum_{c=1}^\nagent  \norm{ \mathrm{R}^{\param, \cvarw}_{(c)} }
+ \frac{\step^2(\nlupdates-1)^2}{\nagent^2} \sum_{c,c'=1}^\nagent \norm{ \mathrm{R}_{(c,c')}^{\cvarw} }
\right)
\eqsp.
\end{align*}
\end{lemma}
\begin{proof}
Plugging \eqref{eq:ineq-bound-covariances-cvar} in \eqref{eq:ineq-bound-covariances-param-cvar}, we obtain
\begin{align*}
\bcovparamcvar
& \le
2 \heterboundhess \bcovparam
\!+\! 4 {\step^3\nlupdates} (\nlupdates\!-\!1)^2 \lip^2
\Bigg(
2 \heterboundhess^2
\bcovparam
\!\!+\! \frac{9\globboundnoise}{\nagent \nlupdates^2} 
\!+\!
4 \heterboundhess
\step (\nlupdates \!-\! 1) \lip
\bcovparamcvar
\!+\! \frac{2}{\nagent^2} \!\!\sum_{c,c'=1}^\nagent \norm{ \mathrm{R}_{(c,c')}^{\cvarw} }
\Bigg)
\!+\! \frac{4 \step}{\nagent \nlupdates} \globboundnoise
\!+\! \frac{2}{\nagent} \!\sum_{c=1}^\nagent  \norm{ \mathrm{R}^{\param, \cvarw}_{(c)} }
\\
& \le
3 \heterboundhess \bcovparam
+ \frac{5 \step}{\nagent \nlupdates} \globboundnoise
+ 16 {\step^4\nlupdates} (\nlupdates-1)^3 \lip^3 \heterboundhess
\bcovparamcvar
+ \frac{8 {\step^3\nlupdates} (\nlupdates-1)^2 \lip^2}{\nagent^2} \sum_{c,c'=1}^\nagent \norm{ \mathrm{R}_{(c,c')}^{\cvarw} }
+ \frac{2}{\nagent} \sum_{c=1}^\nagent  \norm{ \mathrm{R}^{\param, \cvarw}_{(c)} }
\eqsp,
\end{align*}
where we used $\step \nlupdates \lip \le 1/12$ to bound
$4 {\step^3\nlupdates} (\nlupdates-1)^2 \lip^2 \cdot 2 \heterboundhess^2 \le \heterboundhess$ and
$4 {\step^3\nlupdates} (\nlupdates-1)^2 \lip^2 \cdot \frac{9}{\nagent \nlupdates^2} \le \frac{\step}{\nagent \nlupdates}$.
Using this inequality again, we have $16 {\step^4\nlupdates} (\nlupdates-1)^3 \lip^3 \heterboundhess \le 1/12^2$.
This allows to simplify the previous inequality, obtaining
\begin{align}
\label{eq:lem-bound-covparam-interm1}
\bcovparamcvar
& \le
4 \heterboundhess \bcovparam
+ \frac{6 \step}{\nagent \nlupdates} \globboundnoise
+ \frac{9 {\step^3\nlupdates} (\nlupdates-1)^2 \lip^2}{\nagent^2} \sum_{c,c'=1}^\nagent \norm{ \mathrm{R}_{(c,c')}^{\cvarw} }
+ \frac{3}{\nagent} \sum_{c=1}^\nagent  \norm{ \mathrm{R}^{\param, \cvarw}_{(c)} }
\eqsp.
\end{align}
Plugging this bound in \eqref{eq:ineq-bound-covariances-cvar}, we obtain
\begin{align}
\label{eq:lem-bound-covparam-interm2}
\frac{1}{\nagent}
\bcovcvareq
+ \left( 1 - \frac{1}{\nagent} \right)
\bcovcvarneq
\le
4 \heterboundhess^2
\bcovparam
+ \frac{10}{\nagent \nlupdates^2} \globboundnoise
+ \frac{3}{\nagent^2} \sum_{c,c'=1}^\nagent \norm{ \mathrm{R}_{(c,c')}^{\cvarw} }
+ \frac{12 \heterboundhess
\step (\nlupdates - 1) \lip}{\nagent} \sum_{c=1}^\nagent  \norm{ \mathrm{R}^{\param, \cvarw}_{(c)} }
\eqsp,
\end{align}
where we used $4 \heterboundhess
\step (\nlupdates - 1) \lip \cdot 4 \heterboundhess \le 2 \heterboundhess^2$,
~~
$4 \heterboundhess
\step (\nlupdates - 1) \lip \cdot \frac{6 \step}{\nagent \nlupdates} \le \frac{1}{\nagent \nlupdates^2}$ 
~~and
$4 \heterboundhess
\step (\nlupdates - 1) \lip \cdot \frac{9 {\step^3\nlupdates} (\nlupdates-1)^2 \lip^2}{\nagent^2} \le \frac{1}{\nagent^2}$. 

We now plug \eqref{eq:lem-bound-covparam-interm1} and \eqref{eq:lem-bound-covparam-interm2} in \eqref{eq:ineq-bound-covariances-param}, which gives
\begin{align*}
\bcovparam
& \le
(1 - \step \strcvx)^{\nlupdates} \bcovparam 
+ \left( {\step^2\nlupdates(\nlupdates-1) \lip} + \frac{\step^4\nlupdates^2 (\nlupdates-1)^2 \lip^2 \heterboundhess}{4} \right) \cdot 4 \heterboundhess \bcovparam
\\
& \quad
+ \left( 
\frac{6 \step^3\nlupdates(\nlupdates-1) \lip}{\nagent \nlupdates}
+ \frac{10 \step^4\nlupdates^2 (\nlupdates-1)^2 \lip^2}{4 \nagent \nlupdates^2}
+ \frac{\step^2}{\nagent} \right) \globboundnoise
\\
& \quad %
+ \left( {\step^2\nlupdates(\nlupdates-1) \lip \cdot 9 {\step^3\nlupdates} (\nlupdates-1)^2 \lip^2} 
+ 
\frac{3 \step^4\nlupdates^2 (\nlupdates-1)^2 \lip^2}{4}
\right)
\frac{1}{\nagent^2} \sum_{c,c'=1}^\nagent \norm{ \mathrm{R}_{(c,c')}^{\cvarw} }
\\
& \quad
+ \left( {3\step^2\nlupdates(\nlupdates-1) \lip } 
+ \frac{\step^4\nlupdates^2 (\nlupdates-1)^2 \lip^2}{4} \cdot 12 \heterboundhess
\step (\nlupdates - 1) \lip
\right) \frac{1}{\nagent} \sum_{c=1}^\nagent  \norm{ \mathrm{R}^{\param, \cvarw}_{(c)} }
+ \norm{ \mathrm{R}^{\param} }
\eqsp,
\end{align*}
which can be simplified using $\step \nlupdates \lip \le 1/12$ to obtain
\begin{align*}
& \bcovparam
\le
(1 - \step \strcvx)^{\nlupdates} \bcovparam 
+ 5 {\step^2\nlupdates(\nlupdates-1)} \lip \heterboundhess \bcovparam
+ \frac{2\step^2}{\nagent} \globboundnoise
+ \norm{ \mathrm{R}^{\param} }
+ \frac{\step (\nlupdates-1)}{\nagent} \sum_{c=1}^\nagent  \norm{ \mathrm{R}^{\param, \cvarw}_{(c)} }
+ \frac{\step^2(\nlupdates-1)^2}{\nagent^2} \sum_{c,c'=1}^\nagent \norm{ \mathrm{R}_{(c,c')}^{\cvarw} }
\eqsp.
\end{align*}
Now, using $\step \nlupdates \strcvx \le 1$, we have $(1 - \step \strcvx)^\nlupdates \le 1 - \step \strcvx \nlupdates / 2$.
Consequently, we have $(1 - \step \strcvx)^{\nlupdates} \bcovparam 
+ 5 {\step^2\nlupdates(\nlupdates-1)} \heterboundhess \bcovparam \le 1 - \step \nlupdates (\strcvx - 5 \step (\nlupdates-1) \lip \heterboundhess$.
Since we assumed $5 \step (\nlupdates - 1) \lip \heterboundhess \le \strcvx / 2$, we obtain
\begin{align*}
& \bcovparam
\le
(1 - \step \strcvx \nlupdates / 2) \bcovparam 
+ \frac{2\step^2}{\nagent} \globboundnoise
+ \norm{ \mathrm{R}^{\param} }
+ \frac{\step (\nlupdates-1)}{\nagent} \sum_{c=1}^\nagent  \norm{ \mathrm{R}^{\param, \cvarw}_{(c)} }
+ \frac{\step^2(\nlupdates-1)^2}{\nagent^2} \sum_{c,c'=1}^\nagent \norm{ \mathrm{R}_{(c,c')}^{\cvarw} }
\eqsp.
\end{align*}
We then bound the variance term using \Cref{lem:bound-noise-loc}, which implies that
\begin{align*} 
\globboundnoise
& \le
\nlupdates \optvar + \frac{28 \step \smoothcstvar \nlupdates}{\strcvx} \optvar
\eqsp.
\end{align*}
Plugging this bound in the previous inequality, we obtain
\begin{align*}
\frac{\step \strcvx \nlupdates}{2} \bcovparam
\le
\frac{2\step^2 \nlupdates }{\nagent} \optvar 
+ \frac{56 \step^3 \smoothcstvar \nlupdates}{\nagent \strcvx} \optvar
+ \norm{ \mathrm{R}^{\param} }
+ \frac{\step (\nlupdates-1)}{\nagent} \sum_{c=1}^\nagent  \norm{ \mathrm{R}^{\param, \cvarw}_{(c)} }
+ \frac{\step^2(\nlupdates-1)^2}{\nagent^2} \sum_{c,c'=1}^\nagent \norm{ \mathrm{R}_{(c,c')}^{\cvarw} }
\eqsp,
\end{align*}
which gives the first inequality of the theorem.
\end{proof}

\boundcovstationary*  
\begin{proof}
The result follows from \Cref{thm:bound-cov-speed-up} and \Cref{cor:bound-remainder}.
\end{proof}

\subsection{Bounds on intermediate quantities}

\subsubsection{Bound on matrices}
\begin{lemma}{Bound on $\loccontractw{c}$'s powers}
\label{lem:bound-loccontract}
Let $h > 0$, $\step \ge 0$, recall $\loccontractw{c} = \Id - \step \hnf{c}{\paramlim}$.
Assume \Cref{assum:strong-convexity}, \Cref{assum:smoothness}, and that $\step \le 1/\lip$, then it holds that
\begin{align*}
\norm{ \loccontractw{c}^h } \le (1 - \step \strcvx)^h
\eqsp.
\end{align*}
\end{lemma}
\begin{proof}
Follows from \Cref{assum:strong-convexity} and \Cref{assum:smoothness} with $\step \le 1/\lip$.
\end{proof}

\begin{lemma}
\label{lem:bound-shiftedlocamat}
Let $h > 0$, $\step \ge 0$, recall $\loccontractw{c} = \Id - \step {\hnf{c}{\paramlim}}$.
Assume \Cref{assum:strong-convexity}, \Cref{assum:smoothness}, and that $\step \le 1/\lip$, then it holds that
\begin{align*}
\norm{ \shiftedlocmat{c}{1:\nlupdates} }
& \le \frac{\step (\nlupdates-1) \lip}{2}
\eqsp,
\end{align*}
\end{lemma}
\begin{proof}
Recall that $\shiftedlocmat{c}{1:\nlupdates} = \Id - \frac{1}{\nlupdates} \sum_{h=0}^{\nlupdates-1} \loccontractw{c}^{\nlupdates - h - 1}$.
Since  for any (square) matrix $A$ any$\gamma > 0$ and any $k \in \nset^*$ we get that 
\[
\Id - \bigl(I - \gamma A\bigr)^{k} = \gamma A \,\sum_{\ell = 0}^{k-1} \bigl(\Id - \gamma A\bigr)^\ell,
\]
we obtain 
\begin{align}
\label{eq:dev-locshiftedmat}
\shiftedlocmat{c}{1:\nlupdates}
& = \frac{1}{\nlupdates} \sum_{h=0}^{\nlupdates-1} \Big( 
\Id -  (\Id - \step \hnf{c}{\paramlim})^{\nlupdates - h - 1}
\Big)
=  \frac{\step}{\nlupdates} \hnf{c}{\paramlim}
\sum_{h=0}^{\nlupdates-1} 
(\nlupdates - h -1) (\Id - \gamma \hnf{c}{\paramlim})^h
\eqsp.
\end{align}
Using the triangle inequality, \Cref{assum:smoothness} and \Cref{lem:bound-loccontract}, we obtain
\begin{align*}
\norm{ \shiftedlocmat{c}{1:\nlupdates} }
& = \frac{\step \lip}{\nlupdates}
\sum_{h=0}^{\nlupdates-1} (\nlupdates - h - 1)
(1 - \step \strcvx)^{h}  
\eqsp,
\end{align*}
and the lemma follows from $\sum_{h=0}^{\nlupdates-1} h = \frac{\nlupdates(\nlupdates-1)}{2}$.
\end{proof}

\begin{lemma}
\label{lem:bound-diff-loccontract}
Let $h > 0$, $\step \ge 0$, recall $\loccontractw{c} = \Id - \step \gnf{c}{\paramlim}$.
Assume \Cref{assum:strong-convexity}, \Cref{assum:smoothness}, \Cref{assum:heterogeneity}, and that $\step \le 1/\lip$, then it holds that
\begin{align*}
\norm{ \diffcontractc{c} }
& \le \heterboundhess
\eqsp.
\end{align*}
\end{lemma}
\begin{proof}
We have, using \Cref{lem:product_coupling_lemma},
\begin{align*}
\frac{1}{\step \nlupdates}
\left( \loccontractw{c}^\nlupdates - \globcontractw \right) 
& =
\frac{1}{\step \nlupdates \nagent}
\sum_{i=1}^\nagent 
\left( \big( \Id - \step \gnf{c}{\paramlim} \big)^\nlupdates 
- \big( \Id - \step \gnf{i}{\paramlim} \big)^\nlupdates  \right) 
\\
& =
\frac{1}{\nlupdates \nagent}
\sum_{i=1}^\nagent 
\sum_{h=0}^{\nlupdates}
 \big( \Id - \step \gnf{c}{\paramlim} \big)^{h-1} 
 \big( \gnf{c}{\paramlim} - \gnf{i}{\paramlim} \big)
- \big( \Id - \step \gnf{i}{\paramlim} \big)^{\nlupdates-h-1}
\eqsp.
\end{align*}
The result follows from taking the norm, using triangle inequality, \Cref{lem:bound-loccontract}, and \Cref{assum:heterogeneity}.
\end{proof}

\subsubsection{Bound on the noise terms}

\begin{lemma}
\label{lem:bound-noise-loc}
Assume \Cref{assum:strong-convexity}, \Cref{assum:smoothness} and \Cref{assum:smooth-var}.
Let $\step > 0$, $\nlupdates > 0$, such that $\step \nlupdates (\lip + \strcvx) \le 1/12$, then
\begin{align*}
\int
\PE \left[ \norm{\locnoiseabv{c}{1:\nlupdates}}^2 \right]
\statdist{\step, \nlupdates}(\rmd \theta, \rmd \Cvarw)
& \le
\nlupdates \optvar + \frac{28 \step \smoothcstvar \nlupdates}{\strcvx} \optvar
\eqsp.
\end{align*}
\end{lemma}
\begin{proof}
Since $\locnoiseabv{c}{h}$ is a martingale difference sequence, we have
\begin{align}
\PE \left[ \norm{\locnoiseabv{c}{1:\nlupdates}}^2 \right]
= 
\sum_{h=0}^{\nlupdates-1} \PE \left[ \norm{  \loccontractw{c}^{\nlupdates - h - 1} \locnoiseabv{c}{h+1} }^2 \right]
\le
\sum_{h=0}^{\nlupdates-1}
\norm{  \loccontractw{c}^{\nlupdates - h - 1} } 
\PE \left[\norm{ \locnoiseabv{c}{h+1} }^2 \right]
\eqsp.
\end{align}
By \Cref{lem:bound-loccontract}, and \Cref{assum:smooth-var}, we have
\begin{align}
\PE \left[ \norm{\locnoiseabv{c}{1:\nlupdates}}^2 \right]
\le
\sum_{h=0}^{\nlupdates-1}
(1 - \step \strcvx)^h \left( \optvar + \smoothcstvar \norm{ \locscafopabv{c}{h}{\param; \cvar{c}{}}  - \paramlim }^2
\right)
\eqsp.
\end{align}
Integrating over the stationary distribution $\statdist{\step, \nlupdates}$, and using \Cref{lem:crude-bound-local-and-cvar} gives the result.
\end{proof}

\subsubsection{Bound on the remainders}

\begin{lemma}
\label{lem:bound-reste-loc}
Assume \Cref{assum:strong-convexity}, \Cref{assum:smoothness}, \Cref{assum:third-derivative}, and \Cref{assum:smooth-var}.
Let $\step > 0$, $\nlupdates > 0$, such that $\step \nlupdates (\lip + \strcvx) \le 1/12$, then
\begin{align*}
\int \PE \left[ \norm{ \locreste{c}{1:\nlupdates} } \right]
\statdist{\step, \nlupdates}(\rmd \theta, \rmd \Cvarw)
\le 
\frac{28 \step \nlupdates \thirdlip}{\strcvx} \optvar
\eqsp.
\end{align*}
If $\step \lip \le 1 / 48$, $\step \nlupdates (\lip + \strcvx) \le 1/24$, $\step \nlupdates^{1/2} \smoothcstvar^{1/2} \le 1 / 12$ and $\step \smoothcstvar \le \lip / 12$, then it also holds that
\begin{align*}
\int \PE \left[ \norm{ \locreste{c}{1:\nlupdates} }^2 \right]
\statdist{\step, \nlupdates}(\rmd \theta, \rmd \Cvarw)
\le 
\frac{600^2 \step^2 \nlupdates^2 \thirdlip^2}{\strcvx^2} \sqoptvar^4
\eqsp.
\end{align*}
\end{lemma}
\begin{proof}
Taking the norm of $\locreste{c}{1:\nlupdates}$, and using the triangle inequality, \Cref{assum:third-derivative}, and \Cref{lem:bound-loccontract}, we have
\begin{align*}
\norm{ \locreste{c}{1:\nlupdates} }
&
\le
\sum_{h=0}^{\nlupdates-1} \bnorm{ \loccontractw{c}^{\nlupdates - h - 1} \avghhnf{c}{h} \left( \locstoscafop{c}{h}{\param}{\cvar{c}{}}{\locRandState{c}{1:h}} - \paramlim \right)^{\otimes 2} }
\le
\sum_{h=0}^{\nlupdates-1}
\thirdlip
\norm{ \locstoscafop{c}{h}{\param}{\cvar{c}{}}{\locRandState{c}{1:h}} - \paramlim }^2
\eqsp.
\end{align*}
Integrating over the stationary distribution and taking the expectation, and using \Cref{lem:crude-bound-local-and-cvar}, we obtain the first inequality.
The second inequality follows from similar computations, using Jensen's inequality to bound
\begin{align*}
\norm{ \locreste{c}{1:\nlupdates} }^2
&
\le
\nlupdates
\sum_{h=0}^{\nlupdates-1} \bnorm{ \loccontractw{c}^{\nlupdates - h - 1} \avghhnf{c}{h} \left( \locstoscafop{c}{h}{\param}{\cvar{c}{}}{\locRandState{c}{1:h}} - \paramlim \right)^{\otimes 2} }^2
\le
\nlupdates
\sum_{h=0}^{\nlupdates-1}
\thirdlip^2
\norm{ \locstoscafop{c}{h}{\param}{\cvar{c}{}}{\locRandState{c}{1:h}} - \paramlim }^4
\eqsp,
\end{align*}
and the result follows from taking the expectation and integrating over \Scaffold's stationary distribution, then using \Cref{lem:crude-bound-local-and-cvar-higher-order} to bound each term of the sum.
\end{proof}

\section{Non-Asymptotic Rates for \Scaffold~-- Proof of Theorem~\ref{thm:convergence-scaffold-rate} }
\label{sec:non-asymptotic-rates}
\convergenceratescaffold*
\begin{proof}
Let $\stationaryparam{0} \in \rset^d$ and $\stationarycvar{1}{0}, \cdots, \stationarycvar{\nagent}{0} \in \rset^d$ be sampled from \Scaffold's stationary distribution
\begin{align*}
\stationarybigX^0 =
\left(
\stationaryparam{0},
\stationarycvar{1}{0},
\cdots,
\stationarycvar{\nagent}{0}
\right)
\sim
\statdist{\step, \nlupdates}
\eqsp.
\end{align*}
For an \iid\, sequence $\{\locRandState{1:\nagent}{t,1:\nlupdates}\}_{t\in\nset}$ determining the randomness of the algorithm, where $\locRandState{c}{t,h} \sim \distRandState{c}$ for $c \in \iint{1}{\nagent}$ and $h \in \iint{0}{\nlupdates}$, we define two sequences, starting respectively from $\bigX^0 = \left(
\globparam{0},
\cvar{1}{0},
\cdots,
\cvar{\nagent}{0}
\right)$ and $\stationarybigX^0 = \left(
\stationaryparam{0},
\stationarycvar{1}{0},
\cdots,
\stationarycvar{\nagent}{0}
\right)$, and following the recursion for $t \ge 0$,
\begin{align*}
\bigX^{t+1} & =
\left(
\globparam{t+1},
\cvar{1}{t+1},
\cdots,
\cvar{\nagent}{t+1}
\right)
=
\opscaffold\Big(
\bigX^t ; \locRandState{1:\nagent}{t+1,1:\nlupdates}\Big) 
\eqsp,
\\
\stationarybigX^{t+1}
& =
\left(
\stationaryparam{t+1},
\stationarycvar{1}{t+1},
\cdots,
\stationarycvar{\nagent}{t+1}
\right)
=
\opscaffold\Big(
\stationarybigX^t ; \locRandState{1:\nagent}{t+1,1:\nlupdates}\Big) 
\eqsp.
\end{align*}
The first sequence are the actual iterates of \Scaffold, while the second one is its counterpart with the same realization of noise, but initialized in the stationary distribution.
By definition of the stationary distribution, all iterations of this second sequence also follow the stationary distribution, \ie, for all $t \ge 0$,
\begin{align*}
\stationarybigX^t
\sim
\statdist{\step, \nlupdates}
\eqsp.
\end{align*}
We can thus decompose the error in two parts
\begin{align}
\PE\left[ \norm{ \globparam{\nrounds} - \paramlim }^2 \right] \le 2 \PE[ \norm{  \globparam{\nrounds} - \stationaryparam{\nrounds} }^2 ] + 2 \PE[ \norm{ \stationaryparam{\nrounds} - \paramlim }^2 ]
\eqsp,
\label{eq:bound-split-convergence-rate}
\end{align}
where we recall $\bigX^\star = \left( \paramlim, \cvarlim{1}, \dots, \cvarlim{\nagent} \right)$ is the optimal vector.
The first term is an optimization term, which determines the distance from current iterate to an iterate drawn in the stationary distribution.
The second term is the variance in the stationary distribution.
We now bound each of these two terms.

\textbf{Bounding the optimization term.}
Using \Cref{lem:contraction-scaffold-global-update} recursively with the natural filtration of the process $\left\{ \bigX^t \right\}_{t \ge 0}$, we can bound the first term as
\begin{align}
\nonumber
2 \PE[ \norm{  \globparam{\nrounds} - \paramlim }^2 ]
& \le
2 \PE[ \norm{  \bigX^\nrounds - \stationarybigX^\nrounds }^2 ]
\\
\nonumber
& \le
2 \left( 1 - \frac{\step \strcvx}{4} \right)^{\nlupdates \nrounds}
\norm{  \bigX^0 - \stationarybigX^0 }^2
\\
& \le
4 \left( 1 - \frac{\step \strcvx}{4} \right)^{\nlupdates \nrounds}
\norm{  \bigX^0 - \bigX^\star }^2
+ 4 \left( 1 - \frac{\step \strcvx}{4} \right)^{\nlupdates \nrounds}
\norm{ \stationarybigX^0 - \bigX^\star }^2
\eqsp.
\label{eq:bound-first-part-convergence-rate-before-int}
\end{align}
Integrating \eqref{eq:bound-first-part-convergence-rate-before-int} over the stationary distribution and using \Cref{cor:crude-bounds-global}, we have
\begin{align}
2 \PE[ \norm{  \globparam{\nrounds} - \paramlim }^2 ]
& \le
4 \left( 1 - \frac{\step \strcvx}{4} \right)^{\nlupdates \nrounds}
\norm{  \bigX^0 - \bigX^\star }^2
+
\left( 1 - \frac{\step \strcvx}{4} \right)^{\nlupdates \nrounds}
\frac{ 64 \step }{\strcvx} \optvar
\eqsp.
\label{eq:bound-first-part-convergence-rate}
\end{align}

\textbf{Bounding the variance term.}
For the second term, we use \Cref{thm:bound-cov-stationary} to bound
\begin{align}
\nonumber
2 \PE[ \norm{ \stationaryparam{\nrounds} - \paramlim }^2 ]
& =
2 \int \norm{ \param - \paramlim }^2 
\statdist{\step, \nlupdates}(\rmd \param, \rmd \Cvarw)
\le
2 d 
\norm{ \covparam }
\\
\label{eq:bound-second-part-convergence-rate}
& 
\le
\frac{20 d \step }{\nagent \strcvx} \optvar 
+ \frac{ 12 \cdot 15080 d \step^{3/2} \thirdlip}{\strcvx^{5/2}}
\sqoptvar^3
+
\frac{96 \cdot 600^2 d \step^3 \nlupdates \thirdlip^2}{\nagent\strcvx^3} \sqoptvar^4
\eqsp.
\end{align}

\textbf{Final rate.}
Plugging \eqref{eq:bound-first-part-convergence-rate} and \eqref{eq:bound-second-part-convergence-rate} in \eqref{eq:bound-split-convergence-rate}, we obtain
\begin{align*}
\PE\left[ \norm{ \globparam{\nrounds} - \paramlim }^2 \right] 
& \le 
\left( 1 - \frac{\step \strcvx}{4} \right)^{\nlupdates \nrounds}
\left( 
\norm{  \param - \paramlim }^2
+
\frac{\step^2 \nlupdates^2}{\nagent}
\sum_{c=1}^\nagent \norm{ \cvar{c}{} - \cvarlim{c} }^2 
+
\frac{ 64 \step }{\strcvx} \optvar
\right)
\\
& \quad
+ 
\frac{20 d \step }{\nagent \strcvx} \optvar 
+ \frac{ 16 \cdot 15080  \step^{3/2} \thirdlip}{\strcvx^{5/2}}
\sqoptvar^3
+
\frac{96 \cdot 600^2 d \step^3 \nlupdates \thirdlip^2}{\nagent\strcvx^3} \sqoptvar^4
\eqsp,
\end{align*}
and the result follows by taking $\cvar{c}{} = 0$ for all $c \in \iint{1}{\nagent}$ and using the fact that $\cvarlim{c} = - \gnf{c}{\paramlim}$.
\end{proof}

\samplecommcomplexityscaffold*
\begin{proof}
By \Cref{thm:convergence-scaffold-rate}, we have
\begin{align*}
\PE\left[ \norm{ \globparam{\nrounds} - \paramlim }^2 \right]
& \lesssim{}
\Big(1 - \frac{\step \strcvx}{4} \Big)^{\nlupdates \nrounds} 
\!\! 
\left\{ \norm{ \globparam{0} - \paramlim }^2 + \step^2 \nlupdates^2 \heterboundgrad
+ \frac{\step \optvar}{\strcvx}
\right\}
+ \frac{\step }{\nagent \strcvx} \optvar 
+ \frac{ \step^{3/2} \thirdlip}{\strcvx^{5/2}} \sqoptvar^3
+
\frac{\step^3 \nlupdates \thirdlip^2}{\strcvx^3} \sqoptvar^4
\eqsp.
\end{align*}
Under our assumptions, we can take
\begin{align*}
\step \lesssim{} \frac{\nagent \strcvx \epsilon^2}{\optvar}
\eqsp,
\end{align*}
assuming is small enough $\smoothcstvar$, that is $\smoothcstvar \lesssim \optvar / (\nagent \epsilon^2)$.
This implies $\step\optvar/\strcvx \lesssim \nagent \epsilon^2$.
In \Cref{thm:convergence-scaffold-rate}, we require $\step \nlupdates \lip \lesssim 1$ and $\step \nlupdates \lip \heterboundhess \lesssim \strcvx$.
Thus, we set
\begin{align}
\label{eq:complexity-nlupdates-bounds}
\nlupdates
\lesssim \frac{\optvar}{\nagent \lip \strcvx \epsilon^2} \min( 1, \strcvx / \heterboundhess)
\eqsp,
\end{align}
which ensures that both conditions are satisfied when $\step \lesssim{} \frac{\nagent \strcvx \epsilon^2}{\optvar}$.

If the number of clients satisfies $\nagent \lesssim \frac{\strcvx^{2/3}}{\thirdlip^{2/3} \epsilon^{2/3}}$ and $\nagent \lesssim \frac{\lip^{1/2}\strcvx^{1/2}}{\thirdlip \epsilon}$, we thus have
\begin{align*}
\PE\left[ \norm{ \globparam{\nrounds} - \paramlim }^2 \right]
& \lesssim{}
\Big(1 - \frac{\step \strcvx}{4} \Big)^{\nlupdates \nrounds} 
\!\! 
\left\{ \norm{ \globparam{0} - \paramlim }^2 + 
\step^2 \nlupdates^2 \heterboundgrad^2
+ \frac{\step \optvar}{\strcvx}
\right\}
+ \epsilon^2/2
+ \frac{ \nagent^{3/2} \thirdlip \epsilon^3}{\strcvx} 
+
\frac{\nagent^2 \thirdlip^2 \epsilon^4 \min( 1, \strcvx / \heterboundhess)}{\strcvx} 
\\
& \lesssim{}
\Big(1 - \frac{\step \strcvx}{4} \Big)^{\nlupdates \nrounds} 
\!\! 
\left\{ \norm{ \globparam{0} - \paramlim }^2 + \step^2 \nlupdates^2 \heterboundgrad^2
+ \frac{\step \optvar}{\strcvx}
\right\}
+ \epsilon^2
\eqsp.
\end{align*}
Now, we choose $\step$ and $\nlupdates$ as big as possible, which gives
\begin{align}
\label{eq:complexity-nrounds-bounds}
\nrounds
\gtrsim 
\frac{\lip}{\strcvx} \max(1, \heterboundhess/\strcvx) 
\log\left( 
\frac{\norm{ \globparam{0} - \paramlim }^2 + \heterboundgrad^2 / \lip^2 + \optvar/ (\lip \strcvx)}{\epsilon^2}
\right)
\eqsp,
\end{align}
such that $\PE\left[ \norm{ \globparam{\nrounds} - \paramlim }^2 \right] \leq \epsilon^2$.
Since each client computes $\nrounds \nlupdates$ gradients, the result follows from \eqref{eq:complexity-nlupdates-bounds} and \eqref{eq:complexity-nrounds-bounds}.
\end{proof}

\section{Bias of \Scaffold}
\label{sec:bias-scaffold}
We now give first-order expression of the bias of \Scaffold.
For $\param \in \rset^d$ and $\Cvarw = (\cvar{1}{}, \dots, \cvar{\nagent}{}) \in \rset^{\nagent \times d}$, we define the bias in the stationary distribution of the parameters and control variates as
\begin{align*}
\biasparam 
& \eqdef{}
\int \left( \param - \paramlim \right)  \statdist{\step, \nlupdates}( \rmd \param, \rmd \Cvarw )
\eqsp,
\qquad
b
 \eqdef{}
\int \left( \cvar{c}{} - \cvarlim{c} \right)  \statdist{\step, \nlupdates}( \rmd \param, \rmd \Cvarw )
\eqsp.
\end{align*}
Alike the tensors defined in \eqref{eq:def-mat-d2} and \eqref{eq:def-mat-d3}, we define the following tensor that will be used to expand the gradients to third order,
\begin{align}
\label{eq:def-mat-d4}
\avghhhnf{c}{h}(\theta)
& =
\int_0^1 {(1-t)^2} \hhhnf{c}{\paramlim + t\left( \locstoscafop{c}{h}{\param}{\cvar{c}{}}{\locRandState{c}{1:h}} - \paramlim\right)} 
\rmd t
\eqsp.
\end{align}
As in \Cref{sec:variance-scaffold}, we will often denote $\avghhhnf{c}{h} = \avghhhnf{c}{h}(\locparam{c}{h})$ for conciseness.

\begin{lemma}
\label{lem:expansion-interm-matrices}
Assume the step size $\step$ and the number of local updates $\nlupdates$ satisfy $\step \nlupdates (\lip + \strcvx) \leq 1/12$. Under these conditions, it holds that
\begin{align*}
\loccontractw{c}^{h} 
& = \Id - \step h \hnf{c}{\paramlim} + O(\step^2 \nlupdates^2)
\eqsp,
\\
\globcontractw
& = \Id - \step \nlupdates \hf{\paramlim} + O(\step^2 \nlupdates^2)
\eqsp,
\\
\shiftedlocmat{c}{1:\nlupdates} 
& =
\frac{\step (\nlupdates - 1)}{2} \hnf{c}{\paramlim}
+ O(\step^2 \nlupdates)
\eqsp.
\end{align*}
\end{lemma}
\begin{proof}
The first equality follows from expanding $\loccontractw{c} = \left( \Id - \step \hnf{c}{\paramlim} \right)^{h}$ using the Binomial theorem and the fact that $\Id$ and $\hnf{c}{\paramlim}$ commute.
Then, terms of higher order can be bounded by bounding the remainder terms using the exponential series and the fact that $\hnf{c}{\paramlim} \preccurlyeq \lip$ with $\step \nlupdates \lip \le 1$.
The second equality follows from the first one with $h = \nlupdates$ and $\hf{\paramlim} = \frac{1}{\nagent} \sum_{c=1}^\nagent \hnf{c}{\paramlim}$.
The last identity follows from \eqref{eq:dev-locshiftedmat} and \Cref{lem:bound-loccontract}.
\end{proof}

\subsection{Bias on the Control Variates}

\begin{lemma}[Bias of Control Variates]
\label{lem:expression-bias-cvar}
Assume \Cref{assum:strong-convexity}, \Cref{assum:smoothness} and \Cref{assum:smooth-var}.
Let $c \in \iint{1}{\nagent}$, $\randStatew = \locRandState{1:\nagent}{1:\nlupdates}$ be i.i.d. random variables.
Assume the step size $\step$ and the number of local updates $\nlupdates$ satisfy $\step \nlupdates (\lip + \strcvx) \leq 1/12$. Under these conditions, control variates' bias satisfies
  \begin{align}
  \biascvar{c}
  & =
  (\hnf{c}{\paramlim} - \hf{\paramlim}) \biasparam
  + O(\step)
    \eqsp,
\end{align}
\end{lemma}
\begin{proof}
Let $\param \in \rset^d$ and $\cvar{1}{}, \dots, \cvar{\nagent}{} \in \rset^d$.
For $c \in \iint{1}{\nagent}$ and $h \in \iint{0}{\nlupdates}$, define
\begin{align}
\pch{c}{h}
& = \locscafopabv{c}{h}{\param; \cvar{c}{}}
=
    \param
    - \step
    \sum_{\ell=0}^{h-1}
    \left\{
    \gnf{c}{\locscafopabv{c}{\ell}{\param; \cvar{c}{}}}
    + \updatefuncnoise{c}{\ell+1}{\locRandState{c}{\ell+1}}
    + \cvar{c}{}
  \right\}
  \eqsp,
\\
\cvplus{c} 
& = 
\scafopcv{c}{\nlupdates}{\cvar{c}{}}{\param}{\locRandState{1:\nagent}{1:\nlupdates}}
=
\cvar{c}{}
+ \frac{1}{\step \nlupdates}
\left(
\locstoscafop{c}{\nlupdates}{\param}{\cvar{c}{}}{\locRandState{c}{1:\nlupdates}}
-
\globstoscafop{\nlupdates}{\param}{\cvar{1:\nagent}{}}{\locRandState{1:\nagent}{1:\nlupdates}}
\right)
\eqsp.
\end{align}
First, we derive a first-order expansion of the local updates error.
Using \Cref{cor:crude-bounds-global} to bound the remainder term, we have
\begin{align}
\pch{c}{h} - \paramlim
& =
  \param - \paramlim
  - \step
  \sum_{\ell=0}^{h-1}
  \left\{
  \gnf{c}{\paramlim}
  + \updatefuncnoise{c}{\ell+1}{\locRandState{c}{\ell+1}}
  + \cvar{c}{}
  + O(\step^{1/2})
  \right\}
  \\
& =
  \param - \paramlim
  - \step h (\cvar{c}{} - \cvarlim{c})
  - \step \sum_{\ell=0}^{h-1} \updatefuncnoise{c}{\ell+1}{\locRandState{c}{\ell+1}}
  + O(\step^{3/2} h)  
  \eqsp.
\end{align}
Then, we recall the expression of the control variates updates
\begin{align}
\cvplus{c} 
& = 
\cvar{c}{}
+ \frac{1}{\step \nlupdates}
\left(
\locstoscafop{c}{\nlupdates}{\param}{\cvar{c}{}}{\locRandState{c}{1:\nlupdates}}
-
\globstoscafop{\nlupdates}{\param}{\cvar{1:\nagent}{}}{\locRandState{1:\nagent}{1:\nlupdates}}
  \right)
\\
& = 
\cvar{c}{}
- \frac{1}{\nlupdates}
\left(
\sum_{h=0}^{\nlupdates}
    \gnf{c}{\pch{c}{h}}
    + \updatefuncnoise{c}{\ell+1}{\locRandState{c}{\ell+1}}
    + \cvar{c}{}
  -
  \frac{1}{\nagent}
  \sum_{i=0}^{\nagent}
  \sum_{h=0}^{\nlupdates}
  \gnf{i}{\pch{i}{h}}
  + \updatefuncnoise{i}{\ell+1}{\locRandState{i}{\ell+1}}
  + \cvar{i}{}
  \right)
\eqsp.
\end{align}
Taking the conditional expectation, and expanding the gradients we have
\begin{align*}
   \PE\left[\cvplus{c}\right]
  &  = 
    - \frac{1}{\nlupdates}
    \Bigg(
    \sum_{h=0}^{\nlupdates}
    \gnf{c}{\paramlim}
    + \hnf{c}{\paramlim}\PE\left[\pch{c}{h} - \paramlim\right]
    -
    \frac{1}{\nagent}
    \sum_{i=0}^{\nagent}
    \sum_{h=0}^{\nlupdates}
    \hnf{i}{\paramlim}\PE\left[ \pch{i}{h} - \paramlim \right]
    + O(\step)
    \Bigg)
  \\
  & = 
    - \frac{1}{\nlupdates}
    \Bigg(
    \sum_{h=0}^{\nlupdates}
    - \cvarlim{c}{}
    + \hnf{c}{\paramlim}\PE\left[\pch{c}{h} - \paramlim\right]
    -
    \frac{1}{\nagent}
    \sum_{i=0}^{\nagent}
    \sum_{h=0}^{\nlupdates}
    \hnf{i}{\paramlim}\PE\left[ \pch{i}{h} - \paramlim \right]
    + O(\step)
    \Bigg)
    \eqsp.
\end{align*}
Since $\PE\left[ \pch{i}{h} - \paramlim \right] = \param - \paramlim + O(\step \nlupdates)$, we have
\begin{align}
  \PE\left[\cvplus{c} - \cvarlim{c}\right]
  & =
  (\hnf{c}{\paramlim} - \hf{\paramlim})(\param - \paramlim)
  + O(\step)
    \eqsp,
\end{align}
and the result of the lemma follows.
\end{proof}

\subsection{Expression of the Parameter's Variance -- Proof of Lemma~\ref{lem:expansion-cov-first-order-main}}
\label{sec:proof-expansion-cov-first-order-main}

\expansioncovfirstorder*
\begin{proof}
\Cref{thm:bound-cov-speed-up} gives $\covparam = O(\step)$.
Then, by \Cref{lem:ineq-bound-covariances}-\eqref{eq:ineq-bound-covariances-param-cvar} and \Cref{lem:crude-bound-local-and-cvar}, it holds that $\frac{1}{\nagent} \sum_{c=1}^\nagent \norm{ \covparamcvar{c} } = O(\step)$.
Finally, \Cref{lem:crude-bound-local-and-cvar} ensures that $\covcvar{c,c'} = O(1/\nlupdates)$ for all $c,c' \in \iint{1}{\nagent}$.

We recall the expression from \Cref{lem:expansion-squared-theta-plus}, \begin{align*}
\covparam
& =
\globcontractw \covparam \globcontractw
+ \frac{\step\nlupdates}{\nagent} \sum_{c=1}^\nagent
\left( 
\globcontractw\covparamcvar{c}  \shiftedlocmat{c}{1:\nlupdates} 
+  \shiftedlocmat{c}{1:\nlupdates} \covcvarparam{c}\globcontractw \right)
+ \frac{\step^2\nlupdates^2}{\nagent^2}
\sum_{c=1}^\nagent \sum_{c'=1}^\nagent
 \shiftedlocmat{c}{1:\nlupdates} \covcvar{c,c'}  \shiftedlocmat{c'}{1:\nlupdates} 
+ \frac{\step^2}{\nagent} \covonestep
+ \mathrm{R}^{\param}
\eqsp.
\end{align*}
By \Cref{lem:expansion-interm-matrices}, to expand the matrices $\globcontractw$ and $\shiftedlocmat{c}{1:\nlupdates}$ for $c \in \iint{1}{\nagent}$, and using $\step \nlupdates = O(1)$, we thus have
\begin{align}
\label{eq:rec-expand-cov-interm}
\covparam
& =
\covparam 
- \step \nlupdates \hf{\paramlim} \covparam 
- \step \nlupdates \covparam \hf{\paramlim} 
+ \frac{\step^2}{\nagent} \covonestep
+ O(\step^3 \nlupdates^2)
+ O(\step^{5/2} \nlupdates)
\eqsp,
\end{align}
where we also used \Cref{lem:remainder-bound-term-by-term} to obtain $\mathrm{R}^\theta = O(\step^{5/2} \nlupdates)$.
Finally, we expand $\covonestep$ using \Cref{assum:smooth-var} and \Cref{cor:crude-bounds-global}, which gives 
\begin{align*}
\covonestep
& = 
\nlupdates \noisecovmat{\paramlim} + O(\step \nlupdates)
\eqsp.
\end{align*}
Plugging this equation in \eqref{eq:rec-expand-cov-interm} and reorganizing the terms gives the result.

\textbf{Covariance of $\param$ and $\cvar{c}{}$.}
From \Cref{lem:expansion-squared-theta-cvar-plus}, recall
\begin{align*}
\covparamcvar{c}
& =
 \globcontractw 
\covparam
\diffcontractc{c}
+
\globcontractw
\covparamcvar{c}
 \shiftedlocmat{c}{1:\nlupdates}
- \frac{1}{\nagent} \sum_{i'=1}^\nagent
\globcontractw \covparamcvar{i'} \shiftedlocmat{i'}{1:\nlupdates}
+ \frac{\step \nlupdates}{\nagent} \sum_{i=1}^\nagent
\shiftedlocmat{i}{1:\nlupdates} \covcvarparam{i}
\diffcontractc{c}
\\
& \quad
+ \frac{\step\nlupdates}{\nagent} \sum_{i=1}^\nagent \shiftedlocmat{i}{1:\nlupdates} \covcvar{i,c} \shiftedlocmat{c}{1:\nlupdates}
- \frac{\step \nlupdates}{\nagent^2} \sum_{i=1}^\nagent \sum_{i'=1}^\nagent
\shiftedlocmat{i}{1:\nlupdates} \covcvar{i,i'} \shiftedlocmat{i'}{1:\nlupdates}
+ \frac{\step}{\nagent \nlupdates}
\left( \loccovonestep{c} -  \covonestep  \right)
+ \mathrm{R}^{\param, \cvarw}_{(c)}
\eqsp,
\end{align*}
which gives
\begin{align*}
\covparamcvar{c}
& =
\covparam
\diffcontractc{c}
+ \frac{\step}{\nagent \nlupdates}
\left( \loccovonestep{c} -  \covonestep  \right)
+ O(\step^2 \nlupdates)
+ O(\step^{3/2})
\eqsp,
\end{align*}
and the result follows.

\textbf{Covariance of control variates.}
Similarly, we obtain
\begin{align*}
\covcvar{c,c} & =
\diffcontractc{c}
\covparam
\diffcontractc{c'}
+ \frac{1}{\nlupdates^2} \loccovonestep{c}
- \frac{2}{\nagent\nlupdates^2} \loccovonestep{c}
+ \frac{1}{\nagent\nlupdates^2} \covonestep
+ O(\step^2 \nlupdates + \step^{3/2})
\eqsp,
\\
\covcvar{c,c'} & =
\diffcontractc{c}
\covparam
\diffcontractc{c'}
- \frac{1}{\nagent\nlupdates^2} \loccovonestep{c}
- \frac{1}{\nagent\nlupdates^2} \loccovonestep{c'}
+ \frac{1}{\nagent\nlupdates^2} \covonestep
+ O(\step^2 \nlupdates + \step^{3/2})
\eqsp,
\end{align*}
and the last two identities follow.
\end{proof}

\subsection{Bias on the Parameters -- Proof of Theorem~\ref{thm:bias-scaffold} }
\label{sec:proof-thm-bias-scafold}

\begin{lemma}
\label{lem:expansion-interm-U}
Assume \Cref{assum:strong-convexity}, \Cref{assum:smoothness} and \Cref{assum:smooth-var}.
Let $\randStatew = \locRandState{1:\nagent}{1:\nlupdates}$ be i.i.d. random variables.
Assume the step size $\step$ and the number of local updates $\nlupdates$ satisfy $\step \nlupdates (\lip + \strcvx) \leq 1/12$. Under these conditions, it holds that
\begin{align*}
\int
\PE\left[ \left( \locparam{c}{h} - \paramlim \right)^{\otimes 2} \right]
\statdist{\step, \nlupdates}(\rmd \theta, \rmd \Cvarw)
& =
\int \Big( \param - \paramlim \Big)^{\otimes 2} \statdist{\step, \nlupdates}(\rmd \theta, \rmd \Cvarw)
+ 
\mathrm{U}_{(c)}^h
\eqsp,
\end{align*}
where $\mathrm{U}_{(c)}^h = O(\step^2 \nlupdates)$.
\end{lemma}
\begin{proof}
To this end, we expand the gradient in $\locparam{c}{h} - \paramlim = \param - \step \sum_{\ell=0}^{h-1} \left\{ \gnf{c}{\locparam{c}{\ell}} + \cvar{c}{} + \locnoiseabv{c}{\ell+1} \right\} - \paramlim $, which gives
\begin{align*}
\left( \locparam{c}{h} - \paramlim \right)^{\otimes 2}
& =
\Big( \param - \paramlim - \step h (\cvar{c}{} - \cvarlim{c} )
- \step \locnoiseabv{c}{1:h}
- \step \sum_{\ell=0}^{h-1} \avghnf{c}{\ell} \left( \locparam{c}{\ell} - \paramlim \right)  \Big)^{\otimes 2}
\eqsp.
\end{align*}
Expanding the square, we get the result with $\mathrm{U}_{(c)}^h$ given by
\begin{align*}
\mathrm{U}_{(c)}^h
& = - \step h \int 
\Big( \param - \paramlim \Big) 
\Big( \cvar{c}{} - \cvarlim{c} 
+ \frac{1}{h} \sum_{\ell=0}^{h-1} 
 \PE\left[ \avghnf{c}{\ell} \left( \locparam{c}{\ell} - \paramlim \right)  \right] \Big)^\top 
 \statdist{\step, \nlupdates}(\rmd \theta, \rmd \Cvarw) 
\\
& \quad
+ \step^2 h^2 
\int 
\Big(\cvar{c}{} - \cvarlim{c} + \frac{1}{h} \sum_{\ell=0}^{h-1} \avghnf{c}{\ell} \left( \locparam{c}{\ell} - \paramlim \right)  \Big)^{\otimes 2}
\statdist{\step, \nlupdates}(\rmd \theta, \rmd \Cvarw) 
+ \step^2 
\int \PE\left[ \left( \locnoiseabv{c}{1:h}\right)^{\otimes 2} \right] \statdist{\step, \nlupdates}(\rmd \theta, \rmd \Cvarw) 
\\ 
& \quad
+ \step^2 \sum_{\ell=0}^{h-1} \int \PE\left[ 
\avghnf{c}{\ell} \left( \locparam{c}{\ell} - \paramlim \right) \left( \locnoiseabv{c}{1:h}\right)^{\top}
+  \left( \locnoiseabv{c}{1:h}\right)\avghnf{c}{\ell} \left( \locparam{c}{\ell} - \paramlim \right)^{\top}
\right] \statdist{\step, \nlupdates}(\rmd \theta, \rmd \Cvarw) 
\eqsp,
\end{align*}
which satisfies $\mathrm{U}_{(c)}^h = O(\step^2 h)$ by \Cref{cor:crude-bounds-global}, \Cref{lem:crude-bound-local-and-cvar}, \Cref{cor:crude-bounds-global-higher-order}, \Cref{lem:crude-bound-local-and-cvar-higher-order} and $\step \nlupdates \lip \le 1$.
\end{proof}

\begin{lemma}
\label{lem:expression-bias-param-incomplete}
Assume \Cref{assum:strong-convexity}, \Cref{assum:smoothness} and \Cref{assum:smooth-var}.
Let $\randStatew = \locRandState{1:\nagent}{1:\nlupdates}$ be i.i.d. random variables.
Assume the step size $\step$ and the number of local updates $\nlupdates$ satisfy $\step \nlupdates (\lip + \strcvx) \leq 1/12$. Under these conditions, it holds that    
\begin{align*}
\biasparam 
& =
- \frac{1}{2 \nagent} \hf{\paramlim}^{-1} \hhf{\paramlim} \covparam 
+ O(\step^2 \nlupdates + \step^{3/2})
\eqsp.
\end{align*}
\end{lemma}
\begin{proof}
By definition of the local updates, we have, for $h \in \iint{0}{\nlupdates-1}$, assuming $\locparam{c}{h+1}$ is $\mcF_c^h$-measurable,
\begin{align*}
\CPE{ \locparam{c}{h+1} - \paramlim }{\mcF_c^h}
& =
\locparam{c}{h}  
- \paramlim
- \step \gnf{c}{\locparam{c}{h}}
- \step \cvar{c}{}
\eqsp.
\end{align*}
Like in \eqref{eq:expansion-grad-hc-second}, we expand the gradient, but for one more order, and use $\cvarlim{c} = - \gnf{c}{\paramlim}$,
\begin{align*}
\CPE{ \locparam{c}{h+1} - \paramlim }{\mcF_c^h}
& =
\locparam{c}{h} 
- \paramlim
- \step \hnf{c}{\paramlim} \left( \locparam{c}{h} - \paramlim \right)
- \frac{\step}{2} \hhnf{c}{\paramlim} \left( \locparam{c}{h} - \paramlim \right)^{\otimes 2}
\\
& \quad 
- \frac{\step}{2} \avghhhnf{c}{h+1} \left( \locparam{c}{h} - \paramlim \right)^{\otimes 3}
- \step \left( \cvar{c}{} - \cvarlim{c} \right)
\eqsp.
\end{align*}
Taking the expectation, unrolling this equality and averaging the result over $c = 1$ to $\nagent$, we obtain
\begin{align*}
& \PE\left[ \globparam{+} - \paramlim \right]
=
\globcontractw \left( \param - \paramlim \right)
+ \frac{\step \nlupdates}{\nagent} \sum_{c=1}^\nagent \shiftedlocmat{c}{1:\nlupdates} \left( \cvar{c}{} - \cvarlim{c} \right)
\\
& \quad
- \frac{\step}{2 \nagent} \sum_{c=1}^\nagent \sum_{h=0}^{\nlupdates-1} \loccontractw{c}^{\nlupdates - h - 1} \hhnf{c}{\paramlim} \PE\left[ \left( \locparam{c}{h} - \paramlim \right)^{\otimes 2} \right]
- \frac{\step}{2 \nagent} \sum_{c=1}^\nagent \sum_{h=0}^{\nlupdates-1} \loccontractw{c}^{\nlupdates - h - 1} 
\PE\left[ \avghhhnf{c}{h} \left( \locparam{c}{h} - \paramlim \right)^{\otimes 3} \right]
\eqsp.
\end{align*}
Integrating over the stationary distribution of \Scaffold and using \Cref{lem:expansion-interm-U}, we obtain
\begin{align*}
\biasparam 
& =
\globcontractw \biasparam 
+ \frac{\step \nlupdates}{\nagent} \sum_{c=1}^\nagent \shiftedlocmat{c}{1:\nlupdates} \biascvar{c}
- \frac{\step}{\nagent} \sum_{c=1}^\nagent \sum_{h=0}^{\nlupdates-1} \loccontractw{c}^{\nlupdates - h - 1} \hhnf{c}{\paramlim} \Big( \covparam + \mathrm{U}_{(c)}^h \Big)
+ \mathrm{W}
\eqsp,
\end{align*}
where $\mathrm{W} =
- \frac{\step}{2 \nagent} \sum_{c=1}^\nagent \sum_{h=0}^{\nlupdates-1} \loccontractw{c}^{\nlupdates - h - 1} 
\avghhhnf{c}{h} \left( \locparam{c}{h} - \paramlim \right)^{\otimes 3}$ satisfies $\mathrm{W} = O(\step^{5/2} \nlupdates)$ by \Cref{assum:fourth-derivative} and \Cref{lem:crude-bound-local-and-cvar-higher-order}.
Plugging in the expansions from \Cref{lem:expansion-interm-matrices}, we obtain
\begin{align*}
\biasparam 
& =
\left( \Id - \step \nlupdates \hf{\paramlim} + O(\step^2 \nlupdates^2) \right) \biasparam 
+ \frac{\step \nlupdates}{\nagent} \sum_{c=1}^\nagent \left( \frac{\step (\nlupdates - 1)}{2} \hnf{c}{\paramlim}
+ O(\step^2 \nlupdates) \right) \biascvar{c}
\\
& \quad
- \frac{\step}{2 \nagent} \sum_{c=1}^\nagent \sum_{h=0}^{\nlupdates-1} \left( \Id + O (\step \nlupdates) \right) \hhnf{c}{\paramlim} \Big( \covparam + \mathrm{U}_{(c)}^h \Big)
+ \mathrm{W}
\eqsp,
\end{align*}
which gives
\begin{align*}
\step \nlupdates \hf{\paramlim} \biasparam 
& =
\frac{\step^2 \nlupdates  (\nlupdates - 1)}{2 \nagent} \sum_{c=1}^\nagent \hnf{c}{\paramlim} \biascvar{c}
- \frac{\step}{2 \nagent} \sum_{c=1}^\nagent \sum_{h=0}^{\nlupdates-1} \hhnf{c}{\paramlim} \covparam 
+ O(\step^3 \nlupdates^2 + \step^{5/2} \nlupdates)
\eqsp.
\end{align*}
The result follows by multiplying by $(\step \nlupdates \hnf{c}{\paramlim})^{-1}$ on both sides, and using \Cref{lem:expression-bias-cvar} to bound $\biascvar{c}$.
\end{proof}

\biasscaffoldexprssion*
\begin{proof}
The result follows by plugging the expression of $\covparam$ from \Cref{lem:expansion-cov-first-order-main} in \Cref{lem:expression-bias-param-incomplete}.
\end{proof}

\section{Useful Lemmas}
\label{sec:useful-lemmas}
\begin{lemma}[Matrix Product Coupling]
\label{lem:product_coupling_lemma}
For any matrix-valued sequences $(M_k)_{k \in \nset}$, $(M'_k)_{k\in \nset}$ and for any $K \in \nset$, it holds that
\begin{equation*}
  \prod_{k=1}^K M_k - \prod_{k=1}^K M'_k
  = \sum_{k=1}^K \left\{\prod_{\ell=1}^{k-1} M_\ell \right\} \big(M_k - M'_k\big) \left\{\prod_{\ell=k+1}^M M'_\ell \right\} \eqsp .
\end{equation*}
\end{lemma}

\begin{lemma}[Projection]
\label{lem:projection}
Let $\nagent > 0$, $\vecX = (x_1, \dots, x_\nagent)$ and $\vecY = (y_1, \dots, y_\nagent)$ with $x_c, y_c \in \rset^d$ for $c \in \iint{1}{\nagent}$.
We define $\barvecX = (\bar{x}, \dots, \bar{x})$ and $\barvecY = (\bar{y}, \dots, \bar{y})$ with $\bar{x} = \nagent^{-1} \sum_{c=1}^\nagent x_c$ and $\bar{y} = \nagent^{-1} \sum_{c=1}^\nagent y_c$.
It holds that
\begin{align*}
\norm{ \barvecX - \barvecY }^2
& = 
\norm{ \vecX - \vecY }^2
- \norm{ (\barvecX - \vecX ) - (\barvecY - \vecY ) }^2
\eqsp,
\end{align*}
where $\norm{ \cdot }$ is the $\ell_2$-norm over $\rset^{\nagent d}$.
Since $\norm{ \barvecX - \barvecY }^2 = \nagent \norm{ \bar{x} - \bar{y} }$, we also have
\begin{align*}
\norm{ \bar{x} - \bar{y} }^2 
& =
\frac{1}{\nagent} \sum_{c=1}^\nagent \{\norm{ x_c - y_c }^2 - \norm{ (\bar{x} - x_c) - (\bar{y} - y_c) }^2\}
\eqsp.
\end{align*}
\end{lemma}
\begin{proof}
Expanding the norm, we have
\begin{align*}
\norm{ \barvecX - \barvecY }^2
& = 
\norm{ \vecX - \vecY + \barvecX - \vecX - \barvecY + \vecY}^2
\\
& = 
\norm{ \vecX - \vecY }^2
+ 2 \pscal{ \vecX - \vecY}{ \barvecX - \vecX - \barvecY + \vecY }
+ \norm{ \barvecX - \vecX - \barvecY + \vecY }^2
\\
& = 
\norm{ \vecX - \vecY }^2
+ 2 \pscal{ \barvecX - \barvecY }{ \barvecX - \vecX - \barvecY + \vecY }
- \norm{ \barvecX - \vecX - \barvecY + \vecY }^2
\eqsp.
\end{align*}
Then, we notice that
\begin{align*}
2 \pscal{ \barvecX - \barvecY }{ \barvecX - \vecX - \barvecY + \vecY }
& \textstyle =
2 \sum_{c=1}^\nagent \pscal{ \bar{x} - \bar{y} }{ \bar{x} - x_i - \bar{y} + y_i }
=
2 \pscal{ \bar{x} - \bar{y} }{ \nagent( \bar{x} - \bar{y} ) + \sum_{c=1}^\nagent  \{y_i - x_i\} }
~.
\end{align*}
And we have $2 \pscal{ \bar{x} - \bar{y} }{ \nagent( \bar{x} - \bar{y} ) + \sum_{c=1}^\nagent \{ y_i - x_i\} }
 = 0$ since $\nagent \bar{x} = \sum_{c=1}^\nagent x_i$ and $\nagent \bar{y} = \sum_{c=1}^\nagent y_i$.
\end{proof}

\end{document}